\documentclass[11pt]{article}

\usepackage{url}
\usepackage[hidelinks]{hyperref}
\usepackage{amsmath,amsthm,amssymb,amsbsy,amsfonts,amscd,bm}
\usepackage{paralist}
\usepackage{xcolor}
\usepackage{color}
\usepackage{graphicx}
\graphicspath{{./figs/}}
\usepackage{algorithm}
\usepackage{algorithmic}
\usepackage{comment}
\usepackage{multirow}
\usepackage{enumitem}
\usepackage{fancyhdr}
\usepackage{cleveref}
\usepackage{subcaption}
\usepackage{wrapfig}
\usepackage[toc, page]{appendix}
\newtheorem{lemma}{Lemma}%[section] %%    with section number.
\newtheorem{definition}{Definition}%[section]
\newtheorem{theorem}{Theorem}

\theoremstyle{remark}

\newcommand{\e}{\begin{equation}}
\newcommand{\ee}{\end{equation}}
\newcommand{\en}{\begin{equation*}}
\newcommand{\een}{\end{equation*}}
\newcommand{\eqn}{\begin{eqnarray}}
\newcommand{\eeqn}{\end{eqnarray}}
\newcommand{\bmat}{\begin{bmatrix}}
\newcommand{\emat}{\end{bmatrix}}
\DeclareMathAlphabet\mathbfcal{OMS}{cmsy}{b}{n}
\newcommand{ \Brac }[1]{\left\lbrace #1 \right\rbrace}

\newcommand{ \paren }[1]{ \left( #1 \right) }

\newcommand{\mb}{\bm}
\newcommand{\mc}{\mathcal}

\newcommand{\bb}{\mathbb}

\newcommand{\ol}{\overline}

\newcommand{\norm}[2]{\left\| #1 \right\|_{#2}}

\newcommand{\calL}{\mathcal{L}}

\hypersetup{
    colorlinks=true,%
    citecolor=blue,%
    filecolor=blue,%
    linkcolor=blue,%
    urlcolor=blue
}

\newcommand{\NC}{$\mc {NC}$}

\graphicspath{{./figs/}}

\newlength{\imgwidth}
\newboolean{twoColVersion}
\setboolean{twoColVersion}{false}
\newcommand{\twoCol}[2]{\ifthenelse{\boolean{twoColVersion}} {#1} {#2} }

\long\def\comment#1{}

\newcommand{\bbb}{\boldsymbol{b}}

\newcommand{\bh}{\boldsymbol{h}}

\newcommand{\bu}{\boldsymbol{u}}
\newcommand{\bv}{\boldsymbol{v}}
\newcommand{\bw}{\boldsymbol{w}}

\newcommand{\by}{\boldsymbol{y}}
\newcommand{\bz}{\boldsymbol{z}}

\newcommand{\bD}{\boldsymbol{D}}

\newcommand{\bH}{\boldsymbol{H}}

\newcommand{\bM}{\boldsymbol{M}}

\newcommand{\bP}{\boldsymbol{P}}

\newcommand{\bU}{\boldsymbol{U}}

\newcommand{\bW}{\boldsymbol{W}}

\newcommand{\bY}{\boldsymbol{Y}}
\newcommand{\bZ}{\boldsymbol{Z}}

\long\def\red#1{\bgroup\color{red}#1\egroup}

\definecolor{mich-blue}{HTML}{0027CC}
\definecolor{mich-blue-high}{HTML}{0027CC}
\definecolor{red-high}{HTML}{CA2020}
\definecolor{green-high}{HTML}{20A520}
\definecolor{mich-maize}{HTML}{FFCB05}
\definecolor{law-stone}{HTML}{655A52}
\definecolor{burton-beige}{HTML}{9B9A9D}
\definecolor{arch-ivy}{HTML}{7E732F}

 \colorlet{color1}{gray!15}

\newcommand{\MLab}{\texttt{M-lab}}
\newcommand{\MClf}{\texttt{M-clf}}

\newcommand\numberthis{\addtocounter{equation}{1}\tag{\theequation}}

\usepackage[top=1in, bottom=1in, left=1in, right=1in]{geometry}
\usepackage{tgpagella}
\usepackage[utf8]{inputenc} % allow utf-8 input
\usepackage[T1]{fontenc}    % use 8-bit T1 fonts
\usepackage{MnSymbol,bbding,pifont}
\usepackage{authblk}
\usepackage{booktabs}
\usepackage{cleveref}
\usepackage{cite}
\usepackage{wrapfig, lipsum}
\usepackage{xcolor, soul}
\sethlcolor{green}
\title{Neural Collapse in Multi-label Learning with \\  Pick-all-label Loss}

\author{Pengyu Li\thanks{The first two authors contributed equally to the work.}${}^{\hphantom{*},1}$, Xiao Li${}^{*, 1}$, Yutong Wang${}^{1,2}$, Qing Qu${}^{1,2}$
\\ 
\({}^{1}\)Department of Electrical Engineering \& Computer Science\\
\({}^{2}\)Michigan Institute for Data Science\\ University of Michigan}

\begin{document}

\maketitle

\begin{abstract}
We study deep neural networks for the multi-label classification (\MLab) task through the lens of neural collapse (NC).  Previous works have been restricted to the multi-class classification setting and discovered a prevalent NC phenomenon comprising of the following properties for the last-layer features:  (i) the variability of features within every class collapses to zero, (ii) the set of feature means form an equi-angular tight frame (ETF), and (iii) the last layer classifiers collapse to the feature mean upon some scaling. We generalize the study to multi-label learning, and prove for the first time that a generalized NC phenomenon holds with the ``pick-all-label'' formulation, which we term as \MLab~NC. While the ETF geometry remains consistent for features with a single label, multi-label scenarios introduce a unique combinatorial aspect we term the "tag-wise average" property, where the means of features with multiple labels are the scaled averages of means for single-label instances. Theoretically, under proper assumptions on the features, we establish that the only global optimizer of the pick-all-label cross-entropy loss satisfy the multi-label NC.
In practice, we demonstrate that our findings can lead to better test performance with more efficient training techniques for \MLab~learning.

\end{abstract}

\tableofcontents

\section{Introduction}\label{sec:intro}

In recent years, deep learning showed tremendous success in classifying problems \cite{lecun2015deep}, thanks in part to its ability to extract salient features from data \cite{bengio2013representation}.
% The phenomenon of neural collapse has proven to be a useful theoretical framework in gaining a better understanding of the features extracted by deep learning models and improving their performances under challenging settings such as highly imbalanced data \cite{thrampoulidis2022imbalance}.
While the success extends to \textit{multi-label} (\MLab) \emph{classification}, the structures of the learned features in the \MLab~regime is less well-understood.
This work aims to fill this gap by understanding the geometric structures of features for \MLab~learned via deep neural networks and utilize the structure for better training and prediction.

%\qq{yutong can add more here} 
Recently, an intriguing phenomenon has been observed in the terminal phase of training overparameterized deep networks for the task of \textit{multi-class} (\MClf) \emph{classification} in which the last-layer features and classifiers collapse to simple but elegant mathematical structures: all training inputs are mapped to class-specific points in feature space, and the last-layer classifier converges to the dual of class means of the features while attaining the maximum possible margin with a simplex equiangular tight frame (Simplex ETF) structure  \cite{papyan2020prevalence}.  See the top row of \Cref{fig:multi-label-illustration} for an illustration. This phenomenon, termed \emph{Neural Collapse} (NC), persists across a variety of different network architectures, datasets, and even the choices of losses \cite{han2022neural,zhou2022all,zhou2022optimization,yaras2022neural}. %\YW{what does problem formulation mean? not clear. the choice of loss?}. 
The NC phenomenon has been widely observed and analyzed theoretically \cite{papyan2020prevalence,fang2021exploring,zhu2021geometric} in the context of \MClf~learning problems. It is applied to understand transfer learning \cite{galanti2022role,li2022principled}, and robustness \cite{papyan2020prevalence,ji2022unconstrained}, where the line of study has significantly advanced our understanding of representation structures for \MClf~using deep networks.

\begin{figure*}
    \centering
    \includegraphics[width=0.7\textwidth]{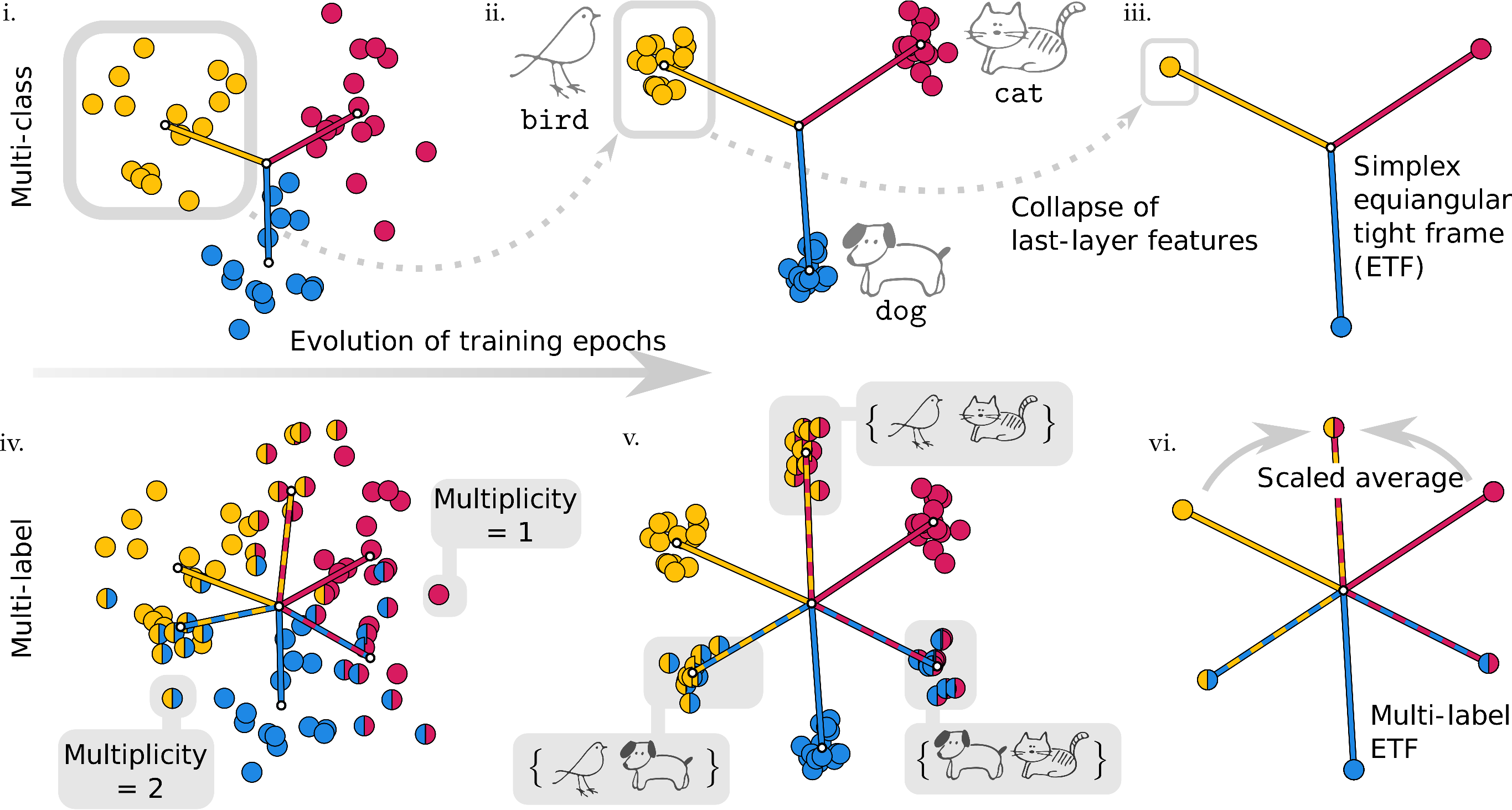}
    \caption{\small \textbf{An illustration of neural collapse for \MClf~(top row) vs. \MLab~(bottom row) learning}. 
    % \qq{Here, as the first figure, maybe we should show something that compares NC between \MLab vs MlClf}
    For illustrative purposes, we consider a simple setting with the number of classes \(K = 3\). The individual panels are scatterplots showing the top two singular vectors of the last-layer features \(\mb{H}\) at the beginning (left) and end (right) stages of training. The solid (resp.\ dashed) line segments represent the mean of the multiplicity \(=1\) (resp.\ \(=2\)) features with the same labels. 
    \emph{Panel i-iii}. As the training progresses, the last-layer features of samples corresponding a single label, e.g., $\mathtt{bird}$, collapse tightly around its mean.
    \emph{Panel iv-vi}. The analogous phenomenon holds in the multi-label setting.
\emph{Panel iv}.\ A training sample has multiplicity \(=1\) (resp.\ \(=2\)) if it has one tag (resp. two tags). 
\emph{Panel vi}.\ At the end stage of training, the feature mean of Multiplicity-2  
\(\{\mathtt{bird},\, \mathtt{cat}\}\)
is a scaled tag-wise average of feature means of its associated multiplicity-$1$ samples, i.e., 
\(\{\mathtt{bird}\}\) and \(\{\mathtt{cat}\}\).}
    \label{fig:multi-label-illustration}
 \vspace{-0.4cm}
\end{figure*}

\paragraph{Our contributions.} 

We demonstrate a general version of the NC phenomenon in \MLab, and our study provides new insights into prediction and training for the \MLab~problem. Specifically, our contributions can be highlighted as follows.

\begin{itemize}[leftmargin=*]
    % \item \textbf{Multi-label neural collapse in real world multilabel learning.}  
     \item \textbf{Multi-label neural collapse phenomenon.}  We show that the last-layer features and classifier learned via overparameterized deep networks exhibit a more general version of NC which we term it as \emph{multi-label neural collapse} (\MLab~NC). 
     Specifically, while features linked to single-label instances retain a Simplex ETF configuration and undergo collapse, the more complex features with higher label counts intriguingly represent a scaled "tag-wise average" of their single-label counterparts; see the bottom row of \Cref{fig:multi-label-illustration} for an illustration. This new pattern, referred to as \emph{multi-label ETF}, is consistently observed in the training of practical neural networks for \MLab~tasks.
    \vspace{-0.1in}
    \item \textbf{Global optimality of \MLab~NC.} Theoretically, we study the global optimality of \MLab~NC based upon a commonly used pick-all-label loss for \MLab~learning. By treating the last-layer feature as free optimization variables \cite{fang2021exploring,zhu2021geometric}, we show that all global solutions exhibit the properties of \MLab~NC with benign global landscape. Moreover, we show that multi-label ETF only requires balanced training samples in each class within \emph{the same} multiplicity, and \emph{allows class imbalanced-ness} across different multiplicities.\footnote{We also empirically demonstrate that \MLab~NC  occurs even when only multiplicity-1 data are balanced (see \Cref{fig:NC_High_imba_c10_Mnist} for an illustration).}
\vspace{-0.1in}
    \item \textbf{\MLab~NC guided prediction and training.} In practice, we show that our findings lead to improved prediction and training for \MLab~learning. For prediction, we propose a new one-nearest-neighbor (ONN) approach: for the given test sample, we assign its label to the tag of the nearest neighboring multi-label ETF in the feature space. Compared to the classical one-vs-all (OvA) approach for \MLab~learning, ONN is much more efficient with higher test accuracy. For training, by fixing the last later to be ETF and reducing the feature dimension, our experimental results demonstrate that we can sufficiently reduce parameters without compromising overall performance.
\end{itemize}

%The features vector from the penultimate layer in a \MLab problem exhibits an equi-angular tight frame (ETF) structure as well as convex combination of the ETF.
%This geometry is an generalization of the ETF from the analysis of NC in the \MClf literature.
%We also see phenomenon that is unique to the \MLab setting that does not show up in the \MClf setting.

%\subsection{Relationship to Prior Arts}

%\qq{needs to condense below, we could only discuss the works most related to ours, and leave all other less related works to the appendices}
%\yw{done}

\paragraph{Related works on multi-label learning.}

In contrast to \MClf, where each sample has a single label, in \MLab~the samples are tagged with multiple labels. As such, the final model output must be set-valued. This presents theoretical and practical challenges unique to the regime of \MLab~learning, especially when the class number is large \cite{liu2021emerging}. For \MLab~learning, a commonly employed strategy is to decompose a given multi-label problem into several binary classification tasks \cite{menon2019multilabel}.  The ``one-versus-all'' (OvA) method, also known as \emph{binary relevance} (BR), involves splitting the task into several binary classification ``subtasks''. Each of these subtasks requires  training an separate binary classifier,  one for each label \cite{moyano2018review, brinker2006unified}. During testing, thresholding is used to convert the (real-valued) outputs of the model (i.e., the logits) into tags.

Although BR is simple to implement and performs well, a notable limitation is the lack of handling of dependency between labels \cite{cheng2010bayes}. The Pick-All-Label (PAL) approach  \cite{reddi2019stochastic, menon2019multilabel} can formulates the multi-label classification task in an ``all-in-one'' manner. Compared to OvA, the PAL approach has the advantage of not needing to train separate classifiers. However, its disadvantage is that during prediction, there is no natural thresholding rule for converting the logits to tags. In this work, we deal with this challenge by proposing the ONN technique as a principled approach, guided by our geometric analysis of \MLab, for performing prediction. 

To the best of our knowledge, no work has previously analyzed the geometric structure in multi-label deep learning. Our work closes this gap by providing a generalization of the NC phenomenon to \MLab~learning, and furthermore develops efficient prediction and training technique via our theoretical understandings.

\paragraph{Related works on neural collapse.} 

The phenomenon known as NC was initially identified in recent groundbreaking research \cite{papyan2020prevalence,han2022neural} conducted on \MClf. These studies provided empirical evidence demonstrating the prevalence of NC across various network architectures and datasets. The significance of NC lies in its elegant mathematical characterization of learned representations or features in deep learning models for \MClf. Notably, this characterization is independent of network architectures, dataset properties, and optimization algorithms, as also highlighted in a recent review paper \cite{kothapalli2023neural}. Subsequent investigations, building upon the "unconstrained feature model" \cite{mixon2020neural} or the "layer-peeled model" \cite{fang2021exploring}, have contributed theoretical evidence supporting the existence of NC. This evidence pertains to the utilization of a range of loss functions, including cross-entropy (CE) loss \cite{lu2020neural,zhu2021geometric,fang2021exploring,yaras2022neural}, mean-square-error (MSE) loss \cite{mixon2020neural,zhou2022optimization,tirer2022extended,rangamani2022neural,wang2022linear,dang2023neural}, and CE variants \cite{graf2021dissecting,zhou2022all}. More recent studies have explored other theoretical aspects of NC, such as its relationship with generalization \cite{hui2022limitations,galanti2022role,galanti2022generalization,galanti2022note,chen2022perfectly}, its applicability to large classes \cite{liu2023generalizing,gao2023study,jiang2023generalized}, and the progressive collapse of feature variability across intermediate network layers \cite{hui2022limitations,papyan2020traces,he2022law,yaras2023law,rangamani2023feature}. Theoretical findings related to NC have also inspired the development of new techniques to improve practical performance in various scenarios, including the design of loss functions and architectures \cite{yu2020learning,zhu2021geometric,chan2022redunet}, transfer learning \cite{li2022principled, xie2022hidden, wang2023far} where a model trained on one task or dataset is adapted or fine-tuned to perform a different but related task, imbalanced learning \cite{fang2021exploring,xie2023neural,lu2022importance,yang2022inducing,thrampoulidis2022imbalance,behnia2023implicit,zhong2023understanding,sharma2023learning} which is a characteristic of dataset where one or more classes have significantly fewer instances compared to other classes, and continual learning \cite{yu2023continual,yang2023neural,zhai2023investigating} in which a model is designed to learn and adapt to new data continuously over time, rather than being trained on a fixed dataset.

\paragraph{Basic notations.}
Throughout the paper, we use bold lowercase and upper letters, such as $\mb a$ and $\mb A$, to denote vectors and matrices, respectively. Non-bold letters are reserved for scalars.  For any matrix $\mb A \in \bb R^{n_1 \times n_2}$, we write $\mb A = \begin{bmatrix} \mb a_1 & \dots & \mb a_{n_2} \end{bmatrix}$, so that $\mb a_i$ ($i \in \{1,\dots, n_2\}$) denotes the \(i\)-th column of $\mb A$.
Analogously,  we use the superscript notation to denote rows, i.e., \((\mb a^j)^\top\) is the \(j\)-th row of \(\mb A\) for each \(j \in \{1,\dots, n_1\}\) with \(\mb A^\top = \begin{bmatrix}
    \mb a^1 &\dots & \mb a^{n_1}
\end{bmatrix}\).
For an integer \(K > 0\), we use $\mb I_K$ to denote an identity matrix of size $K\times K$, and we use $\mb 1_K$ to denote an all-ones vector of length $K$.

\paragraph{Paper organization.} The rest of the paper is organized as follows. In \Cref{sec:problem}, we lay out the basic problem formulations. In \Cref{sec:result}, we present our main results and discuss the implications. In \Cref{sec:exp}, we verify our theoretical findings and demonstrate the practical implications of our result. Finally, we conclude in \Cref{sec:conclusion}. All the technical details are postponed to the Appendices.
For reproducible research, the code for this project can be found at 
\begin{center}
    \url{https://github.com/Heimine/NC_MLab}
\end{center}

%\qq{adding a paragraph on paper organization here}

%See \Cref{table:notations} in \Cref{sec:table-of-notations} for a list of notations used throughout.

\section{Problem Formulation}\label{sec:problem}
We start by reviewing the basic setup for training deep neural networks and later specialize to the problem of \MLab~ with $K$ number of classes.
Given a labelled training instance $(\mb x , \mb y)$, the goal is to learn the network parameter \(\mb \Theta\) to fit the input $\mb x$ to the corresponding training label $\mb y$ such that 
\begin{align*}\label{eq:func-NN}
   \mb y \; &\approx \; \psi_{\mb \Theta}(\mb x) \;=\;  \underset{ \text{\bf linear classifier}\;\mb W }{ \mb W_L} \;  \cdot \; \underset{\text{ \bf feature}\;\; \mb h\;=\; \phi_{\mb \theta}(\mb x)}{\sigma\paren{ \mb W_{L-1} \cdots \sigma \paren{\mb W_1 \mb x \mb  + \mb b_1} + \mb b_{L-1} }}  + \mb b_L, \numberthis
\end{align*}

where $\mb W = \mb W_L$ represents the last-layer linear classifier and $\mb h(\mb x)=\phi_{\mb \theta}(\mb x_{k,i})$ is a deep hierarchical representation (or feature) of the input $\mb x$. For a $L$-layer deep network $\psi_{\mb \Theta}(\mb x)$, each layer is composed of an affine transformation, followed by a nonlinear activation $\sigma(\cdot)$ (e.g., ReLU) and 
normalization (e.g., BatchNorm \cite{ioffe2015batch}).

\paragraph{Notations for multi-label dataset.} %\label{sec:data-notation}
Let $[K]:=\Brac{1,2,\dots,K}$ denote the set of labels. 
For each \(m \in [K]\), let \(\binom{[K]}{m} := \{ S \subseteq [K]: |S| = m\}\) denote the set of all subsets of \([K]\) with size \(m\). Throughout this work, we consider a fixed multi-label training dataset of the form $\Brac{\mb x_i,\mb y_{S_i}}_{i=1}^N$, where $N$ is the size of the training set and \(S_i\) is a nonempty proper subset of the labels. For instance,  \(S_i = \{\mathtt{cat}\}\) and \(S_{i'} = \{\mathtt{dog}, \mathtt{bird}\}\).
Each label $\mb y_{S_i} \in \bb R^K$ is a \emph{multi-hot-encoding} vector:
\begin{equation}
    \mbox{\(j\)-th entry of }\mb y_{S_i} 
    =
    \begin{cases}
        1 &: j \in S_i \\
        0 &: \mbox{otherwise}.
    \end{cases}
\end{equation}

The \emph{Multiplicity} of a training sample \((\mb x_i, \mb y_{S_i})\) is defined as the cardinality of \(S_i\), i.e., the number of labels or tags that is related to \(\mb x_i\). We refer to a feature learned for the sample \((\mb x_i, \mb y_{S_i})\) as the Multiplicity-m feature if \(|S_i| = m\). As such, with the abuse of notation, we also refer to such $\mb x_i$ as a Multiplicity-m sample and such $\mb y_{S_i}$ as a Multiplicity-m label, respectively. Note that a multiplicity-m label is a multi-hot label that can be decomposed as a summation of one-hot multiplicity-1 labels. For example 
\(S_{i'} = \{\mathtt{dog}, \mathtt{bird}\}\) has two tags $S_{j'} = \{\mathtt{dog}\}$ and $S_{k'} = \{\mathtt{bird}\}$, and then the corresponding $2$-hot label can be decomposed into the associated $1$-hot labels of $S_{j'} = \{\mathtt{dog}\}$ and $S_{k'} = \{\mathtt{bird}\}$. In this work, we show the relationship of labels can be generalized to study the relationship of the associated features trained via deep networks through \MLab~NC.

%consist of $m$ multiplicity-1 label component tags, i.e. \(S_{i'} = \{\mathtt{dog}, \mathtt{bird}\}\) has two tags $S_{j'} = \{\mathtt{dog}\}$ and $S_{k'} = \{\mathtt{bird}\}$. The component tag relationship of labels extends to learned features. One question we aim to answer with  would be how to mathematically describe such relationship between higher multiplicity features and its tagged component multiplicity-1 features.
 
 The Multiplicity-m feature matrix $\bH_m$ is column-wise comprised of a collection of Multiplicity-m feature vectors. Moreover, we use  \(M := \max_{i \in [N]} |S_i|\) to denote the largest multiplicity in the training set. To distinguish imbalanced class samples between Multiplicities, for each \(m \in [M]\), we use \(n_m := | \{i \in [N]: |S_i| = m\}|\) to denote the number of samples in each class of a multiplicity order \(m\) (or Multiplicity $m$). Note that \(M \in \{1,\dots, K-1\}\) in general, and a \MLab~problem reduces to \MClf~when \(M=1\).

% For convenience, we let \(\kappa \in \binom{K}{m}\) be 

% \paragraph{Training with the pick-all-labels loss under UFM.}
\paragraph{The ``pick-all-labels'' loss.}
Since \MLab~is a generalization of \MClf, recent work \cite{menon2019multilabel} studied various ways of converting a \MClf~loss into a \MLab~loss, a process referred to as \emph{reduction}.\footnote{``Reduction'' refers to reformulating \MLab~problems in the simpler framework of \MClf~problems.}
% \qq{discuss the benefits of such an approach here: talk about other approaches very briefly cite papers, but this is simple and default, effective perform well}
In this work, we analyze the \emph{pick-all-labels} (PAL) method of reducing the cross-entropy (CE) loss to a \MLab~loss, which is the \emph{default} option implemented by \texttt{torch.nn.CrossEntropyLoss} from the deep learning library PyTorch \cite{paszke2019pytorch}. 
The benefit of PAL approach is that the difficult \MLab~problem can be approached using insights from \MClf~learning using well-understood losses such as the CE loss, which is one of the most commonly used loss functions in classification:
\begin{align*}
    \mc L_{\mathrm{CE}}(\mb z, \mb y_k) \;:=\; - \log \paren{ \exp(z_k) / \textstyle\sum_{\ell=1}^K \exp(z_{\ell})  }.
\end{align*}
where $\mb z = \mb W \mb h$ is called the logits, and $\mb y_k$ is the one-hot encoding for the $k$-th class. To convert the CE loss into a \MClf~loss via the PAL method, for any given label set $S$, consider decomposing a multi-hot label $\mb y_S$ as a summation of one-hot labels: $\mb y_S = \sum_{k \in S} \mb y_k$. Thus, we can define the \emph{pick-all-labels cross-entropy} (PAL-CE) loss as 
\begin{align*}
    \underline{\calL}_{\mathtt{PAL}-\mathrm{CE}}(\mb z,\mb y_S) \;:=\; \textstyle \sum_{ k \in S } \mc L_{\mathrm{CE}}(\mb z, \mb y_k).
\end{align*}
In this work, we focus exclusively on the CE loss under the PAL framework, we simply write
\(\underline{\calL}_{\mathtt{PAL}}\) to denote \(\underline{\calL}_{\mathtt{PAL}-\mathrm{CE}}\). However, by drawing inspiration from recent research \cite{zhou2022all}, it should be noted that 
under the PAL framework, the phenomenon of \MLab~NC can be generalized beyond cross-entropy to encompass a variety of other loss functions used for \MClf~learning, such as mean squared error (MSE), label smoothing (LS),\footnote{The loss replaces hard targets in CE with smoothed ones to achieve better calibration and generalization \cite{szegedy2016rethinking}.} focal loss (FL),\footnote{The loss adjusts its focus to less on the well-classified samples, enhances calibration, and establishes a curriculum learning framework \cite{lin2017focal, mukhoti2020calibrating, smith2022cyclical}.} and potentially a class of Fenchel-Young Losses that unifies many well-known losses \cite{blondel2020learning}. 

Putting it all together, training deep neural networks for \MLab~learning can be stated as follows:
\begin{align}\label{eq:DL-loss}
    \min_{ \mb \Theta} \tfrac{1}{N} \textstyle \sum_{i=1}^N
    \underline{\calL}_{\mathtt{PAL}}( \mb W \phi_{\mb \theta}(\mb x_i) + \mb b,\mb y_{S_i}) + \lambda \norm{\mb \Theta}{F}^2,
\end{align}
where $\mb \Theta = \Brac{ \mb W,\mb b,\mb \theta }$ denote all parameters and $\lambda>0$ controls the strength of weight decay. 
%\qq{expand the discussion on pick-all-label loss, generalize to other losses} 
Here, weight decay prevents the norm of the linear classifier and the feature matrix goes to infinity or $0$.

\paragraph{Optimization under the unconstrained feature model (UFM).} Analyzing the nonconvex loss \eqref{eq:DL-loss} can be notoriously difficult due to the highly non-linear characteristic of the deep network $\phi_{\mb \theta}(\mb x_i)$. In this work, we simplify the study by treating the feature $\mb h_i =  \phi_{\mb \theta}(\mb x_i)$ of each input $\mb x_i$ as a \emph{free} optimization variable. Analysis of NC under UFM has been extensively studied in recent works \cite{zhu2021geometric,fang2021exploring,ji2022unconstrained,yaras2022neural,mixon2020neural,zhou2022optimization,tirer2022extended}, the motivation behind the UFM is the fact that modern networks are highly overparameterized and they are 
 universal approximators \cite{cybenko1989approximation,zhang2021understanding}. More specifically, we study the following problem. 
\begin{definition}[Nonconvex Training Loss under UFM]\label{definition:UFM}
Let \(\mb{Y} = [\mb y_{S_1} \cdots \mb y_{S_N}] \in \mathbb{R}^{K \times N}\) be the multi-hot encoding matrix whose \(i\)-th column is given by the multi-hot vector \(\mb y_{S_i} \in \mathbb{R}^K\). We consider
\begin{align*}\label{eq:DL-loss-ufm}
    \min_{ \mb W , \mb H,\mb b} f(\mb W,\mb H,\mb b) \;:=\; g( \mb W \mb H + \mb b,\mb Y) + \lambda_W \norm{\mb W}{F}^2 + \lambda_H \norm{\mb H}{F}^2 + \lambda_b \norm{\mb b}{2}^2 \numberthis
\end{align*}
with the penalty $\lambda_W,\lambda_H,\lambda_b>0$.

Here, the linear classifier \(\mb W \in \mathbb{R}^{K \times d}\), the features \(\mb H = [\mb h_1, \cdots, \mb h_N] \in \mathbb{R}^{d \times N}\), and the bias \(\mb b \in \mathbb{R}^K\) are all unconstrained optimization variables, and we refer to the columns of \(\mb H\), denoted \(\mb h_i\), as the \emph{unconstrained last layer features} of the input samples \(\mb x_i\). Additionally, the function $g(\cdot)$ is the PAL loss, denoted by
\begin{align*}
g(\mb{W}\mb{H}+\mb{b},\mb Y) &:= 
\tfrac{1}{N}\underline{\calL}_{\mathtt{PAL}}(\mb{W}\mb{H} + \mb{b},\mb{Y}) \; :=\; \tfrac{1}{N} \textstyle \sum_{i=1}^N \underline{\calL}_{\mathtt{PAL}}(\mb{W}\mb{h}_i + \mb{b},\mb{y}_{S_i}).   \end{align*}
\end{definition}

 Although the objective function is seemingly a simple extension of \MClf~case, our work shows that the global optimizers of Problem \eqref{eq:DL-loss-ufm} for \MLab~learning substantially differs from that of the \MClf~ that we present in the following.

\section{Main Results}\label{sec:result}

In this section, we show that the global minimizers of Problem \eqref{eq:DL-loss-ufm} exhibit a more generic structure than the vanilla NC in \MClf~(see \Cref{fig:multi-label-illustration}), where higher multiplicity features are formed by a scaled tag-wise average of associated Multiplicity-$1$ features that we introduce in detail below. Theoretically, we rigorously analyze the global geometry of the optimizer of Problem \eqref{eq:DL-loss-ufm} and its nonconvex optimization landscape, and present our main results in \Cref{thm:GO_thm}. 

% \textcolor{red}{High light the guided multi-label learning}

%In this section, we present our main results in Theorem~\ref{thm:GO_thm} and, which characterizes the global optimizers \((\bW^\star, \bH^\star, \mb{b}^\star)\) of \eqref{eq:DL-loss} as a \emph{multi-label equiangular tight frame}, or \MLab~ETF.
% Let \(\mb h_i^\star\) be the columns of \(\bH\), where \(i\) indexes the training sample. 
% Similar to the \MClf~case, landscape of the non-convex optimization in (\ref{eq:DL-loss}) is benign, a result which we prove in Theorem~\ref{thm:optim_landscape}.
% However unlike \MClf, in the \MLab~regime the columns of $\bH^\star$ no longer form an ETF. 
% Rather, only the subset of columns of \(\bH^\star\) corresponding to samples of multiplicity 
% \(=1\) forms an ETF.
% Morever, the subset of columns of \(\bH^\star\) corresponding to samples of higher multiplicity 
% are scaled averages of the multiplicity 
% \(=1\) columns. 

\subsection{Multi-label Neural Collapse (\MLab~NC)}\label{subsec:MLab-NC}
%Motivated by the above, we investigate whether the optimizers of \eqref{eq:DL-loss} truly reflects our intuition under the UFM assumption.
% We now formally define the  \textbf{multi-label equiangular tight frame} (\MLab~ETF).
%\vspace{-0.1in}
We assume that the training data is balanced with respect to Multiplicity-1 while high-order multiplicity is imbalanced or even has missing classes. Through empirical investigation, we discover that when a deep network is trained up to the terminal phase using the objective function \eqref{eq:DL-loss}, it exhibits the following characteristics, which we collectively term as "multi-label neural collapse" (\MLab~NC):
%Empirically, we find that an overparameterized neural network trained on a Multiplicity-1 balanced data using the objective \eqref{eq:DL-loss} to the terminal phase satisfies the properties below which we collectively refer to as \textbf{multi-label neural collapse} (\MLab~NC): %The properties that differs from the previously known \MClf~case \cite{papyan2020prevalence,zhu2021geometric} are marked by an ``\((\ast)\)'':
% \py{only degree 1 balance here, but later themoren assume all balance}
% \py{define what is Multiplicity-1}
%\vspace{-0.1in}
\begin{enumerate}[leftmargin=*]
    \item{\textbf{{Variability collapse}:}} The within-class variability of last-layer features across different multiplicities and different classes all collapses to zero. In other words, the individual features of each class of each multiplicity concentrate to their respective class means.
   \item{(\(\ast\)) \textbf{{Convergence to self-duality of multiplicity-$1$ features $\mb H_1$ }:}} The rows of the last-layer linear classifier $\mb W$ and the class means of Multiplicity-$1$ feature $\mb H$ are collinear, i.e., \(\bh^\star_i \propto \bw^{\star k} \) when the label set  \( S_i = \{k\}\) is a singleton set. %\LP{describe self-duality, Hessian, etc. in appendix}
    \item{(\(\ast\)) \textbf{{Convergence to the \MLab~ETF}:}} Multiplicity-1 features $\mb H_1 := \big\{\mb h_i^\star | {i : |S_i| = 1}\big\}$ form a Simplex Equiangular Tight Frame, similar to the \MClf~setting \cite{papyan2020prevalence,fang2021exploring,zhu2021geometric}. 
    %Specifically, the class-means for Multiplicity-$1$ features $\mb H_1$ are (\emph{i}) centered at the origin, (\emph{ii}) maximally distant from each other, and (\emph{iii}) linearly separably. 
    Moreover, for any higher multiplicity $m>1$, \emph{the average feature means for classes with label count $m$ are a scaled, tag-wise aggregation of the corresponding single-label ($Multiplicity-1$) feature means across the relevant label set.} In other words, \(\bh^\star_i \propto \sum_{k \in S_i} \bw^{\star k}\) (see the bottom line of \Cref{fig:multi-label-illustration}). This is true regardless of class imbalance between multiplicities. %\qq{imbalanced-ness}
\end{enumerate}

\paragraph{Remarks.} The \MLab~NC can be viewed as a more general version of the vanilla NC in \MClf~\cite{papyan2020prevalence}, where we mark the difference from the vanilla NC above by a ``\((\ast)\)''. The \MLab~ETF implies that, in the pick-all-labels approach to multi-label classification, deep networks learn discriminant and informative features for Multiplicity-$1$ subset of the training data, and use them to construct higher multiplicity features as the tag-wise average of associated Multiplicity-$1$ features. To quantify the collapse of high multiplicity NC, we introduce a new measure \(\mathcal{NC}_m\) in \Cref{sec:exp} and demonstrate that it collapses for practical neural networks during the terminal phase of training. This result is intuitive: since the multi-hot label vector can be decomposed into the sum of its tag-wise one-hot vectors, the corresponding learned features may exhibit a similar scaled tag-average phenomenon.

% Such a result is quite intuitive. For example, consider a sample \(i \in [N]\) whose training label \(\mb y_{S_i}\) has Multiplicity-$2$, e.g., \(S_i = \{\mbox{\texttt{cat}}, \mbox{\texttt{dog}}\}\).
% The  multi-hot  vector label \(\mb y_{S_i}\)  decomposes as a sum of one-hot labels of Multiplicity-$1$, namely, $\mb{y}_{S_i} = \sum_{k \in S_i} \mb{y}_k$. Ideally, the learned representation \(\mb h_i^\star\)  should satisfy such a property as well: that \(\mb h_i^\star\) is a scaled tag-wise average of several \(\mb h_{i'}^\star\) components where each \(i' \in [N]\) corresponds to an training instance of Multiplicity-\(1\) label. The learned representation of an image containing or tagged with both \texttt{cat} and \texttt{dog} should be a scaled average of the learned representation of images containing or tagged with only a \texttt{cat} or a \texttt{dog}. 

Moreover, in the case of data imbalancend-ness, we find that the \MLab~NC~holds as long as the training samples within the same multiplicity are required to be class balanced, and the number of samples between multiplicities does \emph{not} need to be balanced. This can be later confirmed by our theory in \Cref{subsec:thm}. For example, the \MLab~NC still holds if there are more or less training samples for the category $\mathtt{(ant,bee)}$(Multiplicity-2) than that of $\mathtt{(cat,dog,elk)}$ (Multiplicity-3).

\subsection{Global Optimality \& Benign Landscape Under UFM}

In this subsection, we first present our major result by showing that M-lab NC achieves the global optimality to the nonconvex training loss in \eqref{eq:DL-loss-ufm} and discuss its implications. Second, we show that the nonconvex landscape is also benign \cite{zhang2020symmetry}.

\subsubsection{Global Optimality of \MLab~NC}\label{subsec:thm}
%For deriving our main result Theorem~\ref{thm:GO_thm}, we need the following
 %   \emph{data balanced-ness condition}: In the notations of Section~\ref{sec:data-notation}, recall that
 %     \(M := \max_{i \in [N]} |S_i|\) is the largest label multiplicity. 
 %    
% For each \(m \in [M]\), we assume that there exists an integer \(n_m >0\) so that for all possible label set \(T\) of size \(m\), i.e., \(T \in \binom{[K]}{m}\), we have that
%      \(
%n_m = |\{ i : S_i = T\}|
%      \).

%In this subsection, we first present our major result by showing that \MLab~NC achieves the global optimality to the nonconvex training loss in Problem \eqref{eq:DL-loss-ufm} with benign landscape, and discuss its implications. %Second, we show that the nonconvex landscape is also benign \cite{zhang2020symmetry}.

%\subsubsection{Global optimality for \MLab~NC}
For \MLab, we show that  the \MLab~NC is the only global solution to the nonconvex problem in \Cref{definition:UFM}. %\st{For the ease of analysis, we consider the setting that the training data is balanced across all multiplicities and classes.} 
We consider the setting that the training data may exhibit imbalanced-ness between different multiplicities while maintaining class-balancedness within each multiplicity. 

%For instance, there could be $1000$ samples for each class in Multiplicity-1 labels, but only $500$ samples for each class within Multiplicity-2 labels, and so forth.

\begin{theorem}[\textbf{Global Optimality of \MLab~NC}] \label{thm:GO_thm}
In the setting of Definition~\ref{definition:UFM},
assume the feature dimension is no smaller than the number of classes, i.e., $d \ge K-1$, and assume the training are balanced within each multiplicity as we discussed above. Then any global optimizer \(\bW^\star, \bH^\star , \bbb^\star\) of the optimization Problem  \eqref{eq:DL-loss-ufm} satisfies:
\begin{align}
\textstyle
    w^\star := \|\bw^{\star1}\|_2 = \cdots = \|\bw^{\star K}\|_2, \quad \mbox{and} \quad \bbb^\star = b^\star \mb{1}, \label{eq:w_collapse}
\end{align}
where either $b^\star = 0$ or $\lambda_{\bbb} = 0$. Moreover, the global minimizer \(\bW^\star, \bH^\star , \bbb^\star\) satisfies the \MLab~NC properties introduced in \Cref{subsec:MLab-NC}, in the sense that 
\begin{itemize}[leftmargin=*]
    \item The linear classifier matrix $\bW^{\star \top} \in \mathbb{R}^{d \times K}$ forms a K-simplex ETF up to scaling and rotation, i.e., for any $\bU \in \mathbb{R}^{d \times d}$ s.t. $\bU^\top \bU = \mb{I}_d$, the rotated and normalized matrix $\bM := \frac{1}{w^\star}\bU \bW^{\star \top}$ satisfies
\begin{equation}
\textstyle
    \bM^\top\bM = \frac{K}{K-1}\left(\mb{I}_K - \frac{1}{K}\mb{1}_K \mb{1}_K^\top \right). \label{eq:etf_def}
\end{equation}
 \emph{Tag-wise average property}. For each feature $\bh^\star_i$ (i.e., the \(i\)-th column \(\bh^\star_i\) of $\bH^\star$) with $i\in[N]$, there exist unique positive real numbers \(C_1,C_2,\dots, C_M>0\) such that the following holds:  %\qq{does all need to be positive? or at least one or two?} \py{Based on our theorem, I think all $c_m$ are positive, that's the interesting part.} 
    %\qq{for the following, the position of the coefficient is not right?} \py{the coefficient is right, we could adjust our proof to adapt the following equations}
%    There exists   positive real numbers such that the following holds:
%The \(i\)-th column \(\bh^\star_i\) of $\bH^\star$ satisfies  
%\begin{align*}
%   \bh^\star_i \;=\; \begin{cases}
%       c_1 \bw^{\star k} & \text{if }S_i = \{k\}, k \in [K]\\
%       c_m \sum_{k \in S_i} \bw^{\star k} & \text{if }|S_i| > 1
%   \end{cases} 
%\end{align*}
\begin{align*} 
    &\bh^\star_i = C_1 \bw^{\star k} 
     \quad\qquad \mbox{when} \,\, S_i = \{k\},\;k\in [K], \ \ \qquad \mbox{(Multiplicity \(=1 \) Case)} \numberthis \label{eq:h1_collapse}  \\
    &\bh^\star_i = C_m  \sum_{k \in S_i} \bw^{\star k} 
    \ \  \mbox{when} \,\, |S_i| =m,\;1<m\leq M. \qquad \mbox{(Multiplicity \(>1\) Case)} \numberthis \label{eq:hm_collapse}
\end{align*}

\end{itemize}

Moreover, the function $f(\bW, \bH, \bbb)$ in Problem \eqref{eq:DL-loss-ufm} is a strict saddle function \cite{ge2015escaping,sun2015nonconvex,zhang2020symmetry} with no spurious local minimum.
\end{theorem}

%\qq{First, describe how the mathematical equation coincide with MLab-NC in the previous subsection. Second, describe the difference and the challenges in the proof} \py{below are discussion of how mathemetical equation coincide with \MLab~NC, challenges are discusses later in this section, feel free to move it}

% \py{Merge remark and technical contributions}
% \paragraph{Remarks.}

% Under the scaled-averaged property, the coefficient $C_m$ relates features across different multiplicities with a single classifier. Numerically, $C_m$ could be expressed assatisfies a set of non-linear equations (please refer to the Appendix), and it always exists due to the network's optimizability. This existence can be verified through numerical methods.

% While it may appear intuitive and straightforward to extend the analysis of vanilla NC in \MClf~to \MLab~NC, the combinatorial nature of high multiplicity features and the interplay between the linear classifier $\mb W$ and these class-imbalanced high multiplicity features present significant challenges for analysis. For instance, previous attempts to prove \MClf~NC utilized Jensen's inequality and the concavity of the logarithmic function, but these methods are not effective for \MLab~NC.
% Instead, we analyze the gradient of the pick-all-labels cross-entropy and leverage its strict convexity to directly construct the desired lower bound. \qq{duplication here}

%\LP{Also, extra Lemmas from a probabilistic perspective are added in order to proof the scaled-average property.} \qq{data imabalancedeness} 

We discuss the high-level ideas of the proof in \Cref{sec:proof}.
The detailed proof of our results is deferred to \Cref{app:optimality} and \Cref{app:landscape}.
Next, we delve into the implications of our findings from various perspectives.

\paragraph{The global solutions of Problem \eqref{eq:DL-loss-ufm} satisfy \MLab~NC.} Under the assumption of UFM, our findings imply that every global solution of the loss function of Problem \eqref{eq:DL-loss-ufm} exhibits the \MLab~NC that we presented in \Cref{subsec:MLab-NC}. First, feature variability within each class and multiplicity can be deduced from Equations \eqref{eq:h1_collapse} and \eqref{eq:hm_collapse}. This occurs because all features of the designated class and multiplicity align with the (tag-wise average of) linear classifiers, meaning they are equal to their feature means with no variability. Second, the convergence of feature means to the \MLab~ETF can be observed from Equations \eqref{eq:etf_def}, \eqref{eq:h1_collapse}, and \eqref{eq:hm_collapse}. For Multiplicity-$1$ features $\mb H_1^\star$, \Cref{eq:h1_collapse} implies that the feature mean $\ol{\mb H}_1^\star$ converges to $\mb W ^\star$; this, coupled with \Cref{eq:etf_def}, implies that the feature means $\ol{\mb H}_1^\star$ of Multiplicity-$1$ form a simplex ETF. Moreover, the structure of tag-wise average in \Cref{eq:hm_collapse} implies the \MLab~ETF for feature means of high multiplicity samples. Finally, the convergence of Multiplicity-$1$ features towards self-duality can be deduced from \Cref{eq:h1_collapse}.

\begin{figure*}[t]
    \centering
    \subfloat[]
    {\includegraphics[width=0.24\textwidth]{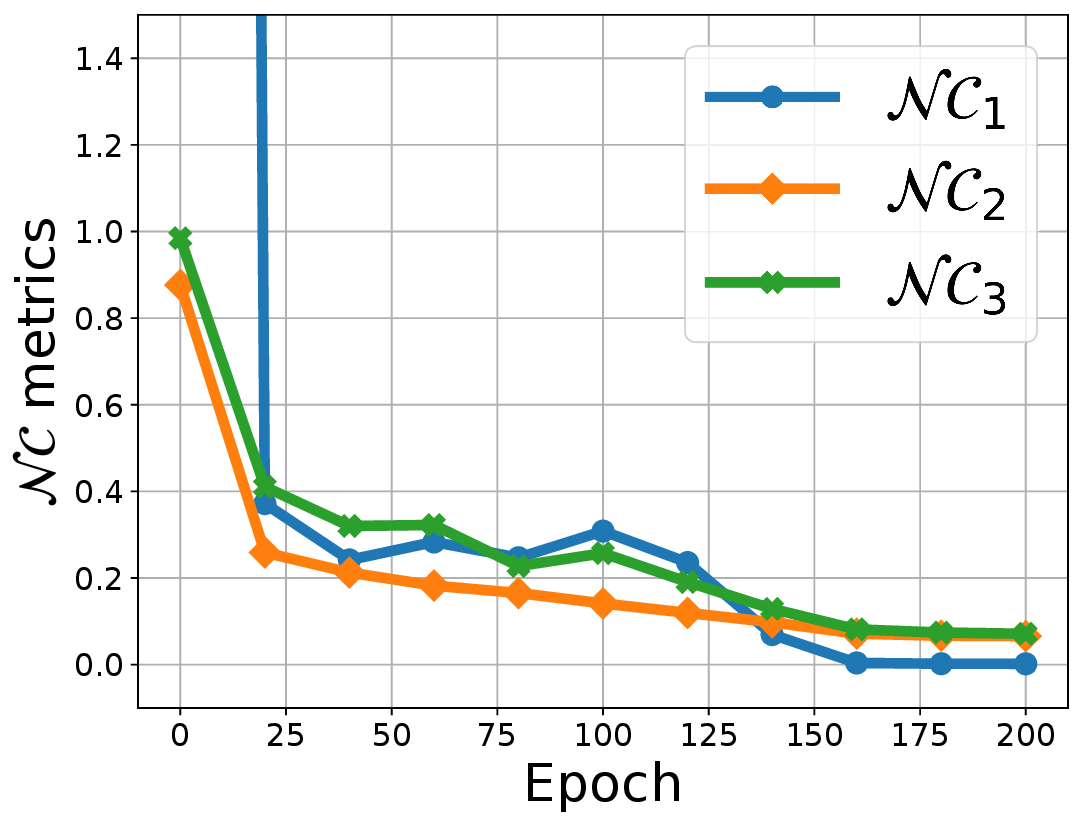}} 
    \subfloat[]{\includegraphics[width=0.24\textwidth]{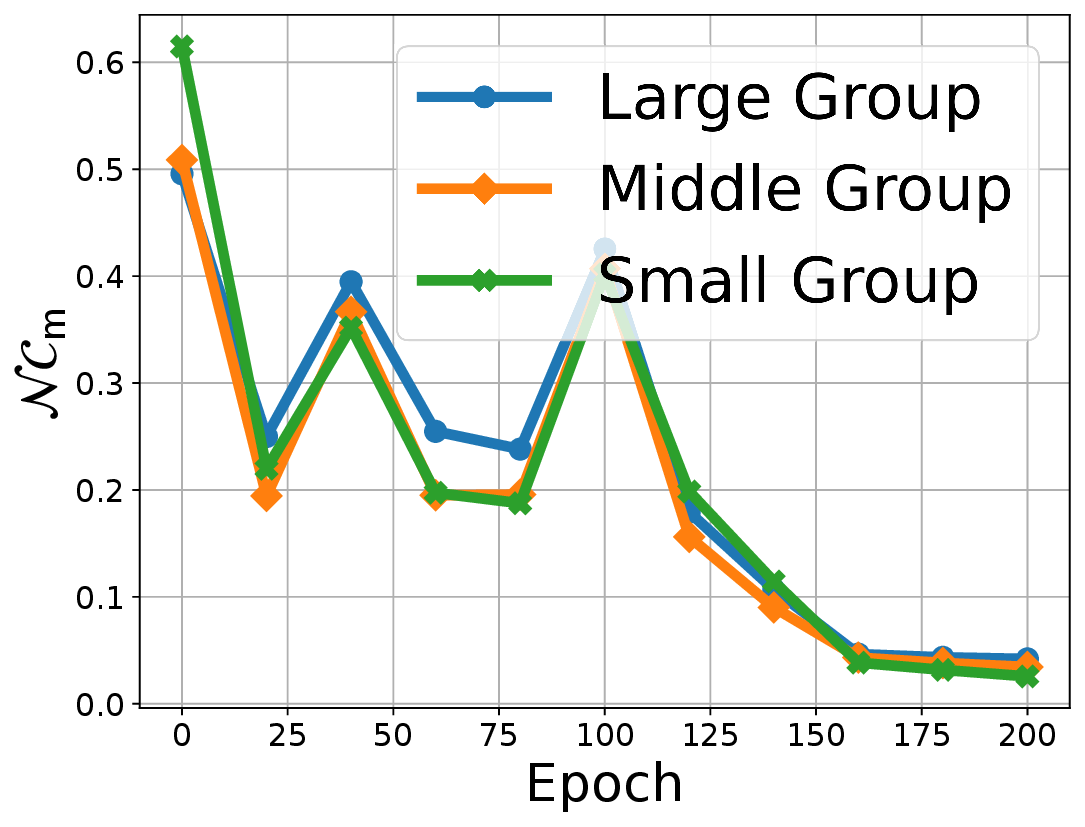}}
    \subfloat[]{\includegraphics[width=0.24\textwidth]{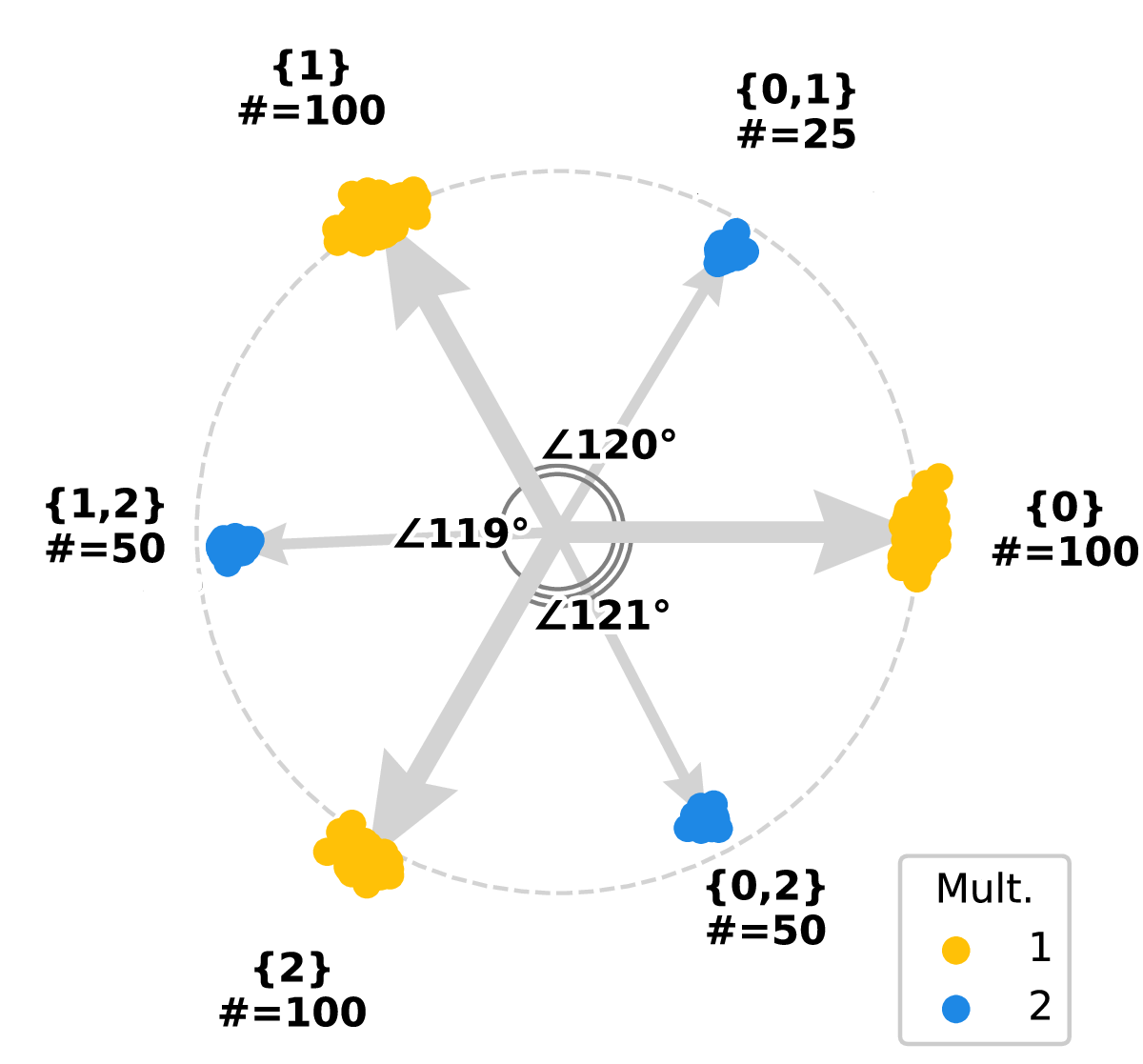}}
    \subfloat[]{\includegraphics[width=0.24\textwidth]{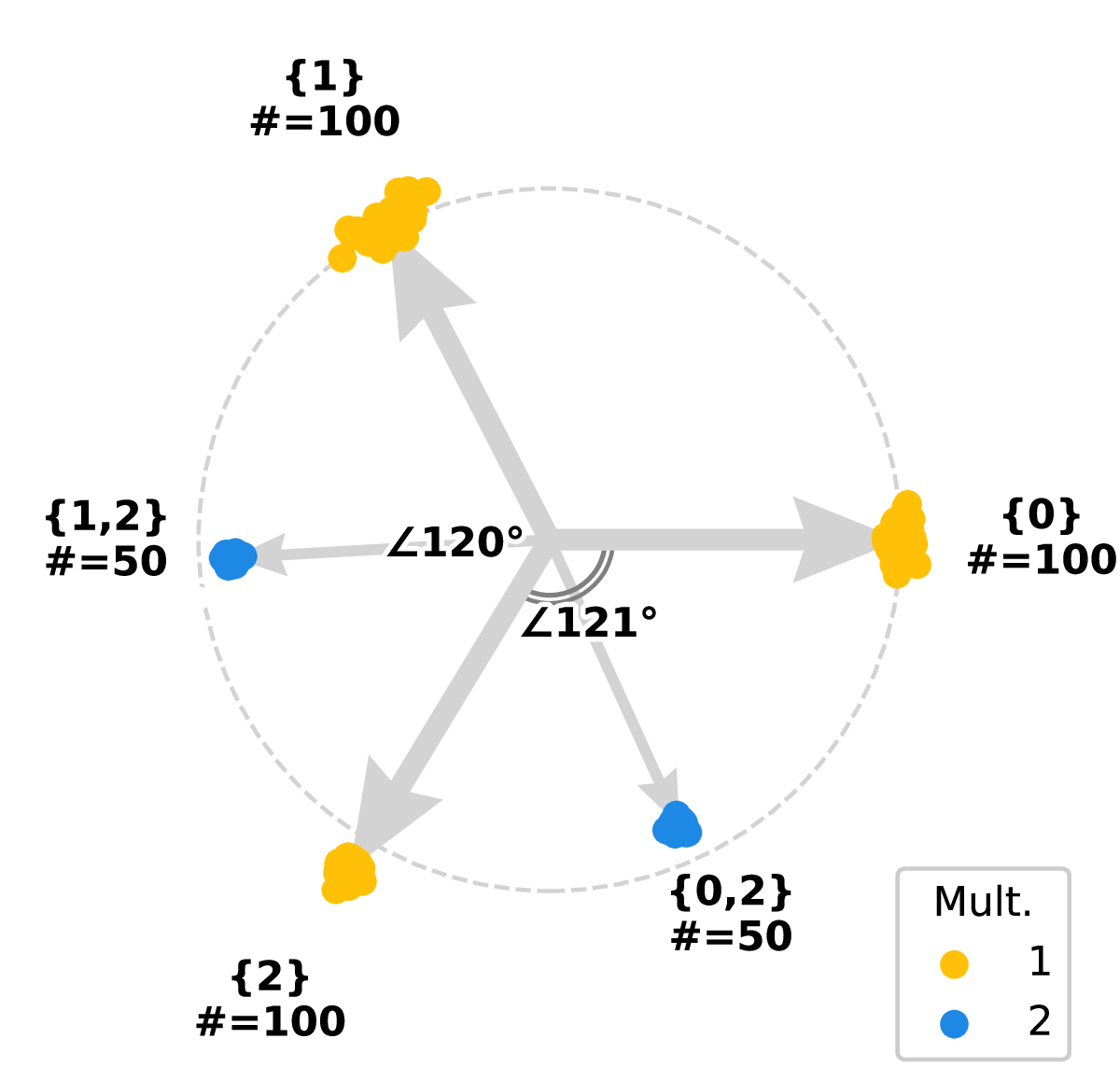}} 
    % \caption{\textbf{\MLab~ NC holds with imbalanced data.} A ResNet18 model is trained using a \MLab\  Cifar10 dataset with balanced Multiplicity-$1$ but imbalanced Multiplicity-$2$ classes. A ResNet18 model is trained using a \MLab\ Cifar10 and a convolution plus MLP model is trained using a \MLab\ MNIST dataset both with balanced Multiplicity-$1$ but imbalanced Multiplicity-$2$ classes.}
    \caption{\small \textbf{\MLab~NC holds with imbalanced data in higher multiplicities.} (a) and (b) plot metrics that measures \MLab~NC on \MLab\ Cifar10; (c) and (d) visualize learned features on \MLab\ MNIST, where one multiplicity-2 class is missing in the set up which results in the reduced \MLab~NC geometry. As we observe, the ETF structure for Multiplicity-1 still holds. More experimental details are deferred to \Cref{sec:exp}. }
    \label{fig:NC_High_imba_c10_Mnist}
\end{figure*}

\paragraph{Data imbalanced-ness in \MLab~learning.} Due to the scarcity of higher multiplicity labels in the training set, in practice the imbalance of training data samples could be a more serious issue in \MLab~than \MClf. It should be noted that there are two types of data imbalanced-ness: (\emph{i}) the imbalanced-ness between classes \emph{within} each multiplicity, and (\emph{ii}) the imbalanced-ness of classes \emph{among} different multiplicities. Interestingly, as long as Multiplicity-$1$ training samples remain balanced between classes, our results in \Cref{fig:NC_High_imba_c10_Mnist} and \Cref{fig:NC-SVHN} imply that the \MLab~NC  holds \emph{regardless} of both within and among multiplicity imbalanced-ness in higher multiplicity. Given that achieving balance in Multiplicity-1 sample data is relatively easy, this implies that our result captures a common phenomenon in \MLab~learning. However, if classes of Multiplicity-1 are imbalanced, we suspect a more general minority collapse phenomenon would happen \cite{fang2021exploring,thrampoulidis2022imbalance}, which is worth of further investigation.

\paragraph{Improving \MLab~prediction \& training via \MLab~NC.} Guided by the feature collapse phenomenon of \MLab~NC, we show that we can improve the prediction accuracy and training efficiency in \Cref{sub:exp-implications}. For prediction, encoding could use an one-nearest-neighbor (ONN) approach to classify new data based on the nearest feature mean in the feature space. Empirical verification confirms that ONN encoding is more efficient and yields superior testing accuracy compared to OvA, as illustrated in \Cref{tab:ONN_OvA}. For training, as shown in \Cref{tab:etf_exp_balance}, we can achieve parameter efficient training for \MLab~by fixing the last layer classifier as simplex ETF and reducing the feature dimension $d$ to $K$. 

%  Another direct application of our theory is shown in \Cref{tab:etf_exp_balance}. 

%In practice, we show that our findings lead to improved prediction and training for \MLab~learning. For prediction, we propose a new one-nearest-neighbor (ONN) approach: for the given test sample, we assign its label to the tag of the nearest neighboring multi-label ETF in the feature space. Compared to the classical one-vs-all (OvA) approach for \MLab~learning, ONN is much more efficient with higher test accuracy. For training, by fixing the last later to be ETF and reducing the feature dimension, our experimental results demonstrate that we can sufficiently reduce parameters without compromising overall performance.
\paragraph{Tag-wise average coefficients for \MLab~ETF with high multiplicity.} The features of high multiplicity are scaled tag-wise average of Multiplicity-$1$ features, and these coefficients are \emph{simple and structured} as shown in \Cref{eq:hm_collapse}. As illustrated in \Cref{fig:multi-label-illustration} (i.e., $K=3$, $M=2$), the feature $\mb h_i^\star$ of Multiplicity-$m$ associated with class-index $S_i$ can be viewed as a \emph{tag-wise average} of Multiplicity-$1$ features in the index set $S_i$. Specifically, the high multiplicity coefficients $\Brac{C_m}_{m=1}^M$ in \Cref{eq:hm_collapse}, which are shared across all features of the same multiplicity, could be expressed as
\[
C_m = \frac{K-1}{\|\bW\|_F^2}\log(\frac{K-m}{m} c_{1,m}), \quad \forall m
\]
where $\Brac{c_{1,m}}_{m=1}^M$ exist.\footnote{They satisfy a set of nonlinear equations (\Cref{app:optimality}).}

\subsubsection{Proof Ideas of \Cref{thm:GO_thm} and Comparison with \MClf~NC}\label{sec:proof}

We briefly outline our proofs for the global optimality in \Cref{thm:GO_thm} as follows: essentially, our proof method first breaks down the $g( \mb W \mb H + \mb b,\mb Y)$ component of the objective function of Problem (\ref{eq:DL-loss-ufm}) into numerous subproblems $g_m( \mb W \mb H_m + \mb b,\mb Y_m)$, categorized by different multiplicity. We determine lower bounds for each $g_m$ and establish the conditions for equality attainment for each multiplicity level. Subsequently, we confirm that equality for these sets of lower bounds of different $m$ values can be attained simultaneously, thus constructing a global optimizer where the overall global objective of (\ref{eq:DL-loss-ufm}) is reached. We demonstrate that all optimizers can be recovered using this approach. As a result, our generalized proof implies \MClf~NC with only single-multiplicity data.

Although our work is inspired by the recent work \cite{zhu2021geometric}, it should be noted that our main results as well as the proving techniques used to establish them significantly differ from that of \cite{zhu2021geometric}. For \MLab, the derivation of optimality conditions is particularly challenging due to the combinatorial complexity of imbalanced features with higher label counts and how they interact with a single linear classifier. We elaborate on this in the following.
\begin{itemize}[leftmargin=*]
    \item We incorporate all multiplicity samples by calculating the gradient of the PAL-CE loss function to obtain the initial lower bound. The tightness condition of such bound uncovers \MLab~learning's unique “in-group and out-group” property hidden behind the combinatorial structure of high multiplicity features. Comparatively, \cite{zhu2021geometric} relied on Jensen's inequality and concavity of log function which falls short under the present of high-multiplicity samples. More details can be found in \Cref{lemma:key-lower-bound-tightness}.
    \item We decoupled the interplay between linear classifier across various multiplicity features by decomposing the loss into different components based on feature multiplicities. Through the decomposition, we then showed that the equality condition for each components can be achieved simultaneously. More details are provided in \Cref{lemma:gm-lower-bound}. 
    \item We further unveil that the higher multiplicity features converge to the "scaled tag-wise average" of its associate tag feature means (\Cref{lemma:ETF_and_sacled_ave}). This requires three new supporting lemmas derived from probabilistic (\Cref{lemma:Y_Moore_pinv}) and matrix theory perspective (\Cref{lemma:proj_subspace_z}, \ref{lemma:pascal_norm}), which is unique in \MLab~learning.
\end{itemize}

\subsubsection{Nonconvex Landscape Analysis}

Due to the nonconvex nature of Problem \eqref{eq:DL-loss}, the characterization of global optimality alone in
\Cref{thm:GO_thm} is not sufficient for guaranteeing efficient optimization to those desired global solutions. Thus, we further study the global landscape of Problem \eqref{eq:DL-loss} by characterizing all of its critical points, we show the following result.

\begin{theorem}[\textbf{Benign Optimization Landscape}]\label{thm:landscape}
    Suppose the same setting of \Cref{thm:GO_thm}, and assume the feature dimension is larger than the number of classes, i.e., $d > K$, and the number of training samples for each class are balanced within each multiplicity. Then the function $f(\bW, \bH, \bbb)$ in Problem \eqref{eq:DL-loss-ufm} is a strict saddle function with no spurious local minimum in the sense that:
    \begin{itemize}[leftmargin=*]
        \item Any local minimizer of $f$ is a global solution of the form described in Theorem \ref{thm:GO_thm}.
        \item Any critical point $(\bW, \bH, \mb b)$ of $f$ that is not a global minimizer is a strict saddle point with negative curvatures, in the sense that there exists some direction $(\mb \Delta_{\mb W},\mb \Delta_{\mb H}, \mb \delta_{\mb b})$ such that the directional Hessian $\nabla^2 f(\bW, \bH, \mb b)[\mb \Delta_{\mb W},\mb \Delta_{\mb H}, \mb \delta_{\mb b}]<0$. 
    \end{itemize}
\end{theorem}
The original proof in \cite{zhu2021geometric} connects the nonconvex optimization problem to a convex low-rank realization and then characterizes the global optimality conditions based on the convex problem. The proof concludes by analyzing all critical points, guided by the identified optimality conditions. 
Because the PAL loss for \MLab~is reduced from the CE loss in \MClf, the above result can be generalized from the result in \cite{zhu2021geometric}.  We defer detailed proofs to \Cref{app:landscape}.
Unlike \Cref{thm:GO_thm}, the result of benign landscape in \Cref{thm:landscape} does not hold for $d = K-1$. The reason is that we need to construct a negative curvature direction in the null space of $\bW$ for showing strict saddle points. Similar to \cite{zhu2021geometric}, we conjecture the \MLab~NC results also hold for $d = K$ and leave it for future work.  

In our paper, we establish the theoretical properties of all critical points, demonstrating that the function is a strict saddle function \cite{ge2015escaping} in the context of multi-label learning with respect to $(\bm W, \bm H)$. It's worth noting that for strict saddle functions like PAL-CE, there exists a substantial body of prior research in the literature that provides rigorous algorithmic convergence to global minimizers. In our case, this equates to achieving a global multi-label neural collapse solution. These established methods include both first-order gradient descent techniques \cite{ge2015escaping, jin2017escape, lee2019first} and second-order trust-region methods \cite{sun2016complete}, all of which ensure efficient algorithmic convergence.

\section{Experiments}\label{sec:exp}

\begin{figure*}[t]
    \centering
    \subfloat[$\mc {NC}_1$ (MLab-\textsc{Mnist})]{\includegraphics[width=0.23\textwidth]{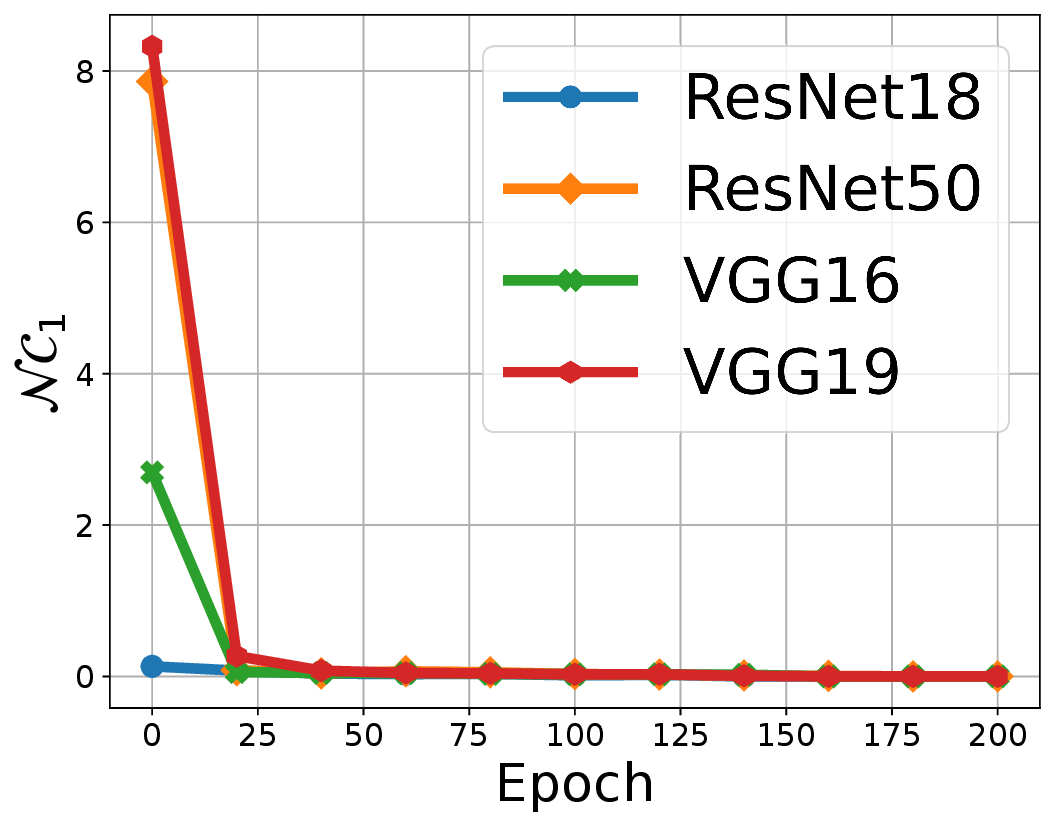}} \
    \subfloat[$\mc {NC}_2$ (MLab-\textsc{Mnist})]{\includegraphics[width=0.23\textwidth]{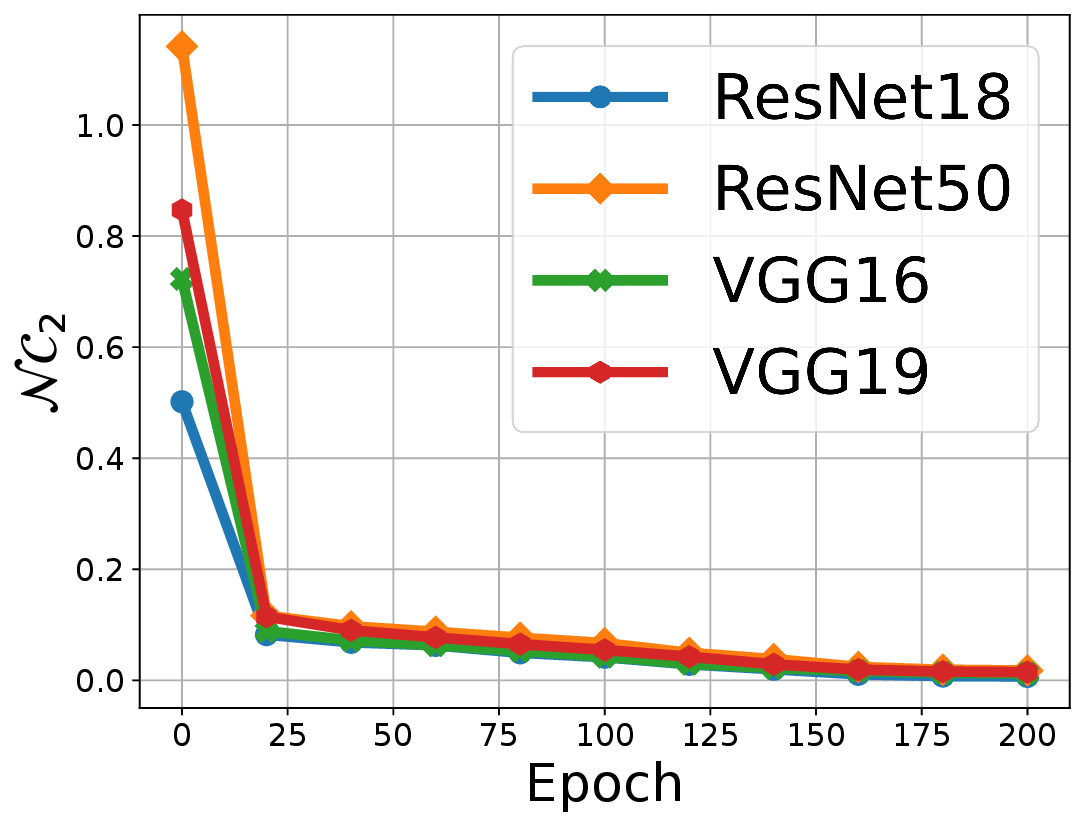}} \
    \subfloat[$\mc {NC}_3$ (\textsc{MLab-Mnist})]{\includegraphics[width=0.23\textwidth]{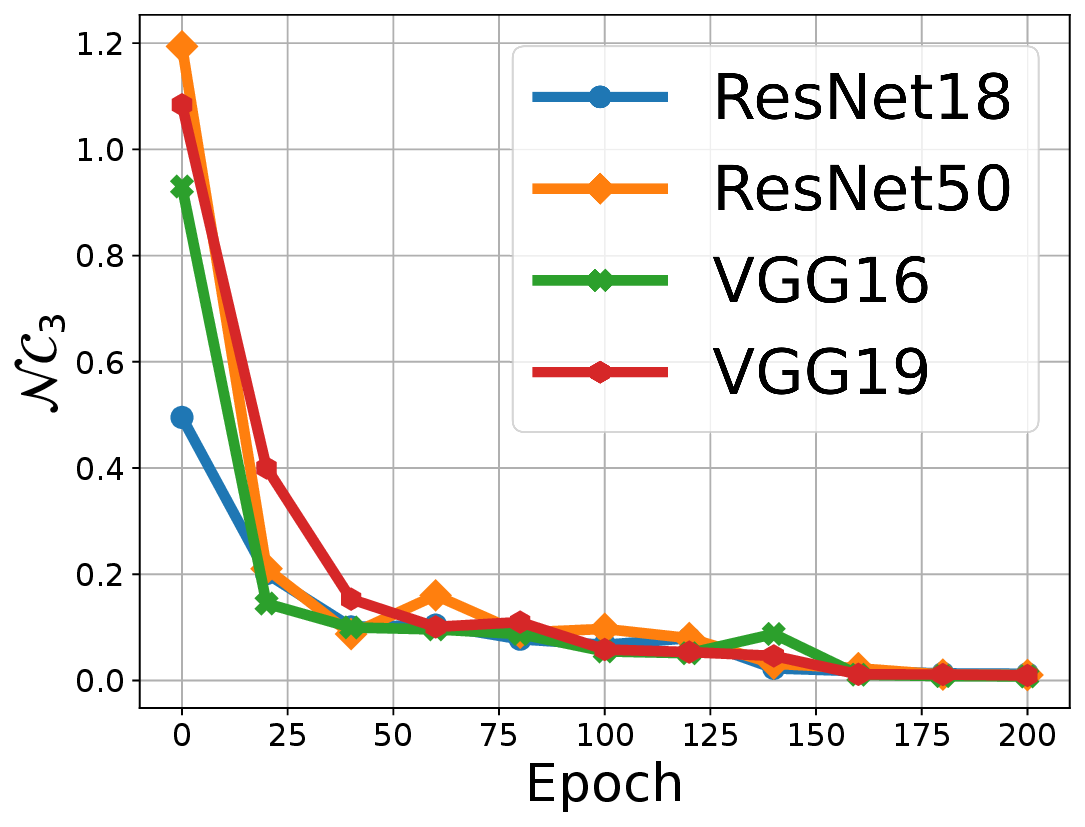}} \
    \subfloat[$\mc {NC}_m$ (MLab-\textsc{Mnist})]{\includegraphics[width=0.23\textwidth]{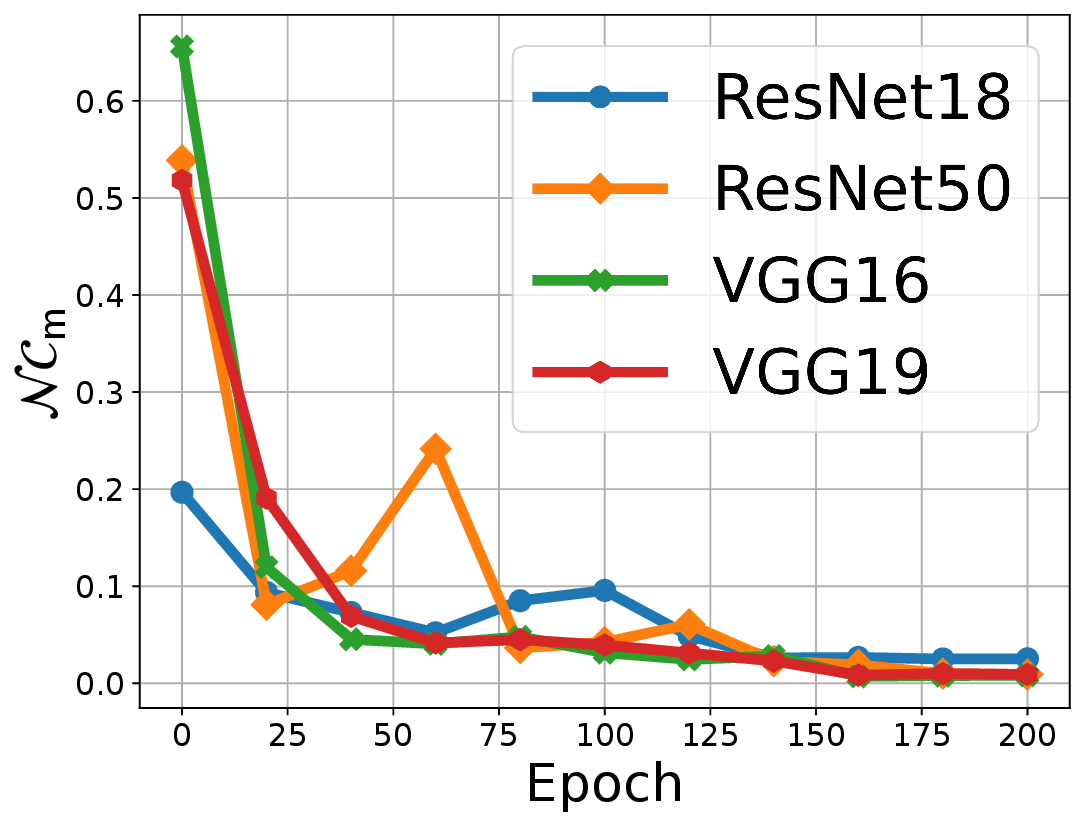}} \\
    \vspace{-0.1in}
    \subfloat[$\mc {NC}_1$ (MLab-{Cifar10})]{\includegraphics[width=0.23\textwidth]{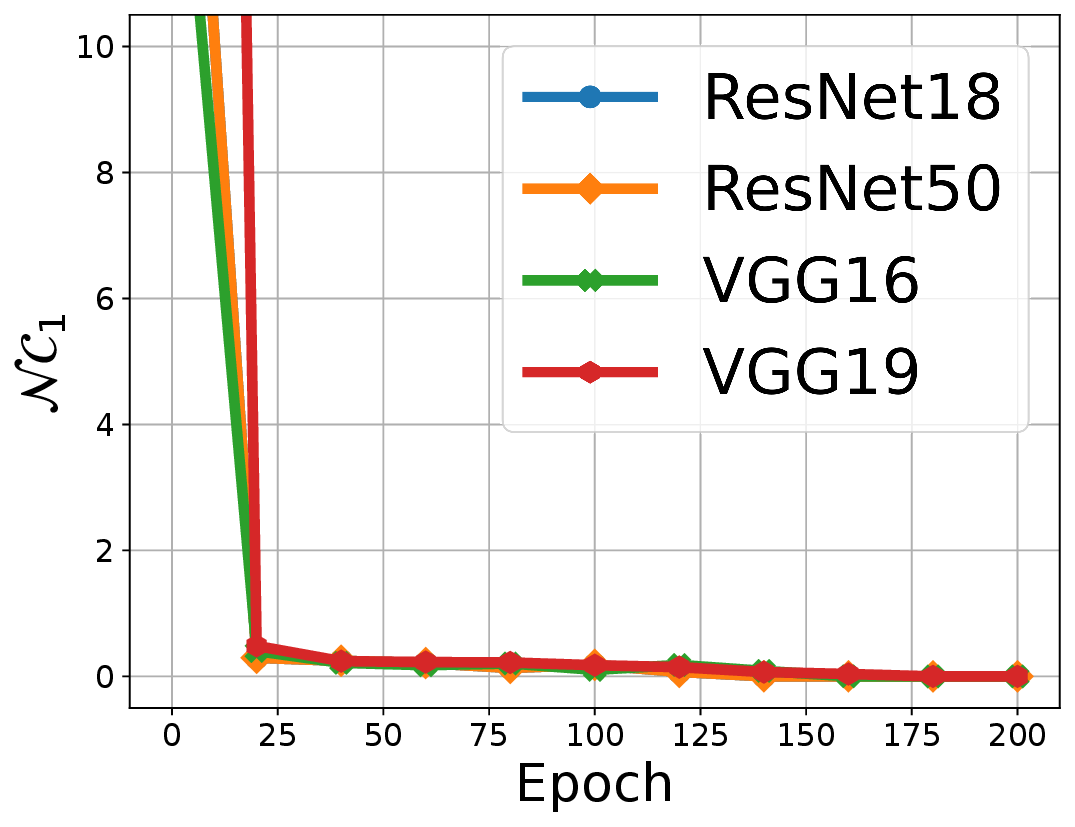}} \
    \subfloat[$\mc {NC}_2$ (MLab-{Cifar10})]{\includegraphics[width=0.23\textwidth]{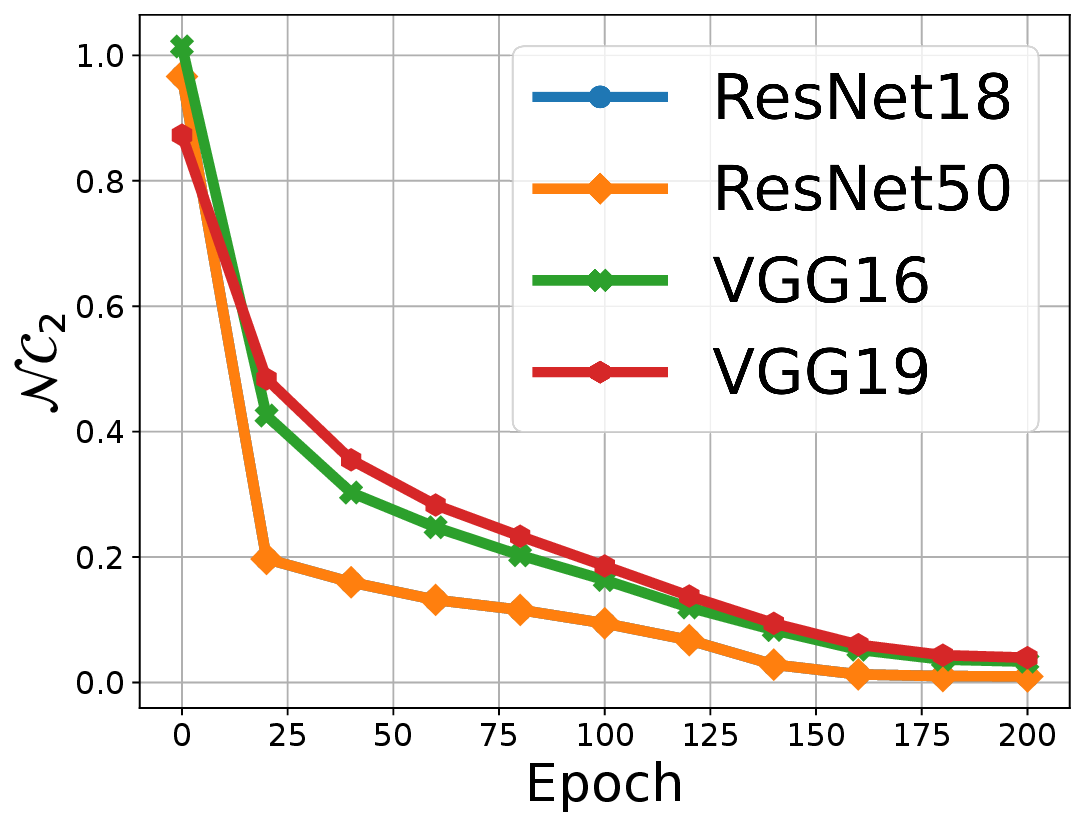}} \
    \subfloat[$\mc {NC}_3$ (MLab-Cifar10)]{\includegraphics[width=0.23\textwidth]{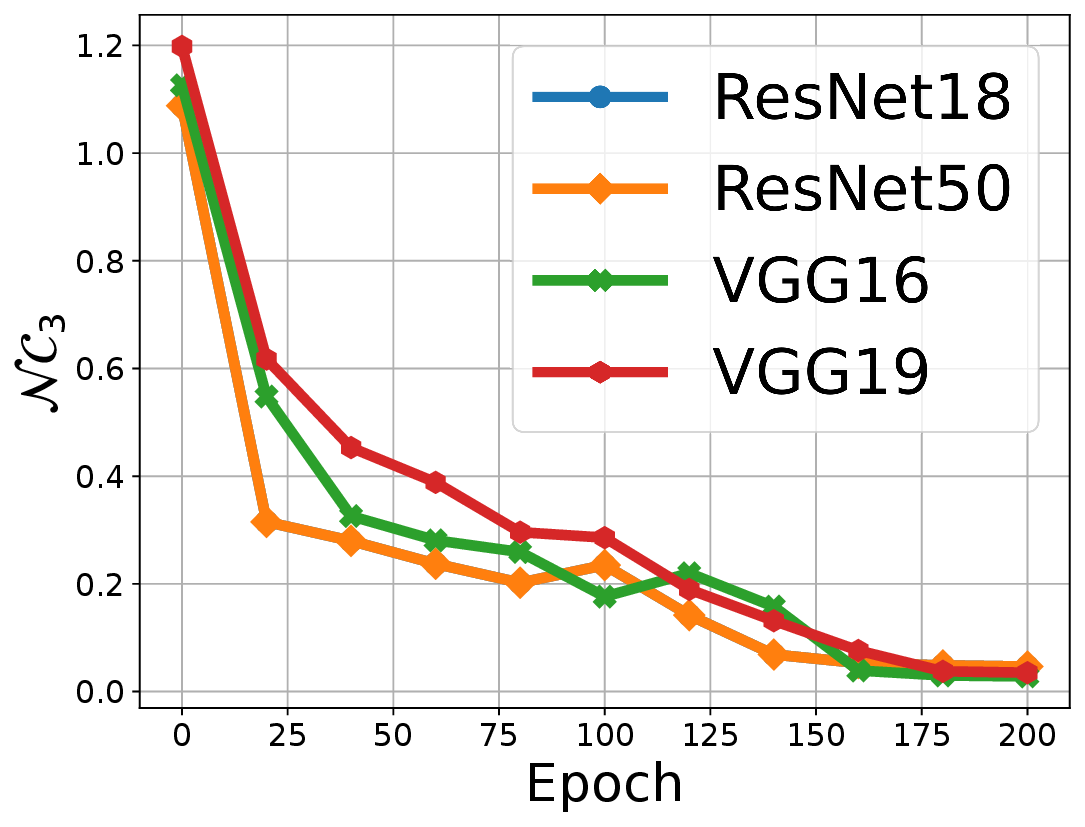}} \
    \subfloat[$\mc {NC}_m$ (MLab-Cifar10)]{\includegraphics[width=0.23\textwidth]{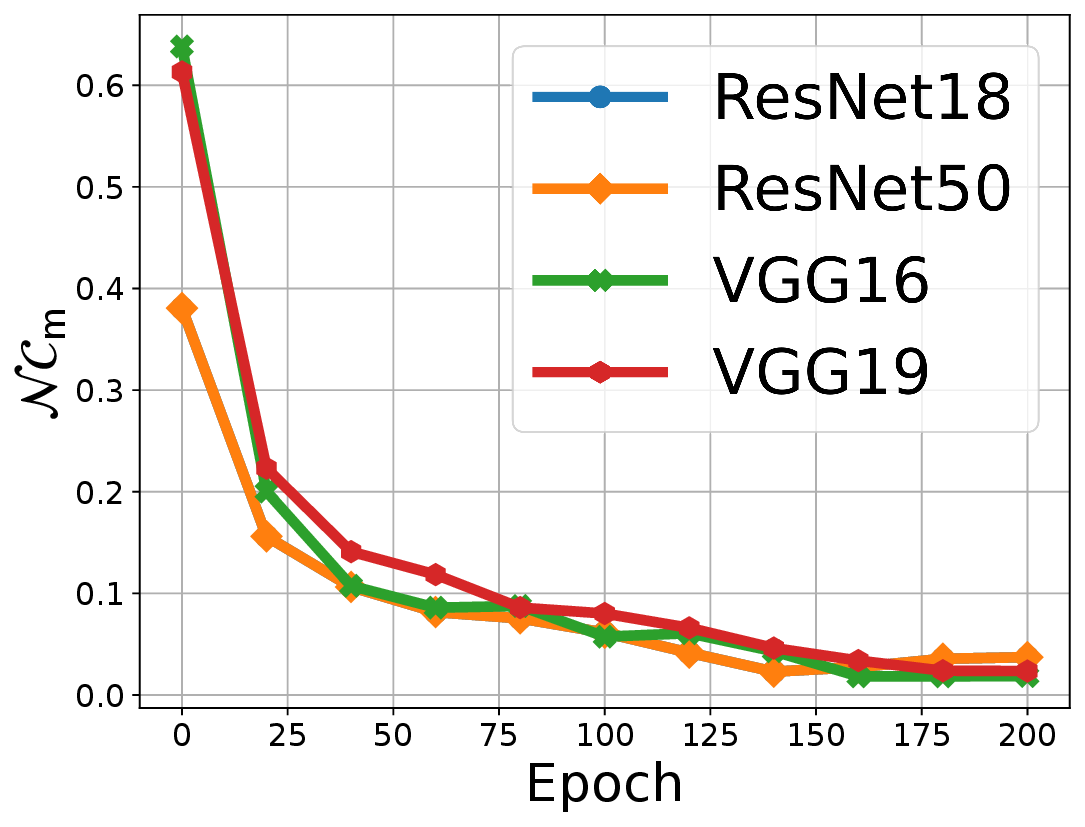}} \\
    \vspace{-0.1in}
    \caption{\textbf{Prevalence of \MLab~NC across different network architectures} on \MLab~MNIST (top) and \MLab~Cifar10 (bottom). From the left to the right, the plots show the four metrics, $\mathcal{NC}_1, \mathcal{NC}_2, \mathcal{NC}_3$, and $\mathcal{NC}_m$, for measuring \MLab~NC. More details about dataset and training setups could be found in \Cref{appendix-section:Mlab_MNIST_C10}.
    }
    \label{fig:NC-MNIST-CIFAR10}
    \vspace{-0.15in}
\end{figure*}

% \vspace{-0.1in}

In this section, we first conduct a series of experiments to demonstrate and analyze the \MLab~NC on different practical deep networks with various multi-label datasets. Second, we show that the geometric structure of \MLab~NC could efficiently guide \MLab~learning in both testing and training stage for better performance. 

% First, when the training data are balanced within multiplicities, \Cref{fig:NC-MNIST-CIFAR10} shows that all practical deep networks exhibit \MLab~NC during the terminal phase of training. We further explore the \MLab~NC under multiplicity imbalanced-ness on both synthetic (\Cref{fig:NC_High_imba_c10_Mnist}) and real data (\Cref{fig:NC-SVHN}), demonstrating that \MLab~NC partially holds irrespective of imbalanced-ness in higher multiplicity data. Finally, we \py{show that our novel ONN inference mechanism achieve higher testing accuracy efficiently compared with OvA in both synthetic data of various settings and real data \Cref{tab:ONN_OvA}.} We also show that saving of a large portion of network parameters is achievable in training deep networks without compromising performance by exploiting \MLab~NC. 

The datasets used in our experiment are real-world \MLab~SVHN \cite{netzer2011reading}, along with synthetically generated \MLab~MNIST \cite{lecun2010mnist} and \MLab~Cifar10 \cite{krizhevsky2009learning}. The detail dataset description, generation, and visualization along with experimental setups could be found in \Cref{app:data_show}.

\subsection{Verification of M-lab NC}
\paragraph{Experimental demonstration of \MLab~NC on practical deep networks.} %\label{subsec:exp_verify}
%In this section, we investigate the presence of \NC\ when using the multi-label MNIST and Cifar10 datasets. Four different metrics are designed to capture the \MLab~NC phenomenon.
%We assess the within-class variability collapse ($\mathcal{NC}_1$), convergence of learned classifier and feature class means to simplex ETF ($\mathcal{NC}_2$), and convergence to self-duality ($\mathcal{NC}3$) using the metrics from the original \NC\ paper \cite{papyan2020prevalence}. Furthermore, to measure the convec combination relationship between higher order $h_S$ and its multiplicity $1$ component $h_i, h_j$ with $S = \{i,j\}$, we 
When the training data (\Cref{appendix-section:Mlab_MNIST_C10}) are balanced within multiplicities, \Cref{fig:NC-MNIST-CIFAR10} shows that all practical deep networks exhibit \MLab~NC during the terminal phase of training as implied by our theory. To show this, we introduce new metrics to measure \MLab~NC on the last-layer features and classifiers of deep networks.

Based on theoretical results in \Cref{subsec:MLab-NC}, we use the original metrics $\mathcal{NC}_1$ (measuring the within-class variability collapse), $\mathcal{NC}_2$ (measuring convergence of learned classifier and feature class means to simplex ETF), and $\mathcal{NC}_3$ (measuring the convergence to self-duality) introduced in \cite{papyan2020prevalence} to measure \MLab~NC on Multiplicity-$1$ features $\mb H_1$ and classifier $\mb W$. Additionally, we also use the $\mathcal{NC}_1$ metric to measure variability collapse on high multiplicity features $\mb H_m\;(m>1)$. Finally, to measure \MLab~ETF (the \textit{tag-wise average} property)
on Multiplicity-$2$ features,\footnote{This is because our dataset only contains labels up to Multiplicity-$2$. The $\mathcal{NC}_m$ could be easily extended to capture scaled average for other higher multiplities.} we propose a new angle metric $\mathcal{NC}_m$, which is defined as:
 \begin{align*}
\mathcal{NC}_m = \frac{ \text{Avg.}(\{geo_{\angle}( \overline{\mb h}_{i},\;\; \overline{{\mb h}}_{j} + \overline{{\mb h}}_{\ell}): (i,j,\ell) \in \mc F_1 \})}{\text{Avg.}(\{geo_{\angle}(\overline{\mb h}_{i^\prime},\;\; \overline{{\mb h}}_{j^\prime} + \overline{{\mb h}}_{\ell^\prime}): (i^\prime,j^\prime,\ell^\prime) \in \mc F_2 \}},
\end{align*}
with the sets
\begin{align*}
    \mc F_1 &= \Brac{ i,j,\ell \mid |S_i| = 2, \,|S_j| = |S_\ell| = 1,\, S_i =  S_j \cup S_\ell  }, \\
    \mc F_2 &= \Brac{ i,j,\ell \mid |S_{i}| = 2, \,|S_{j}| = |S_{\ell}| = 1 }.
\end{align*}
%\qq{the equation needs to be rewritten}
%\py{equation too long to rewrite, maybe move it to Appendix and describe it here?}
%\qq{rewrite the metric for Multi-label, explain clearly}
Here, $geo{\angle}$ represents the geometric angle between two vectors, and $\overline{\bh}_{i}$ is the mean of all features in the label set $S_i$. Intuitively, our $\mathcal{NC}_m$ measures the angle between features means of different label sets or classes. The numerator calculates the average angle difference between multiplicity-$2$ features means and the sum of their multiplicity-$1$ component features means. %\st{of the geometric difference between multiplicity $1$ and $2$ feature mean pairs such that multiplicity $2$ label sets contain that have a scaled average relationship}, 
While the denominator serves as a normalization factor that is the average of all existing pairs regardless of the relationship.\footnote{For example, if we have $4$ total classes for multiplicity-1 samples, they corresponds to $4$ features means and hence $6$ different sums if we randomly pick $2$ features means to sum up. Multiplicity-2 then has $\binom{4}{2} = 6$ features means, there are then 36 possible angles to calculate, we averaged these 36 angles as the denominator.} As training progresses, the numerator will converge to $0$, while the denominator becomes larger demonstrating the angle collapsing. 

As shown in \Cref{fig:NC-MNIST-CIFAR10} and \Cref{fig:NC-SVHN}, practical networks do exhibit \MLab~NC, and such a phenomenon is prevalent across network architectures and datasets. Specifically, the four metrics, evaluated on four different network architectures and two different datasets, all converge to zero as the training progresses toward the terminal phase. %Besides the synthetic dataset, we also explored the \MLab-NC on practical PASCAL VOC2007 dataset \cite{everingham2007pascal} with more experimental results in \Cref{app:extra-exp}. Although VOC2007 lacks balanced Multiplicity-$1$ classes, we can still observe a reduced level of \MLab~NC.

\begin{figure*}[t]
    \centering
    \subfloat[$\mc {NC}_1$]{\includegraphics[width=0.23\textwidth]{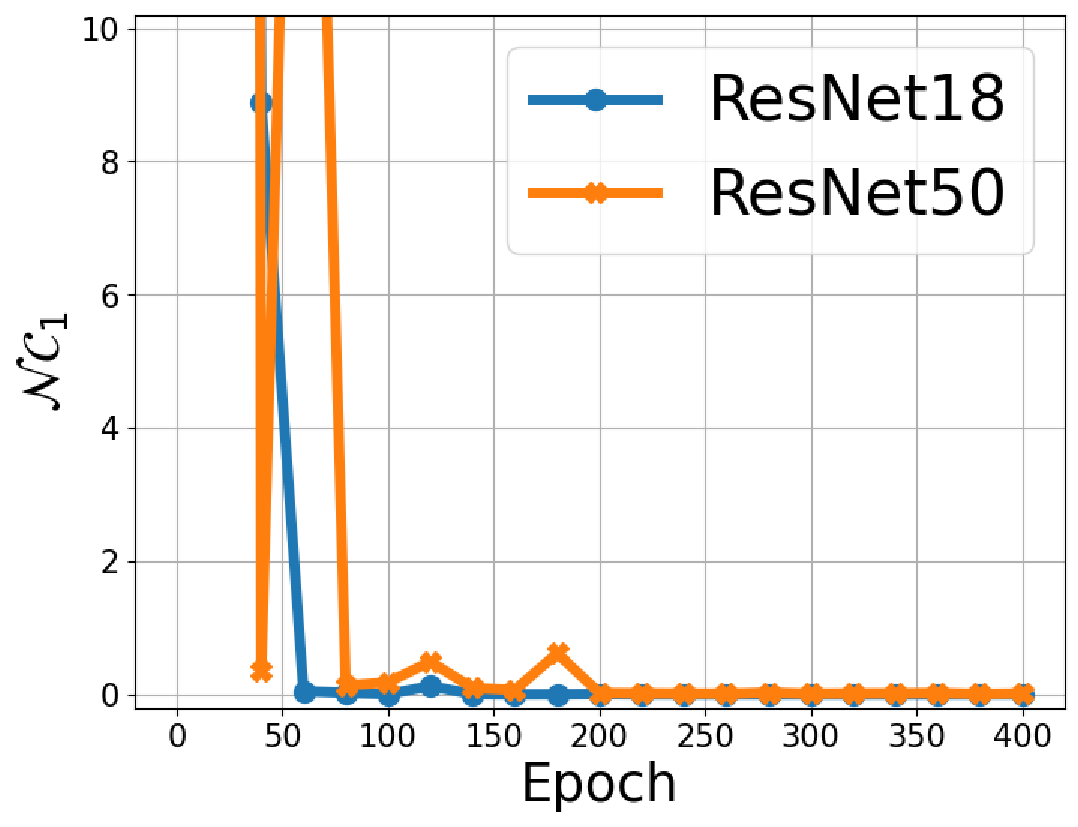}} \
    \subfloat[$\mc {NC}_2$]{\includegraphics[width=0.23\textwidth]{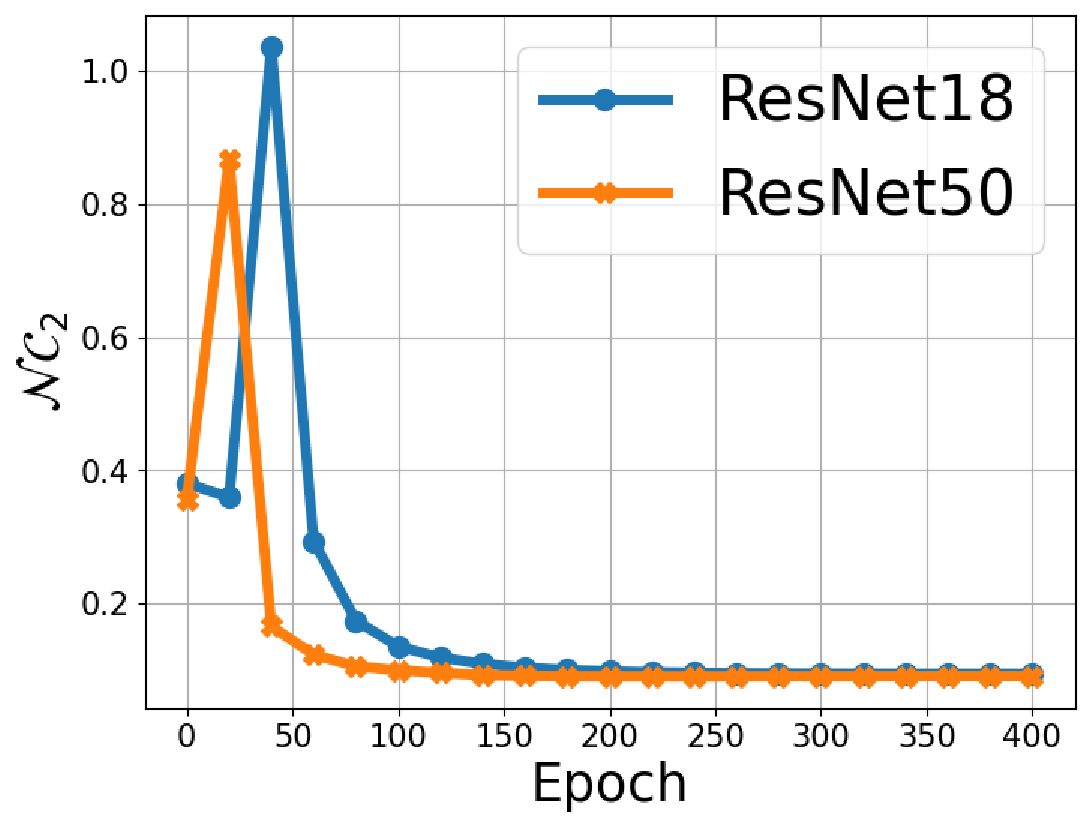}} \
    \subfloat[$\mc {NC}_3$]{\includegraphics[width=0.23\textwidth]{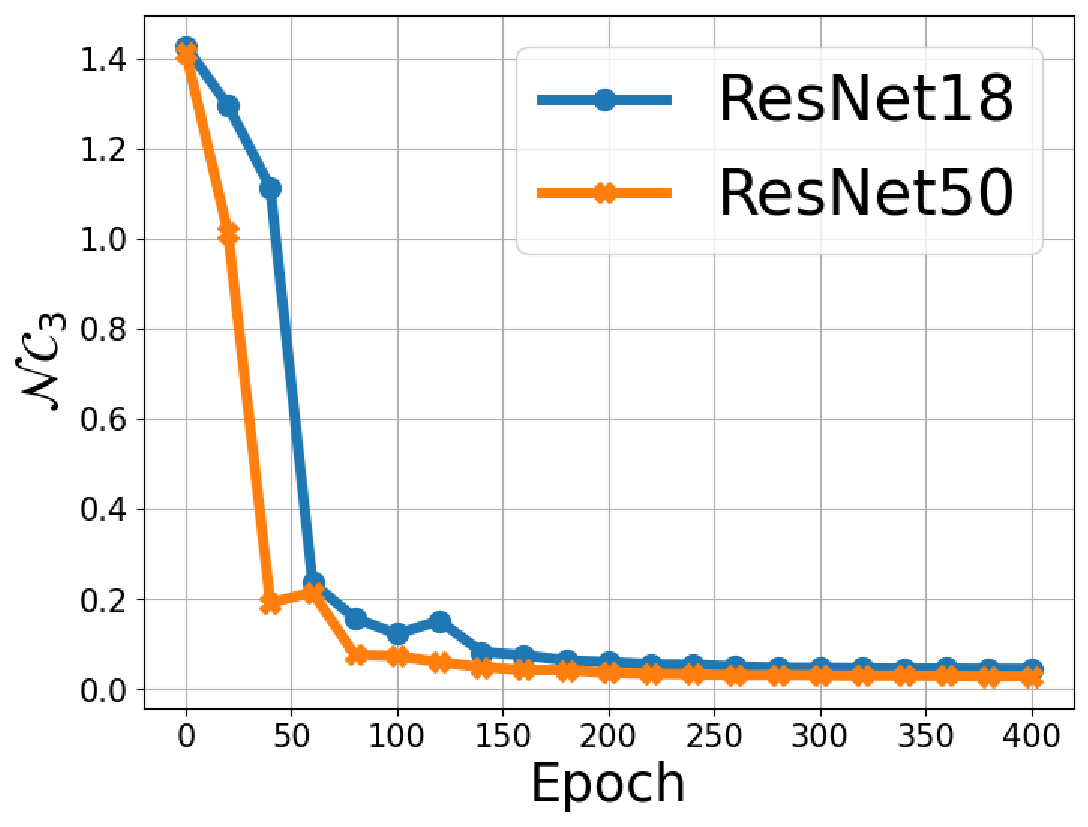}} \
    \subfloat[$\mc {NC}_m$]{\includegraphics[width=0.23\textwidth]{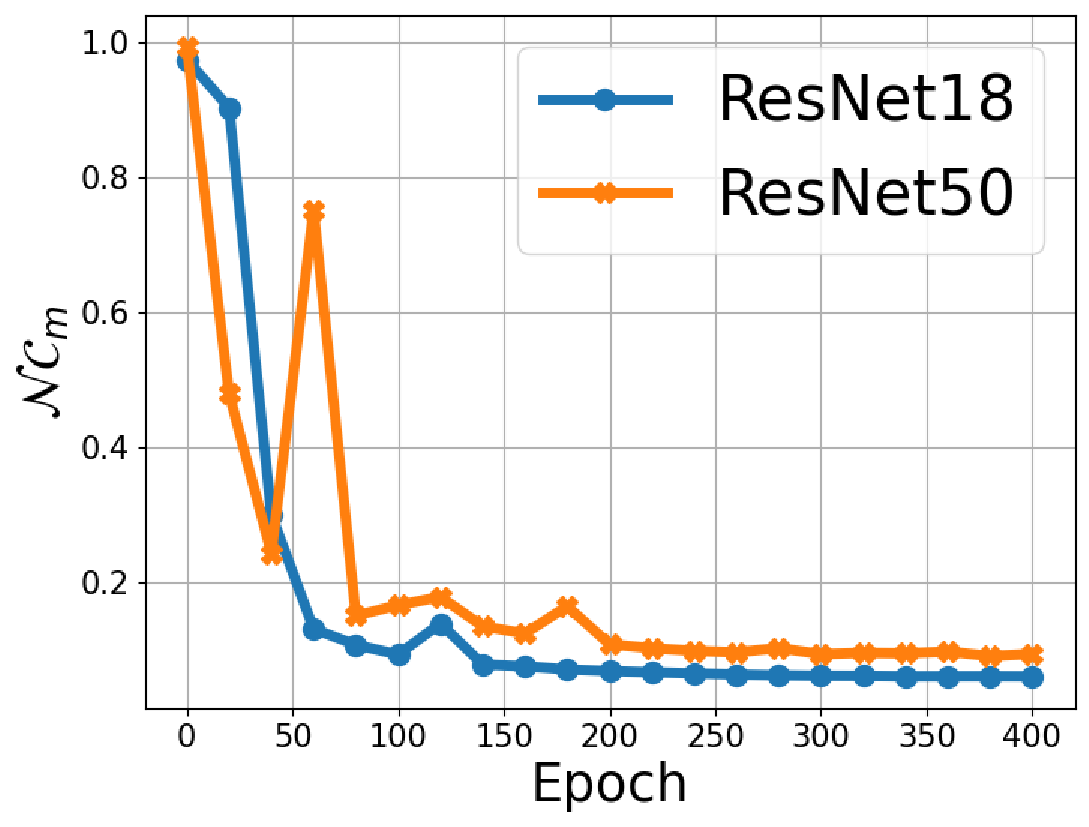}} 
    \vspace{-0.1in}
    \caption{\textbf{Prevalence of \MLab~NC on the \MLab~SVHN dataset}. We train ResNets models on the \MLab~SVHN dataset \cite{netzer2011reading} for $400$ epochs and report $\mathcal{NC}_1, \mathcal{NC}_2, \mathcal{NC}_3$, and $\mathcal{NC}_m$, for measuring \MLab~NC, respectively. See \Cref{appendix-section:imbala_SVHN} for more details.}
    \label{fig:NC-SVHN}
\end{figure*}

\paragraph{\MLab~NC holds despite of class imbalanced-ness in high order multiplicity labels.} Moreover, our experiments imply that maintaining balance in single-label training samples ensures the persistence of \MLab~NC, even amidst imbalance in higher-order label multiplicities, across both synthetic and real-world data sets. To verify this, we create imbalanced \MLab~cifar10 and MNIST datasets, and real-world \MLab~SVHN dataset, with more details in \Cref{app:data_show}.

%Supported and inspired by \Cref{thm:GO_thm}, where $\bW$ only collapse to $\bH_1$, experimentally we found that as long as the training samples of Multiplicity-$1$ remain balanced, we empirically observe that \MLab~NC holds regardless of the imbalanced-ness in high order multiplicity on both synehtic and real-world datasets. To verify this, we create imbalanced \MLab~cifar10 and MNIST datasets, and real-world \MLab~SVHN dataset.

% The cifar10 dataset has balanced Multiplicity-$1$ samples ($5000$ for each class). For the classes of Multiplicity-$2$, we divide them into $3$ groups: the large group ($500$ samples), the middle group ($50$ samples), and the small group ($5$ samples). For \MLab~SVHN dataset, minimal preprocessing is applied to ensure the balanced multiplicity-1 samples and for the details of the dataset, we refer readers to \Cref{fig:svhn_setup}. Our results can be summarized as follows.
\begin{itemize}[leftmargin=*]
    \item \textbf{Experimental results on imbalanced \MLab~Cifar10 dataset.} We run a ResNet18 model with these datasets (\Cref{appendix-section:imbala_Mlab_MNIST_C10}) and report the metrics of measuring \MLab~ NC in \Cref{fig:NC_High_imba_c10_Mnist} (a) (b). We can observe that not only $\mathcal{NC}_1$ to $\mathcal{NC}_3$ collapse to zero, but the $\mathcal{NC}_m$ metric is also converging zero for all $3$ groups of different sizes.
 
    \item \textbf{Experimental results on imbalanced \MLab~MNIST dataset.} For \Cref{fig:NC_High_imba_c10_Mnist} (c) (d) on \MLab\ MNIST (\Cref{appendix-section:imbala_Mlab_MNIST_C10}), we can see from the visualization of the features vectors that the scaled average property still holds despite a missing class in higher multiplicity. Here, we train a simple convolution plus multi-layer perceptron model. This implies that \MLab~ NC holds even under data imbalanced-ness in high order multiplicity.

    \item \textbf{Experimental results on imbalanced \MLab~SVHN dataset.} In addition, we demonstrate that \MLab~NC happens independently of higher multiplicity data imbalanced-ness on real-world \MLab~SVHN dataset (\Cref{appendix-section:imbala_SVHN}). We evaluated the behavior of NC metrics on this dataset as illustrated in \Cref{fig:NC-SVHN}, affirming the continued validity of our analysis in real-world settings.
\end{itemize}

%For more detail, we refer readers to \Cref{fig:svhn_setup}. Subsequently, we assessed the trends of NC metrics on this dataset, as depicted in \Cref{fig:NC-SVHN}. The plots affirm the continued validity of our analysis within real-world settings.

\begin{table*}[t]
\centering
\resizebox{0.98\linewidth}{!}{
\centering
\begin{tabular}{|ccccccccc|}
\hline
\multicolumn{1}{|c|}{Dataset}                                                       & \multicolumn{2}{c|}{c10-L}                                                       & \multicolumn{2}{c|}{c10-M}                                                       & \multicolumn{2}{c|}{c10-S}                                                       & \multicolumn{2}{c|}{SVHN}                                            \\ \hline
\multicolumn{1}{|c|}{\begin{tabular}[c]{@{}c@{}}Encoding \\ Mechanism\end{tabular}} & \multicolumn{1}{c|}{OvA}           & \multicolumn{1}{c|}{\textbf{ONN}}           & \multicolumn{1}{c|}{OvA}           & \multicolumn{1}{c|}{\textbf{ONN}}           & \multicolumn{1}{c|}{OvA}           & \multicolumn{1}{c|}{\textbf{ONN}}           & \multicolumn{1}{c|}{OvA}                    & \textbf{ONN}           \\ \hline
\multicolumn{9}{|c|}{Test Metrics (F1-score $/$ SubsetZero-one Accuracy)} \\ \hline
\multicolumn{1}{|c|}{Overall}                                                       & \multicolumn{1}{c|}{$0.896 / 0.787$} & \multicolumn{1}{c|}{\textbf{$\mathbf{0.899} / \mathbf{0.805}$}} & \multicolumn{1}{c|}{$0.885 / 0.778$} & \multicolumn{1}{c|}{\textbf{$\mathbf{0.888} / \mathbf{0.794}$}} & \multicolumn{1}{c|}{$0.857 / 0.740$} & \multicolumn{1}{c|}{\textbf{$\mathbf{0.863} / \mathbf{0.760}$}} & \multicolumn{1}{c|}{$0.937 / 0.831$}          & \textbf{$\mathbf{0.942} / \mathbf{0.837}$} \\ \hline
\multicolumn{1}{|c|}{Mul-1}                                                         & \multicolumn{1}{c|}{$0.893/ 0.828$} & \multicolumn{1}{c|}{\textbf{$\mathbf{0.894 / 0.829}$}} & \multicolumn{1}{c|}{$0.890/0.832$} & \multicolumn{1}{c|}{$0.890 / \mathbf{0.833}$}          & \multicolumn{1}{c|}{$0.865/0.812$} & \multicolumn{1}{c|}{\textbf{$\mathbf{0.867/0.815}$}}  & \multicolumn{1}{c|}{$0.905/0.761$}          & \textbf{$\mathbf{0.928/0.836}$} \\ \hline
\multicolumn{1}{|c|}{Mul-2}                                                         & \multicolumn{1}{c|}{$0.898/0.775$} & \multicolumn{1}{c|}{\textbf{$\mathbf{0.901/0.797}$}} & \multicolumn{1}{c|}{$0.883/0.754$} & \multicolumn{1}{c|}{\textbf{$\mathbf{0.887/0.777}$}} & \multicolumn{1}{c|}{$0.851/0.691$} & \multicolumn{1}{c|}{\textbf{$\mathbf{0.861/0.723}$}} & \multicolumn{1}{c|}{\textbf{$\mathbf{0.953/0.871}$}} & $0.949/0.842$          \\ \hline
\multicolumn{1}{|c|}{Mul-3}                                                         & \multicolumn{6}{c|}{*} & \multicolumn{1}{c|}{$0.926/0.771$}          & \textbf{$\mathbf{0.935/0.814}$} \\ \hline
\multicolumn{9}{|c|}{Computational Complexity}                                \\ \hline
\multicolumn{1}{|c|}{FLOPs(B)}                                                      & \multicolumn{1}{c|}{$2.46$}          & \multicolumn{1}{c|}{\textbf{$\mathbf{0.0307}$}}        & \multicolumn{1}{c|}{$2.00$}          & \multicolumn{1}{c|}{\textbf{$\mathbf{0.0249}$}}        & \multicolumn{1}{c|}{$1.54$}          & \multicolumn{1}{c|}{\textbf{$\mathbf{0.0192}$}}        & \multicolumn{1}{c|}{$1.068$}                  & $\mathbf{0.0135}$                 \\ \hline
\end{tabular}}
%}

\caption{\textbf{Test accuracy and computational complexity comparison between ONN (ours) and OvA approaches across different multiplicity imbalance-ness for both real and synthetic datasets.} Both approaches are trained with fixed ETF classifier. Models are trained using ResNet18 structure and reported accuracy are the average over $3$ models with random initialization. We report both F1-score and subset zero-one accuracy that only count for perfect predictions \cite{dembczynski2010regret}.} \label{tab:ONN_OvA}
%    \end{minipage}
    % \hfill
\end{table*}

\begin{table*}
    
   % \begin{minipage}{0.48\textwidth}
\centering
\resizebox{0.7\linewidth}{!}{

\begin{tabular}{|lllllllll|}
\hline
\multicolumn{1}{|l|}{\multirow{2}{*}{\textbf{Dataset / Arch.}}} & \multicolumn{2}{l|}{ResNet18}                                                                 & \multicolumn{2}{l|}{ResNet50}                                                                & \multicolumn{2}{l|}{VGG16}                                                                    & \multicolumn{2}{l|}{VGG19}                                                                    \\ \cline{2-9} 
\multicolumn{1}{|l|}{}                                          & \multicolumn{1}{l|}{Learned}            & \multicolumn{1}{l|}{ETF}                            & \multicolumn{1}{l|}{Learned}            & \multicolumn{1}{l|}{ETF}                           & \multicolumn{1}{l|}{Learned}            & \multicolumn{1}{l|}{ETF}                            & \multicolumn{1}{l|}{Learned}            & ETF                                                 \\ \hline
\multicolumn{9}{|c|}{Test IoU (\%)}                                                                                                                                                                                                                                                                                                                                                                                                                            \\ \hline
\multicolumn{1}{|l|}{MLab-MNIST}                                & \multicolumn{1}{l|}{99.5}               & \multicolumn{1}{l|}{99.4}                           & \multicolumn{1}{l|}{99.4}               & \multicolumn{1}{l|}{99.4}                          & \multicolumn{1}{l|}{99.5}               & \multicolumn{1}{l|}{99.5}                           & \multicolumn{1}{l|}{99.5}               & 99.5                                                \\ \hline
\multicolumn{1}{|l|}{MLab-Cifar10}                              & \multicolumn{1}{l|}{87.7}               & \multicolumn{1}{l|}{87.7}                           & \multicolumn{1}{l|}{88.9}               & \multicolumn{1}{l|}{88.6}                          & \multicolumn{1}{l|}{86.9}               & \multicolumn{1}{l|}{87.4}                           & \multicolumn{1}{l|}{88.7}               & 87.0                                                \\ \hline
\multicolumn{9}{|c|}{Percentage of parameter saved (\%)}                                                                                                                                                                                                                                                                                                                                                                                                       \\ \hline
\multicolumn{1}{|l|}{MLab-MNIST}                                & \multicolumn{1}{c|}{\multirow{2}{*}{0}} & \multicolumn{1}{c|}{\multirow{2}{*}{\textbf{20.7}}} & \multicolumn{1}{c|}{\multirow{2}{*}{0}} & \multicolumn{1}{c|}{\multirow{2}{*}{\textbf{4.5}}} & \multicolumn{1}{c|}{\multirow{2}{*}{0}} & \multicolumn{1}{c|}{\multirow{2}{*}{\textbf{15.8}}} & \multicolumn{1}{c|}{\multirow{2}{*}{0}} & \multicolumn{1}{c|}{\multirow{2}{*}{\textbf{11.6}}} \\ \cline{1-1}
\multicolumn{1}{|l|}{\textbf{MLab-Cifar10}}                     & \multicolumn{1}{c|}{}                   & \multicolumn{1}{c|}{}                               & \multicolumn{1}{c|}{}                   & \multicolumn{1}{c|}{}                              & \multicolumn{1}{c|}{}                   & \multicolumn{1}{c|}{}                               & \multicolumn{1}{c|}{}                   & \multicolumn{1}{c|}{}                               \\ \hline
\end{tabular}}
    \caption{\textbf{Comparison of the performances and parameter efficiency between learned and fixed ETF classifier.} When counting parameters, we consider all parameters that require gradient calculation during back-propagation.}
    \label{tab:etf_exp_balance}
    % \end{minipage}
    \vspace{-0.1in}
\end{table*}

\subsection{Practical Implications for \MLab~Learning}\label{sub:exp-implications}
Finally, we show that our findings lead to improved prediction and training for \MLab~learning. For prediction, Our ONN encoding approach attains greater accuracy than the OvA method with improved efficiency, eliminating the need for extra classifier training, as shown in \Cref{tab:ONN_OvA}. For training, our theory supports reducing feature dimension and maintaining a fixed classifier structure without sacrificing training accuracy as shown in \Cref{tab:etf_exp_balance}. Additionally, the detail descriptions of training datasets with experimental setups can be found in \Cref{appendix-section:test_data_show}

\subsubsection{Implication I:  \MLab~NC guided methods for improved test performance.} 
As discussed in \Cref{sec:intro}, the classical OvA and PAL methods have several fundamental limitations, specially, when it comes to how to convert outputs of the model into tags (i.e., binary vectors). In this part, we show that we can improve the PAL based method by our findings. Specifically, our proposed method and comparison baseline are the following.

\begin{itemize}[leftmargin=*]
    \item \textbf{\emph{Proposed} ``one-nearest-neighbor'' (ONN) encoding method}: supported by the \MLab~NC, features within each class collapse to their means across all multiplicities. Utilizing this, encoding new testing data into binary vectors is simplified: a one-nearest-neighbor calculation is performed between the testing data's features and all class means.

    \item \textbf{\emph{Classical} ``one-versus-all'' (OvA) encoding method}: divides the task into multiple binary classification subtasks, where each needs to train an individual binary classifier for every label based on pre-trained features. During testing,  thresholding is used to convert the outputs of the model (i.e., the logits) into tags.
\end{itemize}

We compared our ONN method with OvA across three synthetic \MLab~Cifar-10 datasets and one real-world dataset with different data imbalanced-ness. Two types of test accuracies, F1-score and subset zero-one, are reported in \cref{tab:ONN_OvA}. The F1-score provides a more balanced performance measure, whereas subset zero-one accuracy only considers perfect matches to the ground truth, disregarding partially correct predictions. The detailed setup of training and dataset could be find in \Cref{appendix-section:test_data_show}. As we observe from \cref{tab:ONN_OvA}, our ONN method uniformly outperforms OvA in overall accuracy, with higher accuracy especially in higher multiplicities. Simultaneously, our ONN eliminates the necessity of training multiple binary classifiers unlike OvA, leading to substantially lower computational complexity for predictions, quantified in billions of FLOPs. Remarkably, even when dealing with the class-imbalanced real \MLab~SVHN dataset, our ONN method consistently achieves an overall higher accuracy than OvA.

\subsubsection{Implication II: \MLab~NC guided parameter-efficient training.} %\label{subsec:exp_ParaEff} 
With the knowledge of \MLab~NC in hand, we can make direct modifications to the model architecture to achieve parameter savings without compromising performance for \MLab~classification. Specifically, parameter saving could come from two folds: (i) given the existence of \NC\ in the multi-label case with $d \geq K$, we can reduce the dimensionality of the penultimate features to match the number of labels (i.e., we set $d=K$); (ii) recognizing that the final linear classifier will converge to a simplex ETF as the training converges, we can initialize the weight matrix of the classifier as a simplex ETF from the start and refrain from updating it during training. By doing so, our experimental results in \Cref{tab:etf_exp_balance} demonstrate that we can achieve parameter reductions of up to $20\%$ without sacrificing the performance of the model.\footnote{We use intersection over union (IoU) to measure model performances, in \MLab, we define $\text{IoU}(\hat{\mb y}, \mb y) = ||\mb y||_0 \cdot (\hat{\mb y}^\top \mb y)  \in [0,1] $. Here, the ground truth $\mb y$ represents a probability vector.}

\section{Conclusion}\label{sec:conclusion}
%xxxxxxxxxxxxxxx

%\qq{one paragraph conclusion, this section needs to be significantly reduced, move to appendices}
% \vspace{-0.1in}
% In this study, we extensively analyzed the NC phenomenon in \MLab~learning. In theory, our results establish that \MLab~ETFs are the only global minimizers of the PAL loss funcclassification.ting weight decay and bias. In practice, these findings hold significantimplications for improving the performance and efficiency of \MLab~tasks in both testing and training stages. More detailed discussions on future directions can be found in \Cref{app:related_works}. \textcolor{red}{Discuss about extreme multi-label classification: what is it? how our setting is realistic but not extreme? leave it as an open problem (cite Peng's work and the database mentioned by reviewer)}

In this study, we extensively analyzed the NC phenomenon in \MLab~learning. In theory, our results establish that \MLab~ETFs are the only global minimizers of the PAL loss function, incorporating weight decay and bias. In practice, these findings hold significant implications for improving the performance and efficiency of \MLab~tasks in both testing and training stages. Our work also fosters future directions in multi-label learning such as dealing with data imbalanced-ness, investigating other losses \cite{han2022neural, zhou2022optimization} or designing a better training loss are all worthy directions of pursuit. 

    % \item \textbf{Improving test performance by designing better decision rules.} By reducing \MLab~to \MClf, the PAL loss considered in this work leverages the dependency and relationship between each label. Nevertheless, it is worth noting that during the testing phase, PAL yields a real-valued vector as its output, requiring extra steps for converting it into a multi-hot binary label representation. To obtain the multi-hot binary label, varying threshold levels or employing different encoding techniques are required, which yield different multi-hot labels and hence result in unpredictable testing performance. Therefore, better decision rules are needed for \MLab~learning. Based upon the observation of the \MLab~NC, we could potentially design a better threshold-free decision rule simply by using the feature nearest center. More specifically, the \MLab~NC demonstrates that features collapse to their means across all multiplicities. With this geometric understanding, a simple yet intuitive approach for encoding the predicted real value vector into a multi-hot representation for testing data is to perform a nearest class center calculation between the feature of test data and all the feature means. 
     
\paragraph{Dealing with data imbalanced-ness.} As many multi-label datasets are imbalanced, another important direction is to investigate the more challenging cases where the Multiplicity-1 training data are imbalanced. We suspect a more general minority collapse phenomenon would happen \cite{fang2021exploring,thrampoulidis2022imbalance,zhai2023investigating}. A promising approach might be employing the \emph{Simplex-Encoded-Labels Interpolation (SELI) framework}, which is related to the singular value decomposition of the simplex-encoded label matrix \cite{thrampoulidis2022imbalance}. Nonetheless, we conjecture that the scaled average property will still hold between higher multiplicity features and their multiplicity-1 features. On the other hand, when  Multiplicity-1 training data are imbalanced, it is also worth studying creating a balanced dataset through data augmentation by leveraging recent advances in diffusion models \cite{ho2020denoising,trabucco2023effective,zhang2023emergence}. 
    % \LP{We could use diffusion model \cite{ho2020denoising} to learn the hidden distribution of the minority class in order to generate more data to create the overall balanced dataset.}

    % \item \textbf{using diffusing model to create balanced datasets.}
    
\paragraph{Designing better training loss.} Prior research has underscored the significance of mitigating within-class variability collapse to improve the transferability of learned models \cite{li2022principled}. In the context of \MLab~problems, the principle of Maximal Coding Rate Reduction (MCR$^2$) has been designed and effectively employed to foster feature diversity and discrimination, thereby preventing collapse \cite{yu2020learning,chan2022redunet}. We believe that our \MLab~NC could offer insights into the development of analogous loss functions for \MLab~learning, with the goal of promoting diversity in features.
    %NC not welcome for transfer leanring. \MClf has MCR2 to counter act NC, further analysis needed to see if \MLab~NC is pervasive across different losses. If so, similar to MCR2, \MLab NC could be used to designed better multi-label losses that does not collapse.}
% \end{itemize}

% \begin{itemize}
\paragraph{Extreme Multi-label classification.} The goal of extreme multi-label classification is to tag a data point with the most relevant subset of labels from an extremely large label set. We have explored our \MLab~NC phenomenon on the more challenging Microsoft COCO dataset \cite{lin2014microsoft} with minimal preprocessing to maintain balance among multiplicity-1 data. The visualization of \MLab~NC measures is shown in \cref{fig:cocoNC}. 
\begin{figure*}[h]
\centering
\includegraphics[width=0.95\textwidth]{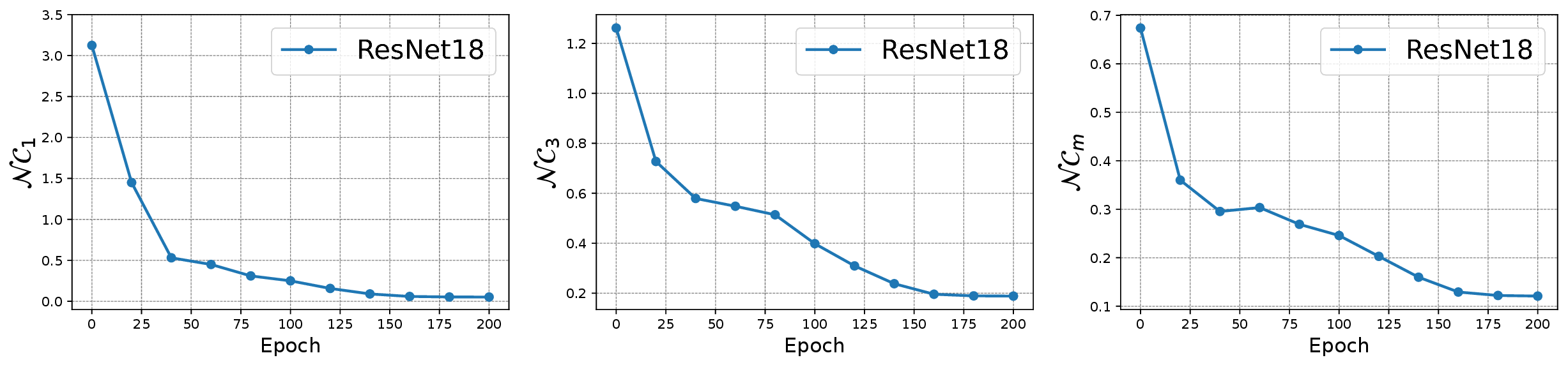} \    \vspace{-0.1in}
    \caption{\MLab~NC phenomenon in extreme MS-COCO dataset. As we can see that all \MLab~NC measures converges to small values.}
    \label{fig:cocoNC}
\end{figure*}
Specifically, a subset of the COCO dataset comprising 32,083 samples was extracted. We included 50 classes with all higher multiplicity samples possible. Most samples belong to higher multiplicity (29,583 samples) compared to multiplicity-1 samples (only 2,500). The ResNet18 architecture was trained for 200 epochs, achieving a training IoU of around $98 \%$. For training efficiency, we fixed the classifier W as an ETF and downsampled images to 64 by 64 pixels. Under these experimental setups, we observed that $\mathcal{NC}_1$, $\mathcal{NC}_3$, and $\mathcal{NC}_m$ all converged to small values, demonstrating the Mlab-NC phenomenon.
% \begin{figure*}[h]
% \centering
% \includegraphics[width=0.95\textwidth]{figures/coco_NCs.pdf} \    \vspace{-0.1in}
%     \caption{\MLab~NC phenomenon in extreme MS-COCO dataset. As we can see that all \MLab~NC measures converges to small values.}
%     \label{fig:cocoNC}
% \end{figure*}

\begin{wrapfigure}{r}{7.5cm}
\vspace{-0.8cm} 
\centering{\includegraphics[width=7.5cm]{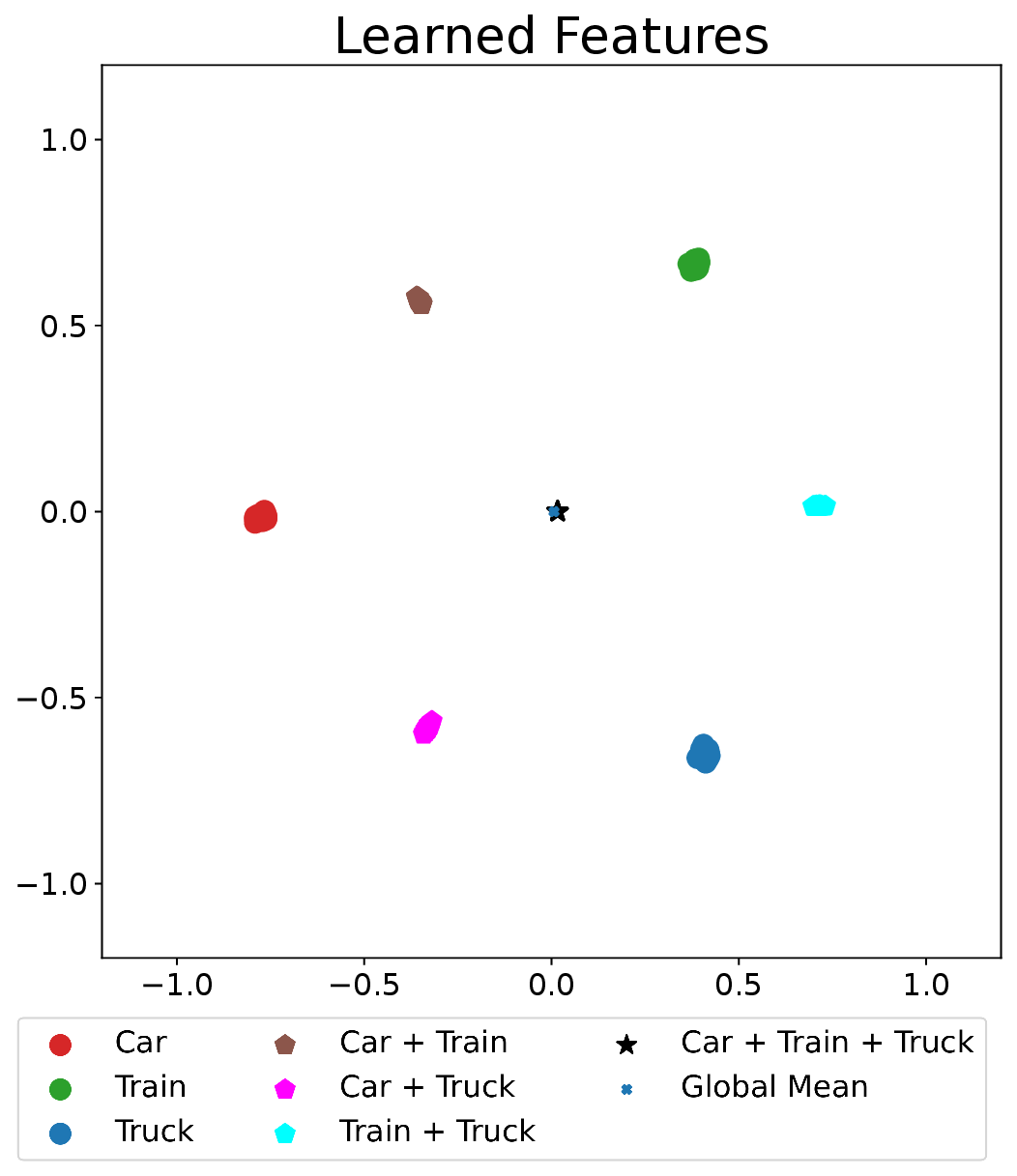}} \vspace{-0.7cm}   
     \caption{\MLab~NC: learned features for data samples with all possible labels is at the origin and follows tag-wise average property.}
     \label{fig:toy_cocoNC_origin}
     \vspace{-0.8cm}
\end{wrapfigure} 
While the Microsoft COCO dataset indeed contains a large number of labels, there is still a gap between its setup and the most extreme multi-label datasets. The work by \cite{jiang2023generalized} studied the phenomenon of neural collapse under extreme multi-class classification where the number of labels exceeds the dimension of learned features. To our knowledge, no one has studied neural collapse in the context of extreme multi-label classification. Our tentative experiments on the Microsoft COCO dataset demonstrate the potential to develop a generalized neural collapse for multi-label classification with a large number of labels.

\paragraph{Features for data with all possible tags.}
% While our theoretical analysis on the scaled tag-wise average property is based on the assumption that the multiplicity of a data point is smaller than total labels (i.e. $M < K$), experimental results on the toy Microsoft COCO dataset indicate that the tag-wise average property still holds for data samples with multiplicity $M = K$. The visualization of learned data features are shown in \cref{fig:toy_cocoNC_origin}. Specifically, we pick three classes: Car, Train, and Truck, where each sample can have $1-3$ labels. This means that the dataset comprises samples in mul-1, mul-2, and mul-3, based on the number of labels they carry. As we can see that both the global feature mean and the multiplicity-3 features are indeed at the origin. 
While our theoretical analysis of the scaled tag-wise average property assumes that the multiplicity of a data point is less than the total number of classes (i.e., $M<K$), experimental results on the toy Microsoft COCO dataset indicate that the tag-wise average property still holds for data samples with multiplicity $M=K$. The visualization of learned data features is shown in \cref{fig:toy_cocoNC_origin}. Specifically, we selected three classes: Car, Train, and Truck, where each sample can have $1$ to $3$ labels. This means that the dataset comprises samples with multiplicities of $1$, $2$, and $3$, based on the number of labels they carry. We then trained a 5-layer convolutional network using the selected data and we chose a feature dimension $2$ for visualization purposes. As shown in the visualization, both the global feature mean and the multiplicity-3 features are indeed located at the origin.

\section*{Code Availability}
Our code is publicly available at \url{https://github.com/Heimine/NC_MLab/}.

\section*{Acknowledgement}
The authors acknowledge support from NSF CAREER CCF-2143904, NSF CCF-2212066, NSF CCF-2212326, NSF IIS 2312842, ONR N00014-22-1-2529, an AWS AI Award, a gift grant from KLA, and MICDE Catalyst Grant. YW also acknowledges the support from the Eric and Wendy Schmidt AI in Science Postdoctoral
Fellowship, a Schmidt Futures program. Results presented in this paper were obtained using CloudBank, which is supported by the NSF under Award \#1925001.

%\newpage 
%\section*{Reproductivity Statement}
%\input{sections/reproductivity}

{\small 
\bibliographystyle{plain}
\bibliography{references}
}

% \section{Appendix}

\appendix
\newpage 

\onecolumn
\par\noindent\rule{\textwidth}{1pt}
\begin{center}
{\Large \bf Appendix}
\end{center}
\vspace{-0.1in}
\par\noindent\rule{\textwidth}{1pt}

%\qq{we can put a table of contents here}
%\py{done}

%\makeatletter
%\setcounter{tocdepth}{3}
%\def\contentsname{Contents}
%\def\tableofcontents{%
%    \section*{\MakeUppercase{\contentsname}}%
%    \@starttoc{toc}%
%    }
%\makeatother
%\tableofcontents

%\qq{use the command of table of contents to generate}

\section*{Organization of Appendices}

%\begin{table*}[h!]
%\centering
%\begin{tabular}{|c|clll|}           \hline
%\Cref{app:data_show} & \multicolumn{4}{c|}{\textbf{Dataset Illustrations, Visualizations and Experimental Setups}} \\ \hline
%\Cref{app:optimality} & \multicolumn{4}{c|}{\textbf{Proofs for Optimality Condition}}          \\ \hline
%\Cref{app:landscape} & \multicolumn{4}{c|}{\textbf{Analysis for Nonconvex Landscape}}             \\ \hline
%\end{tabular}
%\caption{\textbf{Table of Contents for Appendices}}
%\end{table*}

 In \Cref{app:data_show}, We present all the datasets utilized for validating multi-label NC, guiding testing and training, along with the experimental setups introduced in this study. In \Cref{app:optimality}
and \Cref{app:landscape}, we present the proofs for results from the main paper \Cref{thm:GO_thm} (global optimality) and \Cref{thm:landscape} (benign landscapes), respectively.

\section{Dataset Illustration and Visualization} \label{app:data_show}

\begin{figure*}[h!]
    \centering
    \subfloat[MLab-MNIST]{\includegraphics[width=0.20\textwidth]{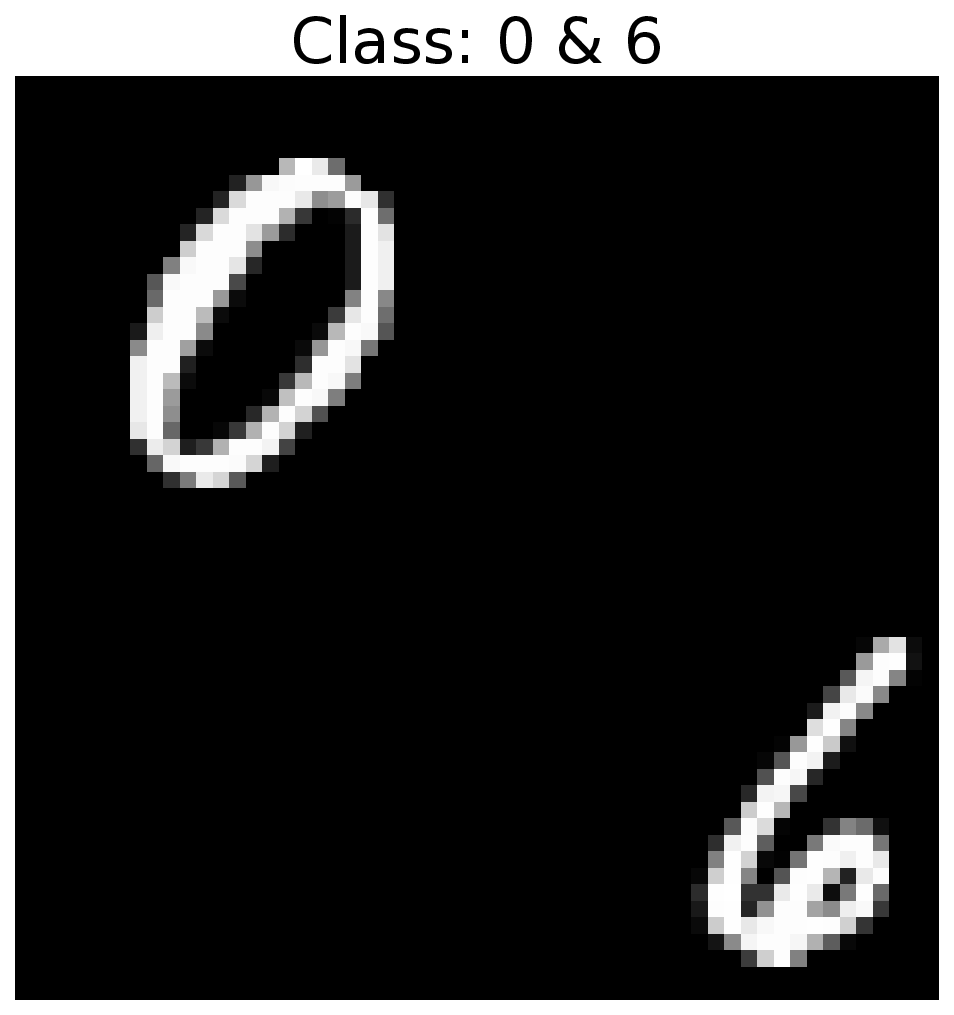}
    \includegraphics[width=0.20\textwidth]{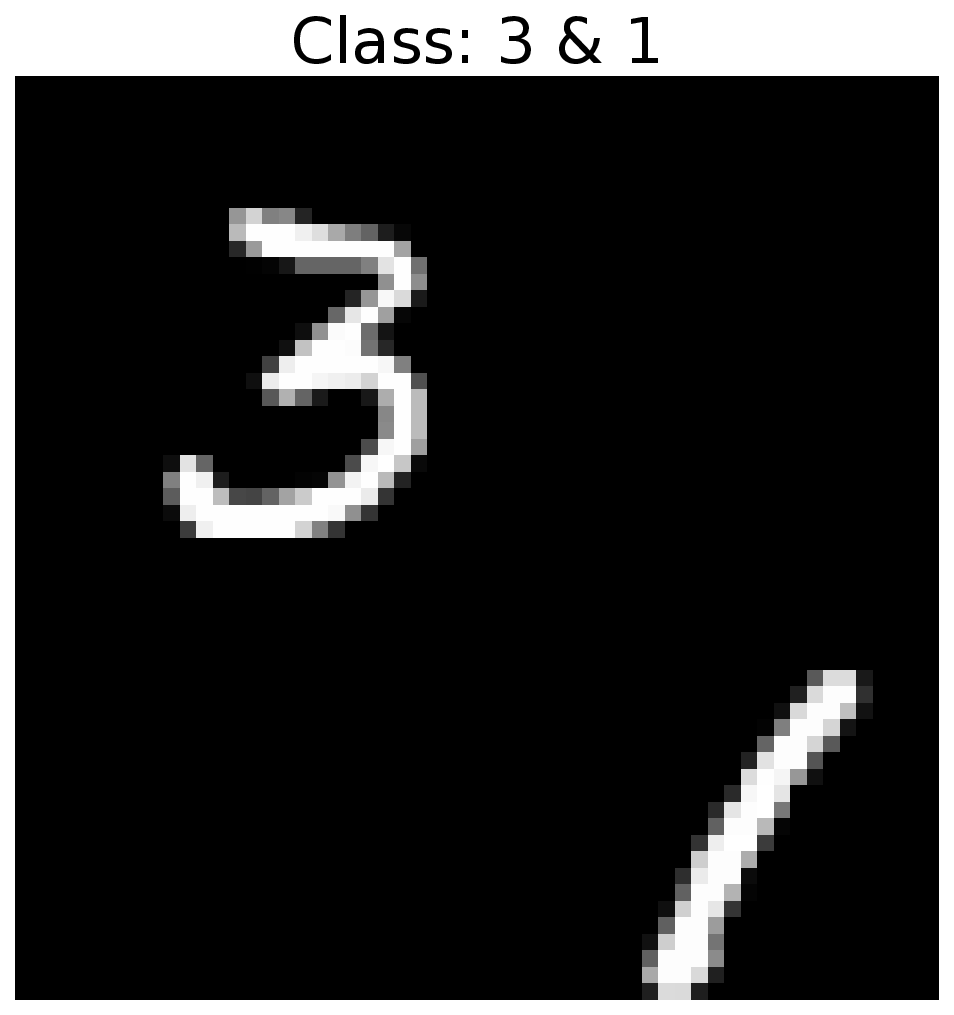}} \
    ~
    \subfloat[MLab-Cifar10] {\includegraphics[width=0.20\textwidth]{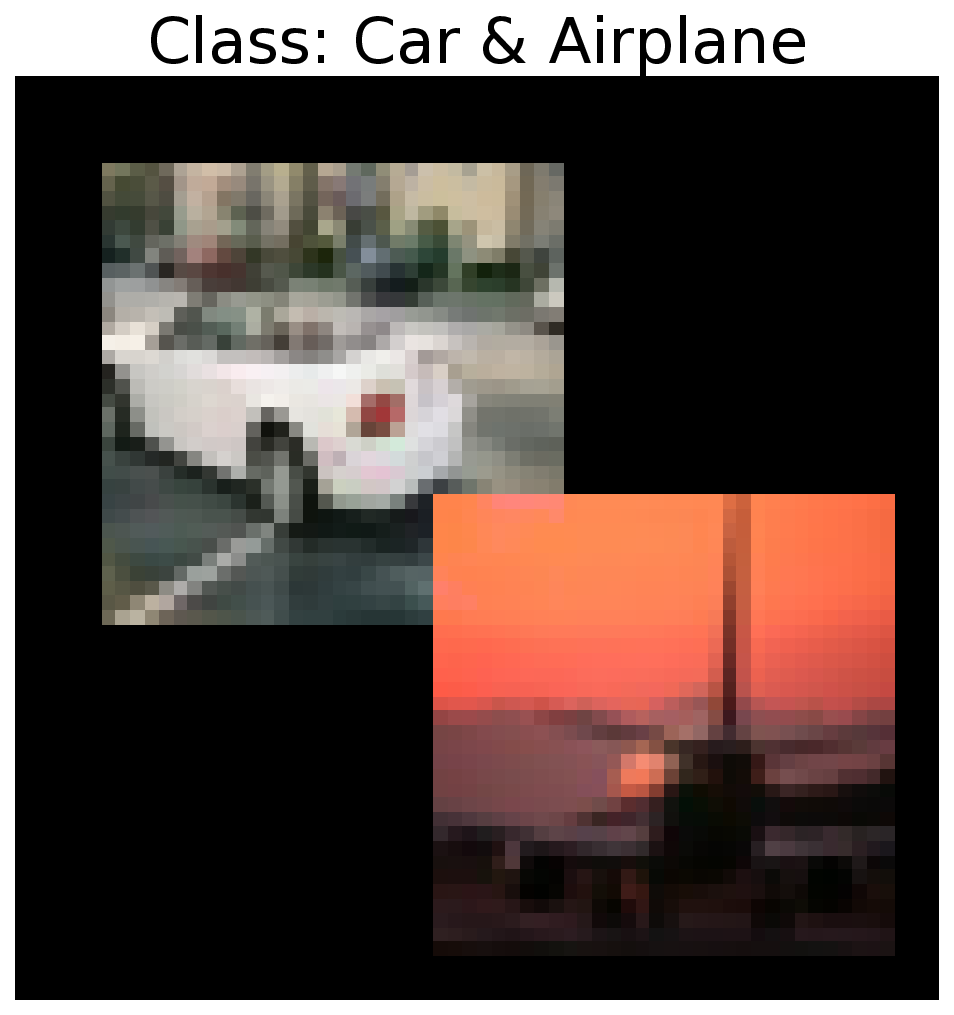}
    \includegraphics[width=0.20\textwidth]{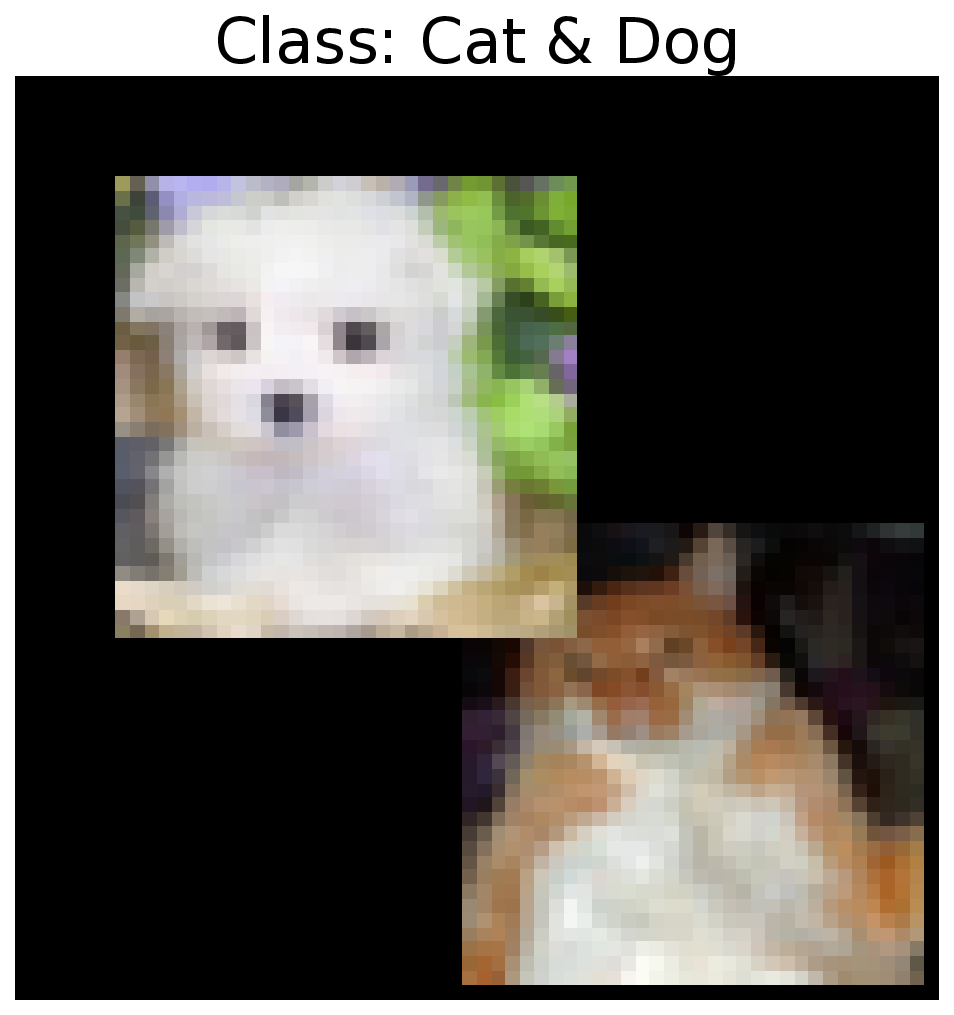}}\\
    \caption{\textbf{Illustration of synthetic multi-label MNIST (left) and Cifar10 (right) datasets.
    }}
    \label{fig:sample_mnist_c10}
\end{figure*}

\subsection{\MLab~MNIST and \MLab~Cifar10 dataset} \label{appendix-section:Mlab_MNIST_C10}
We created synthetic Multi-label MNIST \cite{lecun2010mnist} and Cifar10 \cite{krizhevsky2009learning} datasets by applying zero-padding to each image, increasing its width and height to twice the original size, and then combining it with another padded image from a different class. An illustration of generated multi-label samples can be found in Figure \ref{fig:sample_mnist_c10}. To create the training dataset, for $m=1$ scenario, we randomly picked $3100$ images in each class, and for $m=2$, we generated 200 images for each combination of classes using the pad-stack method described earlier. Therefore, the total number of images in the training dataset is calculated as $10 \times 3100 + \binom{10}{2} \times 200 = 40000$. Those dataset are used to generate results in \Cref{fig:NC-MNIST-CIFAR10} and \Cref{tab:etf_exp_balance}.  %For the test dataset, we included 800 images for each class in the $m=1$ scenario and 50 images for each combination of classes in the $m=2$ scenario, resulting in a total of 10250 images. 

In terms of training deep networks for \MLab, we use standard ResNet \cite{he2016deep} and VGG \cite{simonyan2014very} network architectures. Throughout all the experiments, we use an SGD optimizer with fixed batch size $128$, weight decay  $(\lambda_{\bW}, \lambda_{\bH}) = (5 \times 10^{-4}, 5 \times 10^{-4})$ and momentum $0.9$. The learning rate is initially set to $1 \times 10^{-1}$ and dynamically decays to $1 \times 10^{-3}$ following a CosineAnnealing learning rate scheduler as described in \cite{loshchilov2017sgdr}. The total number of epochs is set to $200$ for all experiments.

\subsection{Multiplicity 2 imbalanced \MLab~MNIST and \MLab~Cifar10 dataset} \label{appendix-section:imbala_Mlab_MNIST_C10}
Following the same padding rule described in \Cref{appendix-section:Mlab_MNIST_C10}, the  multiplicity-2 imbalanced data used to generate \Cref{fig:NC_High_imba_c10_Mnist} are created as follows. The cifar10 dataset has balanced Multiplicity-$1$ samples ($5000$ for each class). For the classes of Multiplicity-$2$, we divide them into $3$ groups: the large group ($500$ samples), the middle group ($50$ samples), and the small group ($5$ samples).

\begin{figure*}[t]
    \centering
    \subfloat[The SVHN Dataset]{\raisebox{0.6cm}{\includegraphics[width=0.38\textwidth]{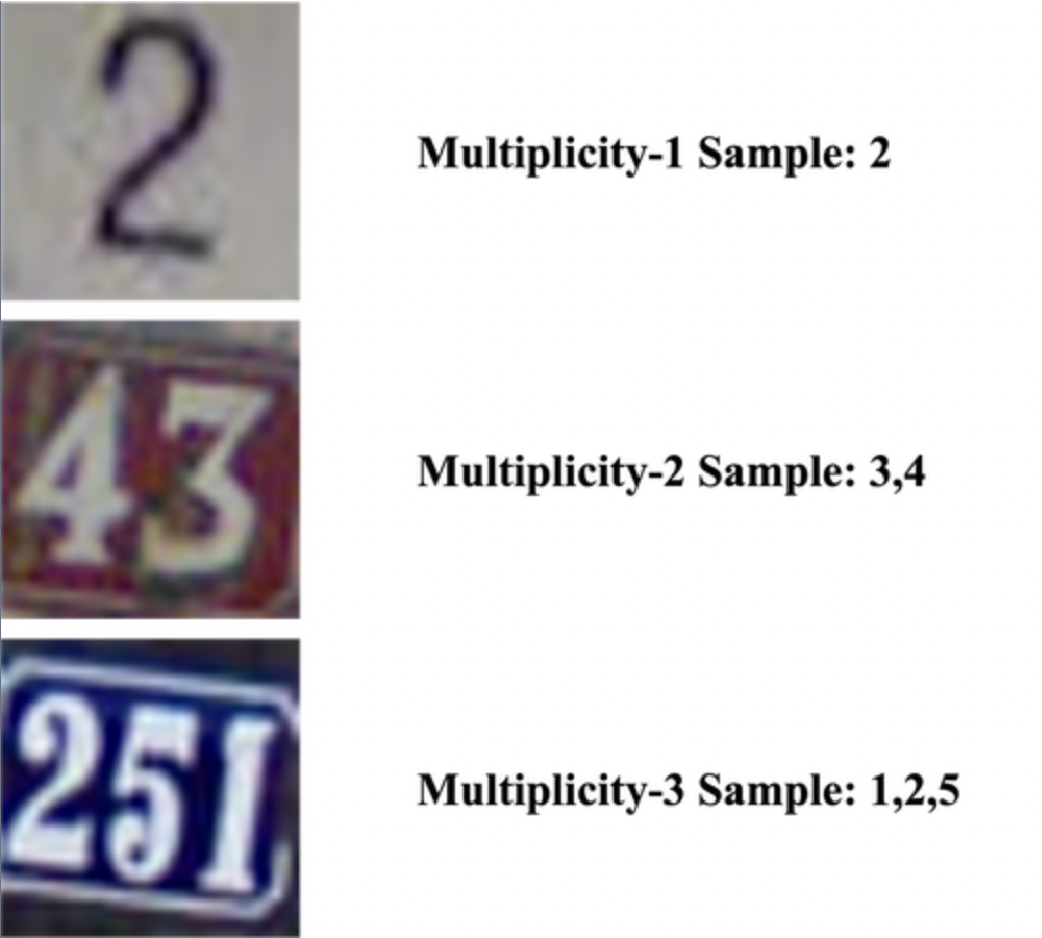}}}\
    \hspace{0.1in}
    \subfloat[Class Distribution]{\includegraphics[width=0.44\textwidth]{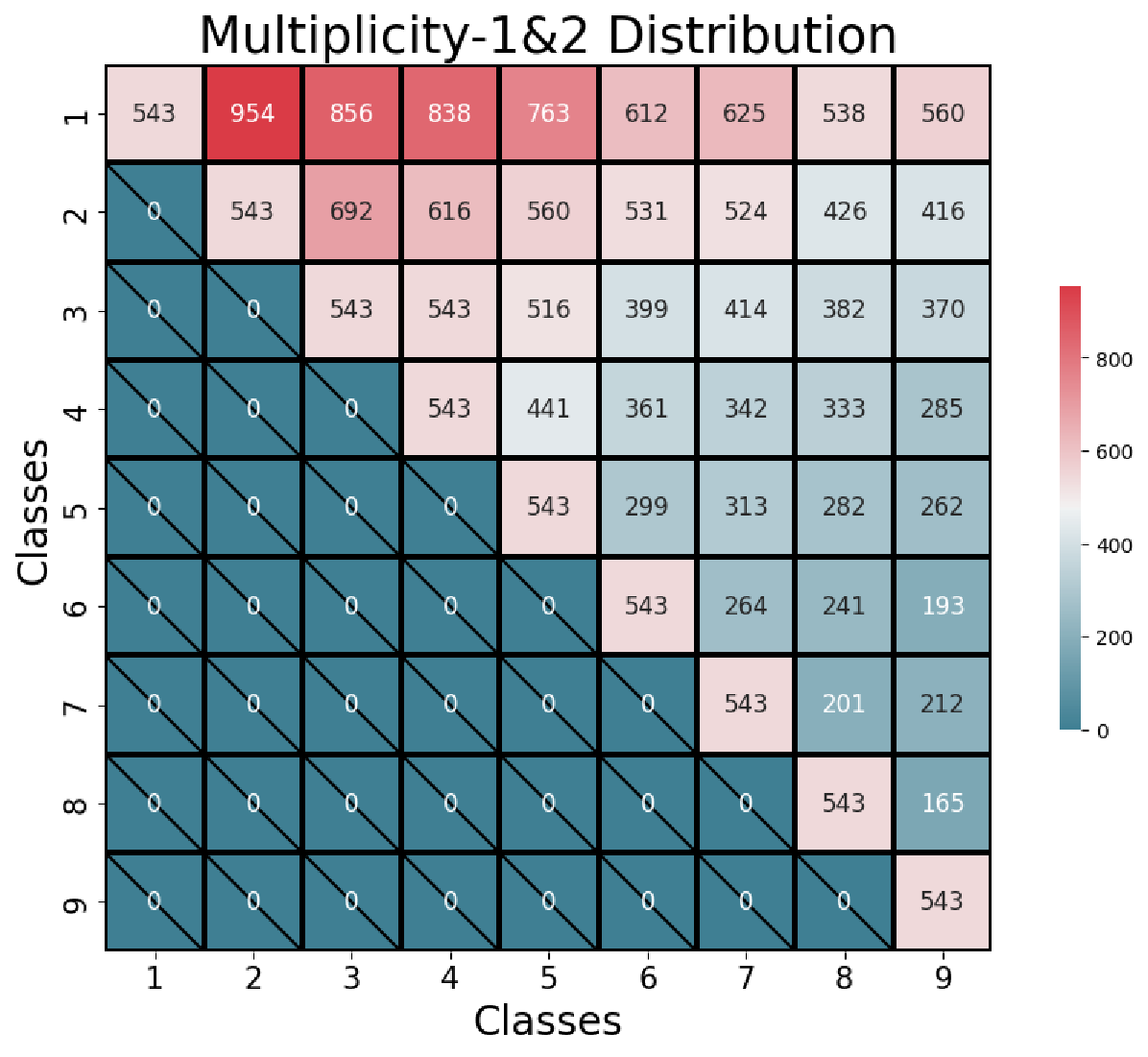}} \
    \caption{\textbf{Illustration of \MLab~SVHN dataset.} As illustrated in (a), the Street View House Numbers (SVHN) Dataset \cite{netzer2011reading} comprises labeled numerical characters and inherently serves as a \MLab~learning dataset. We applied minimal preprocessing to achieve balance specifically within the Multiplicity-1 scenario, as evidenced by the diagonal entries in (b). Furthermore, we omitted samples with Multiplicity-4 and above, as these images posed considerable recognition challenges. Notably, the Multiplicity-2 case remained largely imbalanced, as observed in the off-diagonal entries in (b). Nonetheless, our findings remained robust and consistent in this scenario, as evidenced in \Cref{fig:NC-SVHN}. }
    \label{fig:svhn_setup}
\end{figure*}
\subsection{Multiplicity 2 imbalanced \MLab~SVHN dataset} \label{appendix-section:imbala_SVHN}

To further explore our findings, we conducted additional experiments on the practical SVHN dataset \cite{netzer2011reading} alongside the synthetic datasets. In order to preserve the natural characteristics of the SVHN dataset, we applied minimal pre-processing only to ensure a balanced scenario for multiplicity-1, while leaving other aspects of the dataset untouched. The dataset are visualized in \Cref{fig:svhn_setup}.

% \begin{figure*}[t]
%     \centering
%     \subfloat[MLab-MNIST]{\includegraphics[width=0.20\textwidth]{figures/mnist_0_6.eps}
%     \includegraphics[width=0.20\textwidth]{figures/mnist_3_1.eps}} \
%     ~
%     \subfloat[MLab-Cifar10]{\includegraphics[width=0.20\textwidth]{figures/c10_car_airplane.eps}
%     \includegraphics[width=0.20\textwidth]{figures/c10_dog_cat.eps}} \\
%     \caption{\textbf{Illustration of synthetic multi-label MNIST (left) and Cifar10 (right) datasets.
%     }}
%     \label{fig:sample_mnist_c10}
% \end{figure*}

\subsection{Dataset used to compare test accuracy and efficiency between ONN and OvA} \label{appendix-section:test_data_show}

For training dataset, following the same generation method described in \Cref{appendix-section:Mlab_MNIST_C10} and simply varying data balanced-ness, the synthetic multiplicity imbalanced data used in \Cref{tab:ONN_OvA} are generated from Cifar10 datasets. All $3$ datasets has $1500$ sample in every class in multiplicity-1, we reduce the sample in every class in multiplicity-2 to $1000$, $750$, and $500$, resulting in the ``c10-Large", ``c10-Medium", and ``c10-Small" datasets. The testing datasets are independently generated, each with a sample size equivalent to $20\%$ of the training datasets.

Standard ResNet \cite{he2016deep} network architecture are used for training with only fixing the last layer classifier of as $ETF$. The SGD optimizer with fixed batch size $128$ are used. Specifically, for the three Cifar-10 datasets, models are trained with weight decay $(\lambda_{\bW}, \lambda_{\bH}) = (5 \times 10^{-5}, 10^{-5})$ with $200$ epochs with learning rate of $0.1$. For the SVHN dataset, models are trained with weight decay $(\lambda_{\bW}, \lambda_{\bH}) = (5 \times 10^{-6}, 1.5 \times 10^{-6})$ with $100$ epochs with learning rate of $0.09$. For testing with OvA, the additional linear classifiers are trained until the training loss reaches $0$, typically after approximately $2$ epochs.

\newpage
\section{Optimality Condition}\label{app:optimality}
The purpose of this section is to prove \Cref{thm:GO_thm}.
As such, throughout this section, we assume that we are in the situation of the statement of said theorem. 
Due to the additional complexity of the \MLab~setting compared to the \MClf~setting, analysis of the \MLab~NC requires substantially more notations. 
These notations, which are defined in \cref{appendix-section:notations}, while not necessary for \emph{stating} \Cref{thm:GO_thm},  are crucial for the proofs in \cref{appendix-section:main_proof} .

\subsection{Additional notations}\label{appendix-section:notations}

For the reader's convenience, we recall the following:
\begin{align}
N &:= \mbox{number of samples}\\
N_m &:= \mbox{number of samples \(i \in [N]\) such that \(|S_i| =m\)} \\
n_m &:= N_m / \binom{K}{m}\\
\binom{[K]}{m} &:= \{ S \subseteq [K]: |S| = m\} \\
M &:= \mbox{largest \(m\) such that \(n_m \ne 0\)} \\
d &:= \mbox{dimension of the last layer features}
\end{align}

\subsubsection{Lexicographical ordering on subsets}\label{appendix-section:lex-ordering}
 For each \(m \le K\), recall from the above that the set of subsets of \([K]\) of size \(m\) is denoted by the commonly used, suggestive notation \(\binom{[K]}{m}\).
Moreover, \(|\binom{[K]}{m} | = \binom{K}{m}\).

\textbf{\(\rhd\) Notation convention}. Assume the \emph{lexicographical ordering} on \(\binom{[K]}{m}\). Thus, for each \(k \in \binom{K}{m}\), the \emph{\(k\)-th subset of \(\binom{[K]}{m}\)} is well-defined. 

For example, when \(K = 5\) and \(m = 2\), there are \(\binom{5}{2} = 10\) elements in \(\binom{[5]}{2}\) which, when listed in the lexicographic ordering, are
\[\underbrace{\{1, 2\}}_{\mbox{1st}},
\underbrace{\{1, 3\}}_{\mbox{2nd}},
\underbrace{\{1, 4\}}_{\mbox{3rd}},
\underbrace{\{1, 5\}}_{\mbox{4th}},
\underbrace{\{2, 3\}}_{\cdots},
\{2, 4\},
\{2, 5\},
\{3, 4\},
\{3, 5\},
\underbrace{\{4, 5\}}_{\mbox{10th}}.
\]
In general, we use the notation \(S_{m,k}\) to denote the \(k\)-th subset of \(\binom{[K]}{m}\). In other words,
\[
\binom{[K]}{m} = \{S_{m,1},\,\,S_{m,2},\,\,\dots, \,\, S_{m,\binom{K}{m}}\}.
\]

\subsubsection{Block submatrices of the last layer feature matrix}
\label{appendix-section:block-submatrices-of-H}

Without the loss of generality, we assume that the sample indices \(i \in [N]\) are sorted such that
\(|S_i|\) is non-decreasing, i.e.,
\(
|S_1| \le \cdots \le |S_i| \le \cdots \le |S_N|.\)
Clearly, this does not affect the optimization problem itself.
Denote the set of indices of Multiplicity-\(m\) samples by
\(
\mathcal{I}_m := \{ i \in [N]: |S_i| = m\}
\).
Thus, we have
\[
\mathcal{I}_1 = \{1,\dots, N_1\}, \,\,
\mathcal{I}_2 = \{1+N_1,\dots, N_2 + N_1\}, \,\,
\cdots
\mathcal{I}_m = \{1 + \sum_{\ell=1}^m N_\ell, \dots, N_m + \sum_{\ell=1}^m N_\ell\},
\cdots
\,\,
\]
Below, it will be helpful to define the notation
\[
\mathcal{I}_{m,S} := \{ i \in [N] : S_i = S\}
\]
for each \(m = 1,\dots, M\) and \(S \in \binom{[K]}{m}\).

\textbf{\(\rhd\) Notation convention}. Define the block-submatrices \(\bH_1,\dots, \bH_M\) of \(\bH\) such that
\begin{enumerate}
    \item \(\bH_m \in \mathbb{R}^{d \times N_m}\)
    \item \(\bH = \begin{bmatrix}
        \bH_1 & \bH_2 & \cdots & \bH_M
    \end{bmatrix}\)
\end{enumerate}
Thus, as in the main paper, the columns of \(\bH_m\) correspond to the features of \(\mathcal{I}_m\).

\subsubsection{Decomposition of the PAL-CE loss}\label{appendix-section:decomposition-of-the-loss}

Define
\begin{equation}
\label{equation-definition:gm}    
g_m(\bW\bH_m + \bbb,\bY) :=
\frac{1}{N_m}
\sum_{i \in \mathcal{I}_m}
\underline{\mathcal{L}}_{\mathtt{PAL}}(\bW\bh_i + \bbb, \by_{S_i}).
\end{equation}
Intuitively, \(g_m\) is the contribution to \(g\) from the Multiplicity-\(m\) samples.
More precisely, the function \(g(\bW\bH + \bbb,\bY)\) from
\Cref{eq:DL-loss-ufm} can be decomposed as
\begin{align} \label{eq:sum_up_gm} 
g(\bW\bH + \bbb,\bY) =
\sum_{m=1}^M
\frac{N_m}{N}
g_m
(\bW\bH_m + \bbb,\bY).
\end{align}
% \subsubsection{The balanced-ness condition -- precise statement}

\subsubsection{Triple indices notation}\label{appendix-section:triple-subscript-index}

Next, we state precisely the data balanced-ness condition from \Cref{thm:GO_thm}. In order to state the condition, we need some additional notations.
Fix some \(m \in \{1,\dots ,M \}\) and  let \(S \in \binom{[K]}{m}\). Define
\begin{equation}
    n_{m,S} := \{i \in [N] : S_i = S\}.
\end{equation}
\Cref{thm:GO_thm} made the following
\textbf{data balanced-ness
condition}: 
\begin{equation}
\mbox{\(n_{m,S} = N_m/ \binom{K}{m} =: n_m\) for all \(S  \in \binom{[K]}{m}\).}
\label{equation:data-balanced-ness-precise}
\end{equation}
In other words, for a fixed \(m \in [M]\), the set  \(\mathcal{I}_{m,S}\) has the same constant cardinality equal to \(n_m\) ranging across all \(S \in \binom{[K]}{m}\).

By the data balanced-ness condition, we have for a fixed \(m = 1,\dots, M\) that \(\mathcal{I}_{m,S}\) have the same number of elements across all \(S \in \binom{[K]}{m}\). Moreover, in our notation, we have 
\(|\mathcal{I}_{m,S}| = n_m\).
Below, for each \(m = 1,\dots, M\) and for each \(S \in \binom{[K]}{m}\), choose an arbitrary ordering on \(\mathcal{I}_{m,S}\) once and for all.
Every sample is \emph{uniquely} specified by the following three indices:
\begin{enumerate}
    \item \(m \in [M]\) the sample's multiplicity, i.e., \(m = |S|\)
    \item \(k \in \binom{K}{m}\) the index such that \(S_{m,k}\) is the label set of the sample,
    \item \(i \in [n_m]\) such that the sample is the \(i\)-th element of \(\mathcal{I}_{m,S_{m,k}}\).
\end{enumerate}
More concisely, we now introduce the

\textbf{\(\rhd\) Notation convention}. Denote each sample by the triplet
\begin{equation}
\label{equation:triple-indices-notation}    
(m,k,i) \quad \mbox{where} \quad m \in [M], \, \, k \in \binom{K}{m}, \,\, i \in [n_m].
\end{equation}
Below, \eqref{equation:triple-indices-notation} will be referred to as the \textbf{triple indices notation} and every sample will be referred to by its triple indices \((m,k,i)\) instead of the previous single index \(i \in [N]\).
Accordingly, throughout the appendix, columns of \(\bH\) are expressed as \(\bh_{m,k,i}\) instead of the previous \(\bh_i\),
and thus the block submatrix \(\bH_m\) of \(\bH\) can be, without the loss of generality, be written as
\(
\bH_m
=
\begin{bmatrix}
    \bh_{m,k,i}
    \end{bmatrix}_{m \in [M], \, \, k \in \binom{K}{m}, \,\, i \in [n_m]}.
\)

Moreover, in the triple indices notation, 
\Cref{equation-definition:gm}     can be rewritten as
\begin{equation}
\label{equation-definition:gm-triplet}
    g_m(\bW \bH_m + \bbb) = \frac{1}{N_m}\sum_{i = 1}^{n_m} \sum_{k = 1}^{\binom{K}{m}}{\mathcal{L}_{\mathtt{PAL}}(\bW\bh_{m, k, i}, \by_{S_{m, k}})} 
\end{equation}
\subsection{Proofs} \label{appendix-section:main_proof}
We will first state the proof of \Cref{thm:GO_thm} which depends on several lemmas appearing later in the section. Thus, the proof of \Cref{thm:GO_thm} serves as a roadmap for the rest of this section.

\begin{proof}[Proof of \Cref{thm:GO_thm}] 
Recall the definition of a coercive function: a function \(\varphi: \mathbb{R}^n \to \mathbb{R}\) is said to be coercive if \(\lim_{\|x\| \to \infty} \varphi(x) = +\infty\). It is well-known that a coercive function attains its infimum which is a global minimum.

Now, note that the objective function \(f(\bW,\bH,\bbb)\) in Problem \eqref{eq:DL-loss-ufm} is \textit{coercive} due to the weight decay regularizers (the terms \(\|\bW\|_F^2\), \(\|\bH\|_F^2\) and \(\|\bbb\|_F^2\)) and that the pick-all-labels cross-entropy loss is non-negative. Thus, a global minimizer, denoted below as $(\bm W, \bm H, \bm b)$, of Problem \eqref{eq:DL-loss-ufm} exists. By Lemma B.2, we know that any critical point $(\bm W, \bm H, \bm b)$ of Problem \eqref{eq:DL-loss-ufm} satisfies
\[
\bW^{\top}\bW = \frac{\lambda_{\bH}}{\lambda_{\bW}} \bH \bH^{\top}.
\]
Let $\rho:=\|\bW \|_F^2$. Thus, $\|\bH \|_F^2 = \frac{\lambda_{\bW}}{\lambda_{\bH}} \rho$

We first provide a lower bound for the PAL cross-entropy term $g(\bW \bH + \bbb\mb{1}^\top)$ and then show that the lower bound is tight if and only if the parameters are in the form described in \Cref{thm:GO_thm}.
For each \(m=1,\dots, M\), let \(c_{1,m} >0\) be arbitrary, to be determined below.
Now by \Cref{lemma:gm-lower-bound} and \Cref{lemma:key-lower-bound-tightness}, we have
\begin{equation}
g(\bW \bH    + \bbb) - \Gamma_{2} \ge
- \frac{1}{N} \sqrt{
\sum_{m=1}^M
\left(\frac{1}{1+c_{1,m}} \frac{m}{K-m}\right)^2
\kappa_m n_m \binom{K}{m}^2
}\sqrt{\frac{\lambda_{\bW}}{\lambda_{\bH}} }\rho \nonumber
\end{equation}
where 
\(\Gamma_2 := \sum_{m=1}^M c_{2,m}\) and \(c_{2,m}\) is as in \Cref{lemma:key-lower-bound-tightness}. Therefore, we have 
\begin{align*}
    f(\bW, \bH, \bbb) &= g(\bW\bH + \bbb^{\top}) +
    \lambda_{\bW}\|\bW\|_F^2 + 
    \lambda_{\bH}\|\bH\|_F^2 + {\lambda_{\bbb}}\|\bbb\|_2^2 \\
    &\ge - \frac{1}{N} \sqrt{
\sum_{m=1}^M
\left(\frac{1}{1+c_{1,m}} \frac{m}{K-m}\right)^2
\kappa_m n_m \binom{K}{m}^2
}\sqrt{\frac{\lambda_{\bW}}{\lambda_{\bH}} }\rho + \Gamma_2 + 2\lambda_{\bW} \rho + \frac{\lambda_{\bbb}}{2}\|\bbb\|_2^2 \\
     &\ge - \frac{1}{N} \sqrt{
\sum_{m=1}^M
\left(\frac{1}{1+c_{1,m}} \frac{m}{K-m}\right)^2
\kappa_m n_m \binom{K}{m}^2
}\sqrt{\frac{\lambda_{\bW}}{\lambda_{\bH}} }\rho + \Gamma_2 + 2\lambda_{\bW} \rho
     \numberthis \label{proof:final_inequality}
\end{align*}
where the last inequality becomes an equality whenever either $\lambda_{\bbb} = 0$ or $\bbb = \mb{0}$. Furthermore, by \Cref{lemma:ETF_and_sacled_ave}, we know that the Inequality (\ref{proof:final_inequality}) becomes an equality \textit{if and only if} $(\bW, \bH, \bbb)$ satisfy the following:
\begin{align*}
 &\mbox{(I)}\quad   \|\bw^{1}\|_2 = \|\bw^{2}\|_2 = \cdots = \|\bw^{ K}\|_2, \quad \mbox{and} \quad \bbb = b \mb{1}, \\
    & \mbox{(II)}\quad  \frac{1}{\binom{K}{m}}
    \sum_{k = 1}^{\binom{K}{m}}
    \bh_{m,k,i} = \mb{0}, \quad \mbox{and} \quad \sqrt{\frac{\binom{K-2}{m-1}}{n_m}}\bw^k = 
     \sum_{\ell: k\in S_{m, \ell}}{\bh_{m, \ell, i}},
    % \bh_{m, \{k\}, i}
     \forall m \in [M], k \in [K], i\in [n_m], \\
     &\mbox{(III)}\qquad  \bW^\top\bW = \frac{\rho}{K-1}\left(\mb{I}_K - \frac{1}{K}\mb{1}_K \mb{1}_K^\top \right)
     % , \ \mbox{and} \ c_{1, m} =
\end{align*}

\hspace{0.1227cm} (IV) There exist unique positive real numbers \(C_1,C_2,\dots, C_M>0\) such that the following holds: 
\begin{align*}
    &\bh_{1,k,i} = C_1 \bw^{\ell} 
     \quad \qquad\qquad \mbox{when} \,\, S_{1,k} = \{\ell\},\; \ell\in [K], 
    &\mbox{(Multiplicity \(=1 \) Case)} \\
    &\bh_{m,k,i} = C_m  \textstyle{\sum_{\ell \in S_{m,k}} }\bw^{\ell} 
    \quad \mbox{when} \,\, m > 1.
    & \mbox{(Multiplicity \(>1\) Case)} 
\end{align*}
Note that condition (IV) is a restatement of \Cref{eq:h1_collapse}
and \Cref{eq:hm_collapse}. The choice of the \(c_{1,m}\)'s is given by (V) from \Cref{lemma:ETF_and_sacled_ave}.
\end{proof}

\begin{lemma}\label{lemma:KKT-condition}
we have: 
\[
\bW^{\top} \bW = \frac{\lambda_{\bH}}{\lambda_{\bW}} \bH \bH^{\top} \quad and \quad \rho = \|\bW\|_F^2 = \frac{\lambda_{\bH}}{\lambda_{\bW}}\|\bH\|_F^2
\] 
\end{lemma}
\begin{proof}
    The proceeds identically as in given by \cite{zhu2021geometric}~Lemma B.2 and is thus omitted here.
\end{proof}

% \begin{lemma}\label{lemma:Pascal_H_lemma}[Is this needed?]
% We assume that we have:
% \[
% \binom{K-2}{m-1}\|\bH_m \|_F^2 = 
% \|\bH_m \bD_m\|_F^2
% \]
% where \[
% \bD_m = \mathrm{diag}(\underbrace{\bY_m^\top, \cdots, \bY_m^\top}_{n_m \ \mathrm{of} \ \bY_m}) \in \mathbb{R}^{n_m \binom{K}{m} \times n_m K} \quad and \quad \bY_m \in \mathbb{R}^{K \times n_m }. 
% \]
% \end{lemma}
The following lemma 
is the generalization of \cite{zhu2021geometric}~Lemma B.3 to the multilabel case for each multiplicity.

\begin{lemma}\label{lemma:gm-lower-bound}
% \yw{Actually we need the WHb to be a critical point of \(f\), not \(g\). Because we need the lambda relations.}
Let \((\bW,\bH,\bbb)\) be a critical point for the objective \(f\) from Problem \eqref{eq:DL-loss-ufm}. Let \(c_{1,m} > 0\) be arbitrary and let \(\gamma_{1,m} := 
\frac{1}{1+c_{1,m}} \frac{m}{K-m}
\).
Define
\(\kappa_m := 
\left(\frac{K}{m \binom{K}{m}}\right)^2 \binom{K-2}{m-1}
\).
% and let
% $c_{3, m} := \frac{\frac{\alpha}{\binom{K}{m}} \sqrt{\binom{K-2}{m-1}} \|\bH_m\|_F}{n_m \|\bW\|_F}$.
% $c_{3, m} :=
% \sqrt{\frac{\kappa_m}{n_m}}
% \frac{ \|\bH_m\|_F}{ \|\bW\|_F}$.
% Let $\bW= 
% \begin{bmatrix}
%     \bw^1 & \cdots & \bw^K
% \end{bmatrix}^\top
% \in \mathbb{R}^{K \times d}$, and, without losing for generosity, let $\bH$ has the block notation for each multiplicity, i.e., $\bH= 
% \begin{bmatrix}
%     \ \bH_1 \ | \ \bH_2\ | \  \cdots\ |\ \bH_m \ 
% \end{bmatrix}
% = 
% \begin{bmatrix}
%     \bh_{1, 1, 1} \cdots \bh_{m, k, n_m}
% \end{bmatrix}
% \in \mathbb{R}^{d \times N}$, $N := \sum_{m = 1}^{K-1}{N_m} = \sum_{m = 1}^{K-1} \binom{K}{m} n_m$.
Then
% \begin{align*}
%     g_m(\bW \bH_m + \bbb) - c_{2,m} \ge - \gamma_{1,m} \frac{c_{3,m}}{2}\|\bW\|_F^2 - \gamma_{1,m} 
% \frac{\kappa_m}{2 c_{3,m} n_m}
% \|\bH_m\|_F^{2}. \numberthis\label{gm_lower_bound}
% \end{align*}
\begin{equation}
g(\bW \bH    + \bbb) - \Gamma_{2} \ge
- \frac{1}{N} \sqrt{
\sum_{m=1}^M
\left(\frac{1}{1+c_{1,m}} \frac{m}{K-m}\right)^2
\kappa_m n_m \binom{K}{m}^2
}\sqrt{\frac{\lambda_{\bW}}{\lambda_{\bH}} }\rho. \numberthis \label{g_lower_bound}
\end{equation}
where \(\rho := \|\bW\|_F^2\),
\(\Gamma_2 := \sum_{m=1}^M c_{2,m}\) and 
\(c_{2,m}\) is as in \Cref{lemma:key-lower-bound-tightness}. 
\end{lemma}
% \begin{remark}
    Note that \(\Gamma_2\) depends on \(c_{1,1},\, c_{1,2},\,\dots, \,c_{1,M}\) because \(c_{2,m}\) depends on \(c_{1,m}\) for each \(m \in [M]\).
% \end{remark}
\begin{proof}
Throughout this proof, let $\bz_{m, k, i} := \bW \bh_{m, k, i} + \bbb$ and choose the same $\gamma_{1,m}, c_{2,m}$ for all $i$ and $k$. The first part of this proof aim to find the lower bound for each $g_m(\bW, \bH_m, \bbb)$ along with conditions when the bound is tight. The rest of the proof focus on sum up $g_m$ to get \Cref{g_lower_bound}. Thus, using \Cref{equation-definition:gm-triplet} with the \(\bz_{m, k, i}\)'s, we have that $g_m$ can be written as
\begin{align}
    g_m(\bW \bH_m + \bbb) &= \frac{1}{N_m}\sum_{i = 1}^{n_m} \sum_{k = 1}^{\binom{K}{m}}{\underline{\mathcal{L}}_{\mathtt{PAL}}(\bz_{m, k, i}, \by_{S_{m, k}})} 
\end{align}
By directly applying \Cref{lemma:key-lower-bound-tightness}, the following lower bound holds:
\begin{align*}
    N_m g_m(\bW \bH_m + \bbb) &\ge \gamma_{1,m} \sum_{i = 1}^{n_m} \sum_{k = 1}^{\binom{K}{m}}{\langle \mb{1} -  \tfrac{K}{m}\mathbb{I}_S,\, \bW \bh_{m, k, i} +\bbb \rangle} + N_m c_{2,m} \\
\end{align*}
which implies that
\begin{align*}
    &\ \quad \gamma_{1,m}^{-1}(g_m(\bW \bH_m + \bbb) - c_{2,m}) \\ &\ge \frac{1}{N_m} \sum_{i = 1}^{n_m} \sum_{k = 1}^{\binom{K}{m}}{\langle \mb{1} -  \tfrac{K}{m}\mathbb{I}_S,\, \bW \bh_{m, k, i} +\bbb \rangle} \\
    &=\frac{1}{N_m} \sum_{i = 1}^{n_m} \underbrace{\sum_{k = 1}^{\binom{K}{m}}{\langle \mb{1} -  \tfrac{K}{m}\mathbb{I}_S,\, \bW \bh_{m, k, i} \rangle}}_{(\star)}
     + \frac{1}{N_m} \sum_{i = 1}^{n_m} \underbrace{\sum_{k = 1}^{\binom{K}{m}}{\langle \mb{1} - \tfrac{K}{m}\mathbb{I}_S, \bbb \rangle}}_{(\star\star)}\numberthis\label{gm_lb_1}
\end{align*}
To further simplify the inequality above, we break it down into two parts, namely, the feature part ($\star$) and the bias part ($\star\star$) and analyze each of them separately. We first  show that the term ($\star\star$) is equal to zero. To see this, note that
\begin{align*}
    (\star \star) &= \sum_{k = 1}^{\binom{K}{m}} \left( \sum_{j = 1}^{K}{b_j} - \frac{K}{m}\sum_{j^{\prime} \in S_{m, k}}{b_{j^{\prime}}} \right) \\ 
    &= \sum_{k = 1}^{\binom{K}{m}}\sum_{j = 1}^{K}{b_j} - \frac{K}{m}\sum_{k = 1}^{\binom{K}{m}}\sum_{j^{\prime} \in S_{m, k}}{b_{j^{\prime}}} \\
    &= K\binom{K}{m}\bar{b} - \frac{K}{m} m \binom{K}{m} \bar{b} \\
    &= 0 \numberthis \label{star_0}
\end{align*}
where $\bar{b} = \frac{1}{K}\sum_{j = 1}^{K}{b_j}$ and $\sum_{k = 1}^{\binom{K}{m}}\sum_{j = 1}^{K}{b_j} = K \binom{K}{m} \bar{b}$.  Thus
\[
\sum_{k = 1}^{\binom{K}{m}}
    \sum_{j' \in S_{m,k}}
    b_{j'}
    \overset{(\diamondsuit)}
    {=}
    \sum_{j=1}^K
% \sum_{k = 1}^{\binom{K}{m}}
    \sum_{k:j \in S_{m,k}}
    b_{j}
    {=}
    \sum_{j=1}^K
% \sum_{k = 1}^{\binom{K}{m}}
b_j
\# \{ 
k:j \in S_{m,k}
\}
=
    \sum_{j=1}^K
    \binom{K}{m} \frac{m}{K} b_j
    =
    m \binom{K}{m} \overline{b}.
\]
Note that the equality at $(\diamondsuit)$ holds
 by switching the order of the summation.
Now, substituting the result of \Cref{star_0} into the Inequality (\ref{gm_lb_1}), we have the new lower bound of $g_m$: 
\begin{align}
    \gamma_{1,m}^{-1}(g_m(\bW \bH_m + \bbb) - c_{2,m}) &\ge \frac{1}{N_m} \sum_{i = 1}^{n_m} \underbrace{\sum_{k = 1}^{\binom{K}{m}}{\langle \mb{1} -  \tfrac{K}{m}\mathbb{I}_S,\, \bW \bh_{m, k, i} \rangle}}_{(\star)}
\end{align}
and the bound is tight when conditions are met in \Cref{lemma:key-lower-bound-tightness}. To simplify the expression $(\star)$ we first distribute the outer layer summation and further simplify it as:
\begin{align*}
    (\star) &= \sum_{k=1}^{\binom{K}{m}} \sum_{j = 1}^{K}{\bh_{m, k, i}^{\top} \cdot \bw^{j}} - \frac{K}{m} \sum_{k=1}^{\binom{K}{m}} \sum_{j^{\prime} \in S_{m, k}}{\bh_{m, k, i}^{\top} \cdot \bw^{j^{\prime}}} \\
    &= \sum_{k=1}^{\binom{K}{m}} \sum_{j = 1}^{K}{\bh_{m, k, i}^{\top} \cdot \bw^{j}}  - \frac{K}{m} \sum_{j=1}^{K}
   \sum_{k' : j \in S_{k'}}
    \bh_{m,k',i}^\top \bw^{j} \numberthis \label{switch_sum1} \\
    &= \sum_{k=1}^{\binom{K}{m}} \sum_{j = 1}^{K}{\bh_{m, k, i}^{\top} \cdot \bw^{j}} - \frac{K}{m} \sum_{j=1}^K
\bh_{m,\{j\},i}^\top \bw^{j} \\
    &= \sum_{j = 1}^{K}\sum_{k=1}^{\binom{K}{m}}{\bh_{m, k, i}^{\top} \cdot \bw^{j}} - \frac{K}{m} \sum_{j=1}^K
\bh_{m,\{j\},i}^\top \bw^{j} \numberthis \label{switch_sum2} \\
    & = \sum_{j = 1}^{K} \left( \sum_{k=1}^{\binom{K}{m}}{\bh_{m, k, i}} - \frac{K}{m} \bh_{m, \{j\}, i}\right)^{\top} \bw^{j} \\
    & = \sum_{j = 1}^{K} \left(\binom{K}{m}\overline{\bh}_{m, \bullet, i} - \frac{K}{m} \bh_{m, \{j\}, i}\right)^{\top} \bw^{j} \numberthis\label{star_final}
\end{align*}

where we let $\bh_{m, \{j\}, i} = \sum_{k: j\in S_{m, k}}{\bh_{m, k, i}}$ and \(\overline{\bh}_{m,\bullet,i}\) be the ``average'' of 
\(\bh_{m,k,i}\) over all \(k \in \binom{K}{m}\) defined as:
\begin{equation}
\overline{\bh}_{m,\bullet,i}
:=
\frac{1}{\binom{K}{m}}
\sum_{k = 1}^{\binom{K}{m}}
\bh_{m,k,i}.
\end{equation} Similarly to $(\diamondsuit)$, the Equations (\ref{switch_sum1}) and (\ref{switch_sum2}) holds since we only switch the order of summation. Continuing simplification, we substitute the result in Equations (\ref{star_final}) and (\ref{star_0}) into Inequality (\ref{gm_lb_1}) we have: 
\begin{align*}
    \gamma_{1,m}^{-1}(g_m(\bW \bH_m + \bbb) - c_{2,m}) &\ge \frac{1}{N_m} \sum_{i = 1}^{n_m}{\sum_{j = 1}^{K} \left(\binom{K}{m}\overline{\bh}_{m, \bullet, i} - \frac{K}{m} \bh_{m, \{j\}, i}\right)^{\top} \bw^{j}} \\
    & = \frac{1}{N_m} \sum_{i = 1}^{n_m}{\sum_{k = 1}^{K} \left(\binom{K}{m}\overline{\bh}_{m, \bullet, i} - \frac{K}{m} \bh_{m, \{k\}, i}\right)^{\top} \bw^{k}} \\
    & = \frac{1}{n_m} \sum_{i = 1}^{n_m}{\sum_{k = 1}^{K} \left(\overline{\bh}_{m, \bullet, i} - \frac{K}{m\binom{K}{m}} \bh_{m, \{k\}, i}\right)^{\top} \bw^{k}}
\end{align*}
Furthermore, from the AM-GM inequality (e.g., see Lemma A.2 of \cite{zhu2021geometric}), we know that for any $\bu$, $\bv \in \mathbb{R}^{K}$ and any $c_{3,m} > 0$, 
\begin{equation} \label{eq:AMGM}
    \bu^{\top}\bv \leq \frac{c_{3,m}}{2} \|\bu\|_2^2 + \frac{1}{2 c_{3,m}} \|\bv\|_2^2
\end{equation}
where the above AM-GM inequality becomes an equality when $c_{3,m}\bu = \bv$. Thus letting $\bu = \bw^k$ and $\bv = \left( \overline{\bh}_{m, \bullet, i} - \frac{K}{m\binom{K}{m}}\bh_{m, \{k\}, i}\right)^{\top}$ and applying the AM-GM inequality, we further have: 
% \newpage
\begin{align*}
    &\gamma_{1,m}^{-1}(g_m(\bW \bH_m + \bbb) - c_{2,m})\\ 
    &\ge  \frac{1}{n_m} \sum_{i = 1}^{n_m}{\sum_{k = 1}^{K} \left(\overline{\bh}_{m, \bullet, i} - \frac{K}{m\binom{K}{m}} \bh_{m, \{k\}, i}\right)^{\top} \bw^{k}} \numberthis \label{pal_key_inequal} \\ 
    &\ge \frac{1}{n_m} \sum_{i = 1}^{n_m}\sum_{k = 1}^{K}{ \left( - \frac{c_{3,m}}{2} \|\bw^k\|_2^2 - \frac{1}{2c_{3,m}} \|\overline{\bh}_{m, \bullet, i} - \frac{K}{m\binom{K}{m}} \bh_{m,\{k \},i}\|_2^2 \right)} \\
    & = \frac{1}{n_m}\sum_{i = 1}^{n_m} \sum_{k = 1}^{K}{ - \frac{c_{3,m}}{2} \|\bw^k\|_2^2} - \frac{1}{n_m} \sum_{i = 1}^{n_m}\sum_{k = 1}^{K}{\frac{1}{2c_{3,m}} \|\overline{\bh}_{m, \bullet, i} - \frac{K}{m\binom{K}{m}} \bh_{m,\{k \},i}\|_2^2} \\
    & = -\frac{c_{3,m}}{2}\|\bW\|_F^2 - \frac{1}{2 c_{3,m} n_m} \sum_{i = 1}^{n_m}\sum_{k = 1}^{K}{\|\overline{\bh}_{m, \bullet, i} - \frac{K}{m\binom{K}{m}} \bh_{m,\{k \},i}\|_2^2} \\
    & = -\frac{c_{3,m}}{2}\|\bW\|_F^2 - \frac{1}{2 c_{3,m} n_m} \sum_{i = 1}^{n_m} \Big( K {\|\overline{\bh}_{m, \bullet, i}\|_2^2} + \left(\frac{K}{m\binom{K}{m}}\right)^2 \left(\sum_{k = 1}^{K}{\|\bh_{m, \{k\}, i}\|_2^2}\right) \\
    & \qquad \qquad \qquad\qquad \qquad \qquad\qquad \qquad \qquad\qquad \qquad \qquad - 2 K \langle \overline{\bh}_{m, \bullet, i} ,\ \overline{\bh}_{m, \bullet, i}\rangle\Big) \\ 
% \end{align*}
% \begin{align*}
    & = -\frac{c_{3,m}}{2}\|\bW\|_F^2 - \frac{1}{2 c_{3,m} n_m} \sum_{i = 1}^{n_m} \left( \left(\frac{K}{m\binom{K}{m}}\right)^2 \left(\sum_{k = 1}^{K}{\|\bh_{m, \{k\}, i}\|_2^2}\right) - K \|\overline{\bh}_{m, \bullet, i}\|_2^2 \right) \\
    &= -\frac{c_{3,m}}{2}\|\bW\|_F^2 - \frac{\left(\frac{K}{m\binom{K}{m}}\right)^2}{2 c_{3,m} n_m} \sum_{i = 1}^{n_m}\left( \sum_{k = 1}^{K}{\|\bh_{m, \{k\}, i}\|_2^2} - K \|\overline{\bh}_{m, \bullet, i}\|_2^2 \right) \\
    & \ge -\frac{c_{3,m}}{2}\|\bW\|_F^2 - \frac{\left(\frac{K}{m\binom{K}{m}}\right)^2}{2 c_{3,m} n_m} \sum_{i = 1}^{n_m} \sum_{k = 1}^{K}{\|\bh_{m, \{k\}, i}\|_2^2}
    \numberthis\label{h_mean_0}\\
% \end{align*}
% \begin{align*}
    &= -\frac{c_{3,m}}{2}\|\bW\|_F^2 - \frac{\left(\frac{K}{m\binom{K}{m}}\right)^2}{2 c_{3,m} n_m}\left(\|\bH_m \bD_m\|_F^{2}\right) \numberthis\label{pascal_H} \\
    &= -\frac{c_{3,m}}{2}\|\bW\|_F^2 - \frac{\left(\frac{K}{m\binom{K}{m}}\right)^2 \binom{K-2}{m-1}}{2 c_{3,m} n_m}\left(\|\bH_m\|_F^{2}\right) \quad \quad \quad \quad \quad \quad
    (\mbox{by \Cref{lemma:pascal_norm}})
    \\
    &= -\frac{c_{3,m}}{2}\|\bW\|_F^2 - \frac{\kappa_m}{2 c_{3,m} n_m}\left(\|\bH_m\|_F^{2}\right),
\end{align*}
where we let $\bD_m = \mbox{diag}(\bY_m^{\top}, \cdots, \bY_m^{\top}) \in \mathbb{R}^{(n_m*\binom{K}{m}) \times (n_m * K)}$ and $\bY_m \in \mathbb{R}^{K \times \binom{K}{m}}$ is the many-hot label matrix defined as follows\footnote{See \Cref{appendix-section:lex-ordering} for definition of the \(S_{m,k}\) notation}:
\[
\bY_m = \begin{bmatrix}
    \by_{S_{m,k}}
\end{bmatrix}_{k \in \binom{K}{m}}.
\]
% \py{maybe we need a better definition of $\bD_m$ and $\bY_m$}. 
The first Inequality (\ref{pal_key_inequal}) is tight whenever conditions mentioned in \Cref{lemma:key-lower-bound-tightness} are satisfied and the second inequality is tight if and only if
\begin{equation}\label{eq:AMGM_equal}
    c_{3,m}\bw^{k} = \left(\frac{K}{m\binom{K}{m}}\bh_{m, \{k\}, i} - \overline{\bh}_{m, \bullet, i} \right ) \quad \forall k \in [K], \quad i \in [n_m].
\end{equation}

Therefore, we have 
\begin{align*}
    g_m(\bW \bH_m + \bbb) - c_{2,m} \ge -\gamma_{1,m}\frac{c_{3,m}}{2}\|\bW\|_F^2 - \gamma_{1,m}\frac{\kappa_m}{2 c_{3,m} n_m}\left(\|\bH_m\|_F^{2}\right). \numberthis \label{gm_final_inequality}
\end{align*}
The last Inequality (\ref{h_mean_0}) achieves its equality if and only if 
\begin{equation*}
    \overline{\bh}_{m, \bullet, i} = \mb{0}, \quad \forall i \in [n_m] \numberthis \label{eq:h_mean_0}.
\end{equation*}
Plugging this into (\Cref{eq:AMGM_equal}), we have
% where $\sum_{i=1}^{n_m}\sum_{j=1}^K \| \bh_{m,\{j\}, i}\|_2^2 = \| \bH_{m} \bD_m \|_F^2$ and the last equality holds when $\|\overline{\bh}_{m, \bullet, i}\|_2^2 = 0$.
% % Note that here, we have
% % \[
% % \bD_m = \mathrm{diag}(\bY_m^\top, \cdots, \bY_m^\top) \in \mathbb{R}^{n_m \binom{K}{m} \times n_m K}
% % \]
% Now we reach the final form of the inequality:
% \[
% g_m(\bW \bH_m + \bbb) - c_{2,m} \ge - \gamma_{1,m} \frac{c_{3,m}}{2}\|\bW\|_F^2 - \gamma_{1,m} \frac{\left(\frac{\alpha}{\binom{K}{m}}\right)^2}{2 c_{3,m} n_m}\left(\|\bH_m \bD_m\|_F^{2}\right)
% \]
% where the above inequality achieves equality if and only if the AM-GM achieves equality. Thus we now have $\bu = \bw^k$ and $\bv = \left(\frac{\alpha}{\binom{K}{m}}\bh_{m, \{k\}, i}\right)^{\top}$  
\begin{align*}
    c_{3, m} \bw^k &= \frac{K}{m\binom{K}{m}} \bh_{m, \{k\}, i} \\
    \implies 
    c_{3, m}^2 &= \frac{\left(\frac{K}{m\binom{K}{m}}\right)^2 \sum_{i = 1}^{n} \sum_{k = 1}^{K}{\|\bh_{m, \{k\}, i}\|_F^2}}{n_m \sum_{k=1}^{K}{\|\bw^k}\|_2^2} \\
    &= \frac{\left(\frac{K}{m\binom{K}{m}}\right)^2 \binom{K-2}{m-1} \|\bH_m\|_F^2}{n_m \|\bW\|_F^2} \\
    &= \frac{\kappa_m}{n_m}\frac{\|\bH_m\|_F^2}{\|\bW\|_F^2} \\
    \implies c_{3,m} &= \sqrt{\frac{\kappa_m}{n_m}} \frac{ \|\bH_m\|_F}{ \|\bW\|_F}
    \\
    \implies 
    c_{3, m}^2 
&=
\frac{\kappa_m}{n_m}
\frac{\|\bH_m\|_F^2}{\|\bW\|_F^2}.
\end{align*}

Now, note that by our definition of \(\rho\) and \Cref{lemma:KKT-condition}, we get
\begin{equation}
    \|\bH\|_F^2 = \frac{\lambda_{\bW}}{\lambda_{\bH}} \rho.
    \label{equation:rho-W-H-lambda-relations}
\end{equation}

Recall from the state of the lemma that we defined \(\kappa_m := \left(\frac{K/m}{\binom{K}{m}}\right)^2 \binom{K-2}{m-1} \)
and that
\(
\gamma_{1,m} := \frac{1}{1+c_{1,m}} \frac{m}{K-m}
\).
Thus, continuing from Inequality \eqref{gm_final_inequality}, we have
\[
\gamma_{1,m}^{-1}(g_m(\bW \bH_m + \bbb) - c_{2,m} )\ge -  \frac{c_{3,m}}{2}\|\bW\|_F^2 -  \frac{\kappa_m}{2 c_{3,m} n_m}\|\bH_m \|_F^{2}.
\]
Next, let \(Q > 0\) be an arbitrary constant, to be determined later such that
\begin{equation}
\gamma_{1,m}  = \frac{1}{N_m} Q c_{3,m}^{-1} \frac{\|\bH_m\|_F^2}{\|\bW\|_F^2}, \qquad \forall m \in \{1,\dots, M\}.
\label{equation:gamma-choice}
\end{equation}
A remark is in order: at this current point in the proof, it is unclear that such a \(Q\) exists. However, in \Cref{equation-definition:Q}, we derive an explicit formula for \(Q\) such that \Cref{equation:gamma-choice} holds.
% \yw{Note that here we must show that such a choice is possible, because \(\gamma_{1,m} = \frac{1}{1+c_{1,m}} \frac{m}{K-m}\). In other words, \(\gamma_{1,m}\) must belong to 
% \(
% (0,\frac{m}{K-m})
% \). This is a ``debt''.
% }
Now, given \Cref{equation:gamma-choice}, we have 
\[
g_m(\bW \bH_m + \bbb) - c_{2,m} \ge \frac{1}{N_m}Q\left(-  \frac{ 1}{2}\|\bH_m\|_F^2 -  \frac{1}{2  }\|\bH_m \|_F^{2}\right)
=- \frac{1}{N_m}Q\|\bH_m\|_F^2.
\]
Let
\(
\Gamma_2 := \sum_{m=1}^M \frac{N_m}{N} c_{2,m}
\).
Summing the above inequality on both side over \(m = 1,\dots, M\) according to \Cref{eq:sum_up_gm}, we have
\[\label{eq:QandG}
g(\bW \bH + \bbb) - \Gamma_{2} \ge
-\frac{1}{N}Q\sum_{m=1}^M\|\bH_m\|_F^2
=
-\frac{1}{N}Q\|\bH\|_F^2
=
-\frac{1}{N}Q
     \frac{\lambda_{\bW}}{\lambda_{\bH}} \rho. \numberthis
\] 
where the last equality is due to 
    \Cref{equation:rho-W-H-lambda-relations}.
Now, we derive the expression for \(Q\), which earlier we set to be arbitrary. From \Cref{equation:gamma-choice}, we have
\begin{equation}
    \label{equation:c1m-precise}
\frac{1}{1+c_{1,m}} \frac{m}{K-m}
=
\gamma_{1,m}  = \frac{1}{N_m} Q  \frac{\sqrt{n_m}}{\sqrt{\kappa_m}}\frac{\|\bH_m\|_F}{\|\bW\|_F}.
\end{equation}
Rearranging and using the fact that $N_m = \binom{K}{m} n_m$, we have
\[
\binom{K}{m} \frac{1}{1+c_{1,m}} \frac{m}{K-m}
\sqrt{\kappa_mn_m}
=  Q\frac{\|\bH_m\|_F}{\|\bW\|_F}.
\]
Squaring both side, we have
\[
\left(\frac{1}{1+c_{1,m}} \frac{m}{K-m}\right)^2\kappa_m n_m \binom{K}{m}^2
=  Q^2\frac{\|\bH_m\|_F^2}{\|\bW\|_F^2}.
\]
Summing over \(m = 1,\dots, M\), we have
\[
\sum_{m=1}^M
\left(\frac{1}{1+c_{1,m}} \frac{m}{K-m}\right)^2
\kappa_m n_m \binom{K}{m}^2
  = Q^2
  \sum_{m=1}^M\frac{\|\bH_m\|_F^2}{\|\bW\|_F^2}
  =
   Q^2
   \frac{\|\bH\|_F^2}{\|\bW\|_F^2}
   =
   Q^2
   \frac{\lambda_{\bW}}{\lambda_{\bH}}
\]
Thus, we conclude that 
\begin{equation}
    \label{equation-definition:Q}
Q  = 
   \sqrt{\frac{\lambda_{\bH}}{\lambda_{\bW}}}
   \sqrt{
\sum_{m=1}^M
\left(\frac{1}{1+c_{1,m}} \frac{m}{K-m}\right)^2
\kappa_m n_m \binom{K}{m}^2
}.
\end{equation}
Substituting \(Q\) into \Cref{equation:c1m-precise}, we get
\begin{equation}
\label{equation:c1m-precise2}
\frac{1}{1+c_{1,m}} \frac{m}{K-m}
=
 \frac{1}{\binom{K}{m}\sqrt{\kappa_m n_m}} \frac{\|\bH_m\|_F}{\|\bW\|_F}
\sqrt{\frac{\lambda_{\bH}}{\lambda_{\bW}}}
   \sqrt{
\sum_{m'=1}^M
\left(\frac{1}{1+c_{1,m'}} \frac{m'}{K-m'}\right)^2
\kappa_{m'} n_{m'} \binom{K}{m'}^2
}.
\end{equation}
Finally substituting $Q$ into \Cref{eq:QandG},
\begin{equation*}
g(\bW \bH    + \bbb) - \Gamma_{2} \ge
- \frac{1}{N} \sqrt{
\sum_{m=1}^M
\left(\frac{1}{1+c_{1,m}} \frac{m}{K-m}\right)^2
\kappa_m n_m \binom{K}{m}^2
}\sqrt{\frac{\lambda_{\bW}}{\lambda_{\bH}} }\rho. 
\end{equation*}
which concludes the proof.
\end{proof}

% \hrule 

As a sanity check of the validity of \Cref{lemma:gm-lower-bound},
we briefly revisit the \MClf~case where \(M=1\). We show that our \Cref{lemma:gm-lower-bound} recovers \cite{zhu2021geometric} Lemma B.3 as a special case. Now, from the definition of \(\kappa_m\), we have that \(\kappa_1= 1\). Thus, the above expression reduces to simply
\[
Q = 
\sqrt{\frac{\lambda_{\bH}}{\lambda_{\bW}n_1}}
\frac{1}{1+c_{1,1}} \frac{1}{K-1}.
\]
The lower bound from \Cref{lemma:gm-lower-bound} reduces to simply
\[
g_1(\bW \bH_1 + \bbb) - \gamma_{2,1} \ge -
Q 
 \rho \frac{\lambda_{\bW}}{\lambda_{\bH}}
 =
 -
% \sqrt{\frac{\lambda_{\bH}}{\lambda_{\bW}n_1}}
\frac{1}{1+c_{1,1}} \frac{1}{K-1}
 \rho \sqrt{\frac{\lambda_{\bW}}{\lambda_{\bH}n_1}}
\]
which exactly matches that of \cite{zhu2021geometric} Lemma B.3.

Next, we show that the lower bound in Inequality (\ref{g_lower_bound}) is attained if and only if ($\bW, \bH, \bbb$) satisfies the following conditions. 

\begin{lemma} \label{lemma:ETF_and_sacled_ave}
Under the same assumptions of \Cref{lemma:gm-lower-bound}, the lower bound in Inequality (\ref{g_lower_bound}) is attained for a critical point ($\bW, \bH, \bbb$) of Problem \eqref{eq:DL-loss-ufm} if and only if all of the following hold:
\begin{align*}
 &\mbox{(I)}\quad   \|\bw^{1}\|_2 = \|\bw^{2}\|_2 = \cdots = \|\bw^{ K}\|_2, \quad \mbox{and} \quad \bbb = b \mb{1}, \\
    & \mbox{(II)}\quad  \frac{1}{\binom{K}{m}}
    \sum_{k = 1}^{\binom{K}{m}}
    \bh_{m,k,i} = \mb{0}, \quad \mbox{and} \quad \sqrt{\frac{\binom{K-2}{m-1}}{n_m}} \frac{\|\bH_m\|_F}{\|\bW\|_F} \bw^k = 
     \sum_{\ell: k\in S_{m, \ell}}{\bh_{m, \ell, i}},
    % \bh_{m, \{k\}, i}
     \forall m \in [M], k \in [K], i\in [n_m], \\
     &\mbox{(III)}\qquad  \bW^\top\bW = \frac{\rho}{K-1}\left(\mb{I}_K - \frac{1}{K}\mb{1}_K \mb{1}_K^\top \right)
     % , \ \mbox{and} \ c_{1, m} =
\end{align*}

\textit{(IV)} \hspace{1em}  There exist unique positive real numbers \(C_1,C_2,\dots, C_M>0\) such that the following holds: 
\begin{align*}
    &\bh_{1,k,i} = C_1 \bw^{\ell} 
     \quad \qquad\qquad \mbox{when} \,\, S_{1,k} = \{\ell\},\; \ell\in [K], 
    &\mbox{(Multiplicity \(=1 \) Case)} \\
    &\bh_{m,k,i} = C_m  \textstyle{\sum_{\ell \in S_{m,k}} }\bw^{\ell} 
    \quad \mbox{when} \,\, m > 1.
    & \mbox{(Multiplicity \(>1\) Case)} 
\end{align*}
\hspace{2.5em}(See \Cref{appendix-section:lex-ordering} for the notation \(S_{m,k}\).)

\textit{(V)} \hspace{1em} There exists \(c_{1,1} ,\, c_{1,2} , \, \dots, \, c_{1,M} >0\) such that
\begin{align*}
\label{equation:c1m-precise3}
\frac{1}{1+c_{1,m}} \frac{m}{K-m} 
=
 \frac{1}{\binom{K}{m}\sqrt{\kappa_m n_m}} \frac{\sqrt{ \frac{\binom{K}{m} n_m m (K-m)(K-1)}{K}} * log(\frac{K-m}{m} c_{1,m})}{\rho} \cdot \\
\sqrt{\frac{\lambda_{\bH}}{\lambda_{\bW}}}
   \sqrt{
\sum_{m'=1}^M
\left(\frac{1}{1+c_{1,m'}} \frac{m'}{K-m'}\right)^2
\kappa_{m'} n_{m'} \binom{K}{m'}^2 
}\numberthis
\end{align*}
% \yw{The \(\bh_{m, \{k\}, i}\) and \(\overline{\bh}_{m, \bullet, i}\) notation was defined in the previous proof. I unpacked it here, so our lemma introduces as little things as possible. In the proof, I recalled these definitions so this change doesn't mess up your proof.}
% \yw{I deleted the stars, because this lemma is about a \emph{critical point}, not necessarily the \emph{global minimizer}.}
\end{lemma}
% \py{for the above, maybe we need to redefine $c_{1,m}$ like zhu's paper?}
% \yw{Dropping the \(c_{1,m}\) for now. It's a bit complicated.}
The proof of \Cref{lemma:ETF_and_sacled_ave} utilizes the conditions in \Cref{lemma:key-lower-bound-tightness}, and the conditions in \Cref{eq:AMGM_equal} and \Cref{eq:h_mean_0} during the proof of \Cref{lemma:gm-lower-bound}.
\begin{proof} 
Similar as in the proof of \Cref{lemma:gm-lower-bound}, define
\(\bh_{m,\{k\}, i} := \sum_{\ell: k\in S_{m, \ell}}{\bh_{m, \ell, i}}\) \\
and
\(\overline{\bh}_{m, \bullet, i} := \frac{1}{\binom{K}{m}}
    \sum_{k = 1}^{\binom{K}{m}}
    \bh_{m,k,i} \).
From the proof of \Cref{lemma:gm-lower-bound}, the lower bound is attained whenever the conditions in \Cref{eq:AMGM_equal} and \Cref{eq:h_mean_0} hold, which  respectively is equivalent to the following: 
\begin{align*}
    &\overline{\bh}_{m, \bullet, i} = \mb{0} \qquad \mbox{and} \\
    &\sqrt{\frac{\binom{K-2}{m-1}}{n_m}} \frac{\|\bH_m\|_F}{\|\bW\|_F} \bw^k = \bh_{m, \{k\}, i}, \forall m \in [M\, k \in [K], i\in [n_m] \numberthis\label{lemma4_h_condi}, 
\end{align*}

In particular,  the $m = 1$ case further implies 
\begin{align*}
    \sum_{k=1}^{K} \bw^{k} = \mb{0}.
\end{align*}
Next, under the condition described in \Cref{lemma4_h_condi}, when $m = 1$, if we want Inequality (\ref{g_lower_bound}) to become an equality, we only need Inequality (\ref{pal_key_inequal}) to become an equality when $m = 1$, which is true if and only if conditions in \Cref{lemma:key-lower-bound-tightness} holds for $\bz_{1, k, i} = \bW \bh_{1, k, i} \forall i \in [n_m]$ and $\forall k \in [K]$.
% We know: 
% \[
% \overline{\bh}_{m,\bullet,i}=
% \frac{1}{\binom{K}{m}}
% \sum_{k = 1}^{\binom{K}{m}}
% \bh_{m,k,i}
% \]
% \[
% c_{3, m} \bw^k = \frac{\alpha}{\binom{K}{m}} \bh_{m, \{k\}, i} \quad \forall i \in [n]
% \]
First let $[\bz_{1, k, i}]_j = \bh^\top_{1, k, i} \bw^j + b_j$, we would have: 
\begin{align*}
    \sum_{j=1}^K [\bz_{1, k, i}]_j &= K \bar{b} \quad \mbox{and} \quad
    K[\bz_{1, k, i}] = c_{3, 1} \left( K \|\bw^k\|_2^2 \right) + K b_k. \numberthis \label{z_sum_and_Kz} 
\end{align*}
% \yw{(\(\star\))We need to define \(\beta\)} \py{fixed} 
We pick $\gamma_{1,1} = 1 \beta$, where $\beta$ is defined in (\ref{alpha_beta_def}), to be the same for all $k \in [K]$ in multiplicity one, which also means to pick $\frac{1}{\beta} - (K - 1)$ to be the same for all $k \in [K]$ within one multiplicity. Note under the first (\textit{in-group equality}) and second \textit{out-group equality} condition in \Cref{lemma:key-lower-bound-tightness} and utilize the condition (\ref{z_sum_and_Kz}), we have 
\begin{align*}
\frac{1}{\beta} - (K - 1) &= \frac{(K - 1) \mathrm{exp}(z_{out}) + \mathrm{exp}(z_{in})}{\exp(z_{out})} - (K - 1)\\ 
&= (K - 1) + \exp(z_{in} - z_{out}) - (K - 1)\\
&= \exp(z_{in} - z_{out}) \\ 
&= \exp \left( \frac{K z_{in} - z_{in} - (K-1) z_{out}}{K-1} \right) \\
&= \exp \left( \frac{K z_{in} - \sum_j z_j }{K-1}\right) \\
\end{align*}
\begin{align*}
& = \left( \mathrm{exp} \left( \frac{\sum_j{z_j} - K z_{in}}{K-1} \right) \right)^{-1} \\
& = \left( \mathrm{exp} \left( \frac{\sum_j{z_j} - K z_{k}}{K-1} \right) \right)^{-1} \\
& = \exp \left( \frac{K}{K-1} \left(\bar{b} - c_{3,1} \|\bw^k\|_2^2 \right) - b_k\right)^{-1}
\end{align*}
Since the scalar $\gamma_{1,1}$ is picked the same for one $m$, but the above equality we have 
\begin{align}
    c_{3,1} \|\bw^k\|_2^2  - b_k = c_{3,1} \|\bw^\ell\|_2^2 - b_\ell \quad \forall \ell \neq k.
\end{align}
this directly follows after Equation ($29$) from the proof in Lemma B.4 of \cite{zhu2021geometric} to conclude all the conditions except the scaled average condition, which we address next. To this end, we use the second condition in (\ref{lemma4_h_condi}) which asserts for $m\ge 2$ that: 
\begin{align*}\label{h_h_relationship} 
    &\sqrt{\frac{n_m}{\binom{K-2}{m-1}}} \frac{\|\bW\|_F}{\|\bH_m\|_F}\bh_{m, \{k\}, i} = \bw^k \\
    \implies &\sqrt{\frac{n_1}{\binom{K-2}{1-1}}} \frac{\|\bW\|_F}{\|\bH_1\|_F} \bh_{1, \{k\}, i} = \sqrt{n_1} \frac{\|\bW\|_F}{\|\bH_1\|_F} \bh_{1, k, i} = \bw^k  = \sqrt{\frac{n_m}{\binom{K-2}{m-1}}} \frac{\|\bW\|_F}{\|\bH_m\|_F} \bh_{m, \{k\}, i} \\
    \implies & \bh_{m, \{k\}, i} = \sqrt{\frac{n_1 \binom{K-2}{m-1}}{n_m}} \frac{\|\bH_m\|_F}{\|\bH_1\|_F} \bh_{1, k, i} = c_{h,m} \bh_{1, k, i} \numberthis 
\end{align*}
where $c_{h,m} = \sqrt{\frac{n_1 \binom{K-2}{m-1}}{n_m}} \frac{\|\bH_m\|_F}{\|\bH_1\|_F}$. Let $\widetilde{\bH}_1$ (resp.\ $\widetilde{\bH}_m$) be the block-submatrix corresponding to the first \(K\) columns of \(\bH_1\) (resp.\ first \(\binom{K}{m}\) columns of \(\bH_m\)).
Define \(\widetilde{\bY}_1\) and \(\widetilde{\bY}_m\) similarly.
Then, \Cref{h_h_relationship} can be equivalently stated in the following matrix form: 
\begin{align*}
    c_{h,m} \widetilde{\bH}_1 &= \widetilde{\bH}_m \widetilde{\bY}_m^{\top} 
\end{align*}
% which further implies that
% \begin{align*}
%     c_{h,m} \widetilde{\bH}_1 \widetilde{\bY}_m &= \widetilde{\bH}_m \widetilde{\bY}_m^{\top} \widetilde{\bY}_m 
%     = \widetilde{\bH}_m,
% \end{align*}
Let $\bP_m = \widetilde{\bY}_{m}^{\top} (\widetilde{\bY}_{m}^{\top})^{\dagger}$ be the projection matrix onto the subspace $\widetilde{\bY}_{m}$, then we have
\[
\widetilde{\bH}_m \bP_m = \widetilde{\bH}_m \widetilde{\bY}_{m}^{\top} (\widetilde{\bY}_{m}^{\top})^{\dagger} = c_{h,m}\widetilde{\bH}_1 (\widetilde{\bY}_{m}^{\top})^{\dagger},
\]
which simplifies as
\[
\widetilde{\bH}_m \bP_m = c_{h,m} \widetilde{\bH}_1 (\widetilde{\bY}_{m}^{\top})^{\dagger}.
\]
Applying \Cref{lemma:minimum_norm_proj} to the LHS and \Cref{lemma:Y_Moore_pinv} to the RHS we have
\begin{align*}
    \widetilde{\bH}_m &= c_{h,m} \widetilde{\bH}_1 (\tau_m \widetilde{\bY}_m + \eta_m \bm \Theta) \\
    \widetilde{\bH}_m &=  c_{h,m} \cdot \tau_m \widetilde{\bH}_1 \widetilde{\bY}_m
\end{align*}
and substituting $\widetilde{\bH}_1$ using the relationship between $\widetilde{\bH}_1$ and $\bW$, namely, $ c_{h, 1} \cdot (\bW^\top) = \widetilde{\bH}_1 $, we now have
\[
\widetilde{\bH}_m = c_{h, m} \cdot \tau_m \cdot c_{1, m} (\bW^\top \widetilde{\bY}_m)
\]
where 
\begin{align*}
    C_m &= c_{h,m} \cdot c_{h,1} \cdot  \tau_m \\
    & = \sqrt{\frac{n_1}{n_m \binom{K-2}{m-1}}} \frac{\|\bH_m\|_F}{\|\bH_1\|_F} \cdot \sqrt{\frac{1}{n_1}} \frac{\|\bH_1\|_F}{\|\bW\|_F} \\
    & = \sqrt{\frac{1}{n_m \binom{K-2}{m-1}}} \frac{\|\bH_m\|_F}{\|\bW\|_F}    
\end{align*}

This proves (IV).  
Finally, to proof (V), following from \Cref{equation:c1m-precise2} in the proof of \Cref{lemma:gm-lower-bound}, we only need to further simplify $\|\bH_m\|_F$. 

We first establish a connection the between $\|\bW\bH_m\|_F^2$ and $\|\bH_m\|_F^2$. By definition of Frobenius norm and the last layer classifier $\bW$ is an ETF with expression $\bW^\top \bW = \frac{\rho}{K-1}\left(\mb{I}_K - \frac{1}{K}\mb{1}_K \mb{1}_K^\top \right)$, we have
\begin{align*}
        \|\bW\bH_m\|_F^2 &= tr(\bW\bH_m \bH_m^\top \bW^\top) \\
        &= \frac{\rho}{K-1} tr(\bH_m \bH_m^\top (\mb{I}_K - \mb{1}_K \mb{1}_K^\top)) \\
        &= \frac{\rho}{K-1} \|\bH_m\|_F^2
\end{align*}

Since variability within feature already collapse at this point, we can express $\|\bW\bH_m\|_F^2$ in terms of $z_{m, in}$ and $z_{m, out}$:
\begin{align*}
     \|\bW\bH_m\|^2_F = \frac{\rho}{K-1} \|\bH_m\|_F^2 = \binom{K}{m}n_m(mz_{m,in}^2 + (K-m)z_{m,out}^2).
\end{align*}
From the second equality we could express $\|\bH_m\|$ as:
\begin{align*}
\label{eq:Hm_norm_initial}
 \|\bH_m\|_F = \sqrt{\frac{\binom{K}{m} n_m (K-1)}{\rho} (mz_{m,in}^2 + (K-m)z_{m,out}^2)} \numberthis
\end{align*}

Recall from \Cref{lemma:key-lower-bound-tightness}, we have the following equation to express $z_{m, in}$ and $z_{m, out}$
\begin{align*}
    &z_{in} - z_{m,out} = log(\frac{K-m}{m} c_{1,m}). 
\end{align*}
As column sum of $\bH_m$ equals to $\mb{0}$, the column sum of $\bW\bH_m$ also equals to $\mb{0}$ as well. Given the extra constrain of \textit{in-group equality} and \textit{out-group equality} from \Cref{lemma:key-lower-bound-tightness}, it yields: 
\begin{align*}
    mz_{m,in} + (K-m)z_{m,out} = 0
\end{align*}

Now we could solve for $z_{m, in}$ and $z_{m, out}$ in terms of $c_{1,m}$
\begin{align*}
        &z_{m, in} = \frac{K-m}{K}log\left(\frac{K-m}{m} c_{1,m}\right) \\
    &z_{m, out} = -\frac{m}{K}log\left(\frac{K-m}{m} c_{1,m}\right)
\end{align*}

Substituting above expression for $z_{m,in}$ and $z_{m, out}$ into \Cref{eq:Hm_norm_initial}, we have
\begin{align*}
        \|\bH_m\|_F = \sqrt{ \frac{ \binom{K}{m} n_m m (K-m)(K-1)}{\rho K}} \log(\frac{K-m}{m} c_{1,m})
\end{align*}
Finally, we substituting the above expression of $\|\bH_m\|_F$ in to \Cref{equation:c1m-precise2} and conclude: 
\begin{align*}
    \frac{1}{1+c_{1,m}} \frac{m}{K-m}
=
 \frac{1}{\binom{K}{m}\sqrt{\kappa_m n_m}} \frac{\sqrt{ \frac{\binom{K}{m} n_m m (K-m)(K-1)}{K}} * log(\frac{K-m}{m} c_{1,m})}{\rho} \cdot \\
\sqrt{\frac{\lambda_{\bH}}{\lambda_{\bW}}}
   \sqrt{
\sum_{m'=1}^M
\left(\frac{1}{1+c_{1,m'}} \frac{m'}{K-m'}\right)^2
\kappa_{m'} n_{m'} \binom{K}{m'}^2
}.
\end{align*}

Revisiting and combining results from (IV) and (V), we have the scaled-average constant $C_m$ to be 
\begin{align*}
    C_m &= \sqrt{\frac{1}{n_m \binom{K-2}{m-1}}} \frac{\sqrt{ \frac{ \binom{K}{m} n_m m (K-m)(K-1)}{\rho K}} \log(\frac{K-m}{m} c_{1,m})
}{\|\bW\|_F} \\
&= \frac{K-1}{\rho} \log(\frac{K-m}{m} c_{1,m})
\end{align*}
where $c_{1,m}$ is a solution to the system of equation \Cref{equation:c1m-precise3}. 
Note that \Cref{equation:c1m-precise3} hold for all $m$. Thus, we could construct a system of equation whose variable are $c_{1,1}, \cdots, c_{1,m}$. Even when missing some multiplicity data, we sill have same number of variable $c_{1,m}$ as equations. We numerically verifies that under various of UFM model setting (i.e. different number of class and different number of multiplicities), $c_{1,m}$ does solves the above system of equation.  
\end{proof}

%\begin{lemma} \label{lemma:scaled_average}
% Let $\widetilde{\bH}_m$ and $\widetilde{\bY}_m$ be defined the same way as previous Lemma, then we have 
% $\widetilde{\bH}_m = \tau_m \widetilde{\bH}_1 \bY_m$, where $\tau_m = c_{h,m} \alpha$
% \end{lemma}

% \begin{proof}
% \end{proof}

\begin{lemma} \label{lemma:minimum_norm_proj}
Let let $\bP_m = \widetilde{\bY}_{m}^{\top} (\widetilde{\bY}_{m}^{\top})^{\dagger}$ be the projection matrix then we have, $\widetilde{\bH}_m \bP_m = \widetilde{\bH}_m$
\end{lemma}
\begin{proof}
As $\bP_m$ is a projection matrix, we have that $\|\widetilde{\bH}_m\|^2_{F} = \|\widetilde{\bH}_m \bP_m\|^2_{F}$ if and only if $\widetilde{\bH}_m = \widetilde{\bH}_m \bP_m$. So it is suffice to show that $\|\widetilde{\bH}_m\|^2_{F} = \|\widetilde{\bH}_m \bP_m\|^2_{F}$. We denote $\bW \widetilde{\bH}_m \bP_m$ as the projection solution and by \cref{lemma:proj_subspace_z} we have that
\[
\bW \widetilde{\bH}_m \bP_m =\bW \widetilde{\bH}_m,
\]
which further implies that the projection solution $\bW \widetilde{\bH}_m \bP_m$ also solves $g$
\[
g(\bW \widetilde{\bH}_m, \widetilde{\bY}) = g(\bW \widetilde{\bH}_m \bP_m, \widetilde{\bY}).
\] 
When it comes to the regularization term, by minimum norm projection property, we have $\|\widetilde{\bH}_m\|^2_{F} \ge \|\widetilde{\bH}_m \bP_m\|^2_{F}$. Note if the projection solution results in a strictly smaller frobenious norm i.e. $\|\widetilde{\bH}_m\|^2_{F} > \|\widetilde{\bH}_m \bP_m\|^2_{F}$, then $f(\bW, \widetilde{\bH}_m \bP_m, \bbb) < f(\bW, \widetilde{\bH}_m, \bbb)$, this contradict the assumption that $\bZ_m = \bW \widetilde{\bH}_m$ is the global solutions of $f$. Thus, the only possible outcomes is that $\|\widetilde{\bH}_m\|^2_{F} = \|\widetilde{\bH}_m \bP_m\|^2_{F}$, which complete the proof.
\end{proof}

\begin{lemma} \label{lemma:proj_subspace_z}
We want to show that the optimal global solution of $f$, $\bW \widetilde{\bH}_m \bP_m$, is the same after projected on to the space of $\widetilde{\bY}_m$, i.e., $\bW \widetilde{\bH}_m \bP_m = \bW \widetilde{\bH}_m$
\end{lemma}

\begin{proof}
Let $\bZ_m = \bW \widetilde{\bH}_m$ denote the global minimizer of the loss function $f$ for an arbitrary multiplicity $m$. Since $\bZ_m$ has both the in-group and out-group equality property, we could express it as
\[
\bZ_m = d_1 \widetilde{\bY}_m + d_2 \bm \Theta,
\] for some constant $d_1$, $d_2$, and all-one matrix $ \bm \Theta$ of proper dimension. 
Note that it is suffice to show that $\bZ_m$ lives in the subspace of which the projection matrix $\bP_m$ projects onto. 
By \cref{lemma:Y_Moore_pinv}, as $(\widetilde{\bY}_{m}^{\top})^{\dagger}$ is the Moore–Penrose pseudo-inverse of $\widetilde{\bY}_{m}^{\top}$ by, we could rewrite $\bP_m$ as
\begin{align*}
    \bP_m &= \widetilde{\bY}_{m}^{\top} (\widetilde{\bY}_{m}^{\top})^{\dagger}\\
    &=\widetilde{\bY}_{m}^{\top} \left( \widetilde{\bY}_{m}\widetilde{\bY}_{m}^{\top}\right)^{\dagger} \widetilde{\bY}_{m}\\
    & =\widetilde{\bY}_{m}^{\top} \left( \widetilde{\bY}_{m}\widetilde{\bY}_{m}^{\top}\right)^{-1} \widetilde{\bY}_{m}.
\end{align*}
Hence we can see that the subspace which $\bP_m$ projects onto is spanned by columns/rows of $\widetilde{\bY}_m$. In order to show that $\bZ_m = d_1 \widetilde{\bY}_m + d_2 \bm \Theta$ is in the subspace spanned by columns of $\widetilde{\bY}_m$, it is suffice to see that the columns sum of $\widetilde{\bY}_m = \frac{m}{K} \binom{K}{m} \mb{1}$.  Thus, we finished the proof.
\end{proof}

\begin{lemma} \label{lemma:Y_Moore_pinv} The Moore-Penrose pseudo-inverse of $\widetilde{\bY}_{m}^{\top}$ has the form
    $(\widetilde{\bY}_{m}^{\top})^{\dagger} = \tau_m \widetilde{\bY}_m + \eta_m \bm \Theta$, where $\bm 
 \Theta$ is the all-one matrix with proper dimension and $\tau_m = \frac{a + c}{bc}$, $\eta_m = -\frac{a}{bc}$, for $a = \frac{m-1}{k-1} \binom{K-1}{m-1}$, $b = \frac{m}{k} \binom{K}{m}$, $c = \frac{m}{k-1} \binom{K-1}{m}$.
\end{lemma}
\begin{proof}

First, we have the column sum of $\widetilde{\bY}_{m}$ can be written as a constant times an all-one vector
\begin{align}
    \sum_j^{\binom{K}{m}} {(\widetilde{\bY}_{m})_{:, j}} = \frac{m}{K} \binom{K}{m} \mb{1}
\end{align} 
This property could be seen from a probabilistic perspective. We let $i \in [K]$ be fixed and deterministic, and let $S \subseteq [K]$ be a random subset of size $m$ generating by sampling without replacement. Then 
\[
Pr\{i \notin S \} = \frac{K-1}{K} \times \frac{K-2}{K-1} \times \cdots \times \frac{K-m}{K-m+1} = \frac{K-m}{K}.
\]
This implies that $Pr\{i \in S \} = \frac{m}{K}$ and each entry of the column sum result is exactly $\frac{m}{K} \binom{K}{m}$ as we sum up all $\binom{K}{m}$ columns of $\widetilde{\bY}_{m}$. 

Second, the label matrix  $\widetilde{\bY}_{m}$ has the property that 
\begin{align} \label{eq:Ym_property2}
\widetilde{\bY}_m\widetilde{\bY}_m^{\top} = \begin{bmatrix}
b & & a\\
 & \ddots & \\
 a & & b
\end{bmatrix}, \quad
\widetilde{\bY}_m(\bm  \Theta - \widetilde{\bY}_m^{\top}) = \begin{bmatrix}
0 & & c\\
 & \ddots & \\
 c & & 0
\end{bmatrix},    
\end{align}
where $a = \frac{m-1}{k-1} \binom{K-1}{m-1}, \quad$ $b = \frac{m}{k} \binom{K}{m}, \quad$ $c = \frac{m}{k-1} \binom{K-1}{m}$. 
Again, from a probabilistic perspective, any off-diagonal entry of the product $\widetilde{\bY}_m\widetilde{\bY}_m^{\top}$ is equal to $(\widetilde{\bY}_m)_{i, :} (\widetilde{\bY}_m)_{i^{\prime}, :}^{\top}$, for $i \neq i^{\prime}$. Note that $(\widetilde{\bY}_m)_{i, :}$ is a row vector of length $\binom{K}{m}$, whose entry are either $0$ or $1$ and the results of $(\widetilde{\bY}_m)_{i, :} (\widetilde{\bY}_m)_{i^{\prime}, :}^{\top}$ would only increase by one if both $(\widetilde{\bY}_m)_{i, j} = 1$ and $(\widetilde{\bY}_m)_{i^{\prime}, j}^{\top} = 1$ for $j \in [\binom{K}{m}]$. From the previous property we know that there is $\frac{m}{K}$ probability that $(\widetilde{\bY}_m)_{i, j} =1$. In addition, conditioned on $(\widetilde{\bY}_m)_{i, j} =1$, there are $\frac{m-1}{K-1}$ probability that $(\widetilde{\bY}_m)_{i^{\prime}, j} =1$. Thus, $a = \frac{m}{K} \frac{m-1}{K-1} \binom{K}{m} = \frac{m-1}{K-1} \binom{K-1}{m-1}$. For similar reasoning, we can see that conditioned on $(\widetilde{\bY}_m)_{i, j} =1$, there are $1 - \frac{m-1}{K-1} = \frac{K-m}{K-1}$ probability that $(\bm \Theta_{i^{\prime}, j} -(\widetilde{\bY}_m)_{i^{\prime}, j}) =1$. Thus, $c = \frac{m}{K} \frac{K-m}{K-1} \binom{K}{m} = \frac{m}{K-1}\binom{K-1}{m}$. For the similar probabilistic argument, it is easy to see that diagonal of $\widetilde{\bY}_m\widetilde{\bY}_m^{\top}$ are all $b = \frac{m}{K} \binom{K}{m}$ and diagonal of $\widetilde{\bY}_m(\bm \Theta - \widetilde{\bY}_m^{\top})$ are all $0$. 
Then by the second property (\Cref{eq:Ym_property2}), we are about to cook up a left inverse of $\widetilde{\bY}^{\top}$:
\begin{align*}
    \frac{1}{b} \left( \widetilde{\bY}\widetilde{\bY}^{\top} - \frac{a}{c} (\widetilde{\bY}(\mb \Theta - \widetilde{\bY}^{\top})) \right) &= \mb{I} \\
    \widetilde{\bY}\left( \frac{1}{b} \widetilde{\bY}^{\top} - \frac{a}{bc} \bm \Theta + \frac{a}{bc} \widetilde{\bY}^{\top} \right) & = \mb{I} \\
    \widetilde{\bY}\left( \frac{a + c}{bc} \widetilde{\bY}^{\top} - \frac{a}{bc} \bm \Theta \right) & = \mb{I} \\
    \left( \frac{a + c}{bc} \widetilde{\bY} - \frac{a}{bc} \bm \Theta \right) \widetilde{\bY}^{\top} &= \mb{I}
\end{align*}
Let, $\tau_m = \frac{a + c}{bc}$, $\eta_m = -\frac{a}{bc}$, then the pseudo-inverse of $\widetilde{\bY}^{\top}$, namely $(\widetilde{\bY}^{\top})^{\dagger}$ could be written as

\[
(\widetilde{\bY}^{\top})^{\dagger} = \tau_m \widetilde{\bY} + \eta_m \bm \Theta
\]

This inverse is also the Moore–Penrose inverse which is unique since it satisfies that:
\begin{align}
    &\widetilde{\bY}^{\top} (\widetilde{\bY}^{\top})^{\dagger} \widetilde{\bY}^{\top} = \widetilde{\bY}^{\top} \mb{I} = \widetilde{\bY}^{\top} \\
     &(\widetilde{\bY}^{\top})^{\dagger} \widetilde{\bY}^{\top}  (\widetilde{\bY}^{\top})^{\dagger} = \mb{I}  (\widetilde{\bY}^{\top})^{\dagger} = (\widetilde{\bY}^{\top})^{\dagger} \\
    &(\widetilde{\bY}^{\top} (\widetilde{\bY}^{\top})^{\dagger})^{\top} = \widetilde{\bY}^{\top} (\widetilde{\bY}^{\top})^{\dagger} \\
    &((\widetilde{\bY}^{\top})^{\dagger} \widetilde{\bY}^{\top})^{\top} = (\widetilde{\bY}^{\top})^{\dagger} \widetilde{\bY}^{\top}
\end{align}

\end{proof}

\begin{lemma} \label{lemma:pascal_norm}
We would like to show the following equation holds: 
\begin{align*}
    \|\bH_m \bD_m\|_F^2 = \binom{K-2}{m-1}\|\bH_m\|_F^2
\end{align*}
\end{lemma}
\begin{proof}
Note due to how we construct $\bD_m$, it is suffice to show that $ \|\widetilde{\bH}_m \widetilde{\bY}_m^{\top}\|_F^2 = \binom{K-2}{m-1}\|\widetilde{\bH}_m\|_F^2$. Recall the definition that $a = \frac{m-1}{k-1} \binom{K-1}{m-1}$ and $b = \frac{m}{k} \binom{K}{m}$. By unwinding the definition of binomial coefficient and simplifying factorial expressions, we can see that $b - a = \binom{K-2}{m-1}$. Along with the assumption that columns sum of $\widetilde{\bH}_m$ is $\mb{0}$ i.e. $\overline{\bh}_{m, \bullet, i} = \mb{0}, \quad \forall i \in [n_m]$ and the property described in \Cref{eq:Ym_property2}, we have
\begin{align*}
    &\|\widetilde{\bH}_m \widetilde{\bY}_m^{\top}\|_F^2 = \binom{K-2}{m-1}\|\widetilde{\bH}_m\|_F^2 \\
    \iff &\|\tau_m \widetilde{\bH}_1 \widetilde{\bY}_m \widetilde{\bY}_m^{\top}\|_F^2 = \binom{K-2}{m-1}\|\tau_m \widetilde{\bH}_1 \widetilde{\bY}_m \|_F^2\\ 
    \iff & \tau_m^2 (b-a)^2 \|\widetilde{\bH}_1\|_F^2 = \tau_m^2 (b-a) \|\widetilde{\bH}_1 \widetilde{\bY}_m\|_F^2 \\
    \iff & (b-a) \|\widetilde{\bH}_1\|_F^2 = \|\widetilde{\bH}_1 \widetilde{\bY}_m\|_F^2 \\
    \iff & (b-a) \|\widetilde{\bH}_1\|_F^2 = Tr(\widetilde{\bH}_1 \widetilde{\bY}_m \widetilde{\bY}_m \widetilde{\bH}_1^{\top}) \\
    \iff & (b-a) \|\widetilde{\bH}_1\|_F^2 = Tr( (b-a)\widetilde{\bH}_1 \widetilde{\bH}_1^{\top}) \\
    \iff & (b-a) \|\widetilde{\bH}_1\|_F^2 = (b-a) \|\widetilde{\bH}_1\|_F^2
\end{align*}
Thus, we complete the proof.
\end{proof}
% \yw{I commented out the ``numerically'' part.}
% We numerically checked that RHS = $\binom{K-2}{m-1}\bH_2$. Two approach:
% \begin{enumerate}
%     \item directly show $\bY_2^{\top} \bY_2$ rescale $\bH_2$
%     \item pinv of $\bY_2^{\top} = \alpha \bY_2 + \beta \bO$ 
% \end{enumerate}

\paragraph{Remarks.} We conjecture that the feature learned from data with all possible labels $(M = K)$ will collapse to the origin which align with our tag-wise average property. Theoretically, $CE_{PAL}(\mathbf{z}^*, \mathbf{1})$ reaches its minimum as long as $\mathbf{z}^* = \zeta \cdot  \mathbf{1}$ for arbitrary constant $\zeta$. Since $\mathbf{z}^* = \mathbf{W} \mathbf{H}_{K}$ where $\mathbf{W}$ is ETF, it is easy to conclude that $\mathbf{H}_{K}$ has same index value on a given row. Due to regularization terms on $\mathbf{H}$, we conclude that $\mathbf{H}_K = \mathbf{0}$, i.e, the feature learned from data with all labels collapse to the origin. Extra experiment visualizing this phenomenon on the MSCOCO dataset could be found in \Cref{app:related_works}.

The following result is a \MLab~generalization of 
 Lemma B.5 from \cite{zhu2021geometric}:
\begin{lemma}\label{lemma:key-lower-bound-tightness}
 % \label{lemma:key-lower-bound-tightness}
    Let $S \subseteq \{1,\dots, K\}$ be a subset of size $m$ where \(1 \le m < K\). Then for all $\mb{z} = (z_1,\dots, z_K)^\top \in \mathbb{R}^K$ and all $c_{1,m} > 0$, there exists a constant $c_{2,m}$ such that
\begin{equation}
\label{equation:lemma-gradient-based-lower-bound-}     \mathcal{L}_{\mathtt{PAL}}(\mb{z}, \mb{y}_S) \ge 
\frac{1}{1+
c_{1,m}} \frac{m}{K-m}
\cdot \langle  \mb{1} -  \tfrac{K}{m}\mathbb{I}_S,\, \mb{z}\rangle + c_{2,m}.
\end{equation}
In fact, we have
\begin{align*}
 c_{2,m} :=
\frac{c_{1,m}m}{c_{1,m}+1} \log (m)
+
\frac{mc_{1,m}}{1+c_{1,m}}
\log
\left(
\frac{c_{1,m}+1}
{ c_{1,m}}\right)
+
\frac{m}{ 
  c_{1,m}
+
1
}
\log\left((K-m) (c_{1,m}+1)\right).
\end{align*}
The Inequality (\ref{equation:lemma-gradient-based-lower-bound-})     is tight, i.e., achieves equality,  if and only if $\mb{z}$ satisfies all of the following:
\begin{enumerate}
    \item For all $i,j \in S$, we have $z_i = z_j$  (in-group equality). 
    Let \(z_{\mathrm{in}} \in \mathbb{R}\) denote this constant.
    \item For all for all $i,j \in S^c$, we have  $z_i = z_j$ (out-group equality).
    Let \(z_{\mathrm{out}} \in \mathbb{R}\) denote this constant.
    \item $z_{\mathrm{in}} - z_{\mathrm{out}}
=
\log\left(\tfrac{(K-m)}{m} c_{1,m}\right)
=
\log\left(\gamma_{1,m}^{-1}-\tfrac{(K-m)}{m} \right)
    $.
    % \item $\gamma_{1,m} = m \beta $ \yw{what is \(\beta\)? Need to specify here.} %$= \frac{1}{(1+c_1)(K-m)}$ 
    % \item $c_{2,m} = \mathcal{L}_{\mathtt{PAL}}(\mb{z^*}, \mb{y}_S) - m \beta \cdot \langle \mb{1} - \frac{K}{m} \cdot \mathbb{I}_S, \mb{z^*} \rangle $ %$ = c_2$
\end{enumerate}
% 1. $z_i = z_j$ for all $i,j \in S$ (in-group equality), and
% 2. $z_i = z_j$ for all $i,j \in S^c$ (out-group equality).
\end{lemma}

\begin{proof}
% [Proof of lemma \ref{lemma:key-lower-bound-tightness}]
Let \(\mb{z}\) and \(c_{1,m}\) be fixed.
For convenience, let \(\gamma_{1,m} :=
\frac{1}{1+
c_{1,m}} \frac{m}{K-m}
\).
Below, let \(z_{\mathrm{in}}, z_{\mathrm{out}} \in \mathbb{R}\) be arbitrary to be chosen later. Define \(\mb{z}^* = (z_1^*,\dots, z_K^*) \in \mathbb{R}^K\) such that 
\begin{equation}
z_k^* = \begin{cases}
    z_{\mathrm{in}} &: k \in S \\
    z_{\mathrm{out}} &: k \in S^c.
\end{cases}
\label{equation:definition-of-z-star}
\end{equation}

    For any $\mb{z} \in \mathbb{R}^K$, recall from the definition of pick-all-labels cross-entropy loss that
\[
\mathcal{L}_{\mathtt{PAL}}
(\mb{z}, \mb{y}_S) = 
\sum_{k \in S}
{\mathcal{L}_{\mathrm{CE}}(\mb{z}, \bm y_{k})} 
\]
In particular, the function \(\mb{z} \mapsto \mathcal{L}_{\mathtt{PAL}}
(\mb{z}, \mb{y}_S) \) is a sum of strictly convex functions and is itself also strictly convex.
Thus, the first order Taylor approximation of $\mathcal{L}_{\mathtt{PAL}}(\mb{z}, \mb{y}_S)$ around $\mb{z}^*$ yields the following lower bound:
\begin{align}
\mathcal{L}_{\mathtt{PAL}}(\mb{z}, \mb{y}_S)
&\ge 
\mathcal{L}_{\mathtt{PAL}}(\mb z^*, \mb{y}_S) + \langle \nabla \mathcal{L}_\mathtt{PAL}(\mb z^*, \mb{y}_S), \ \mb{z} - \mb z^* \rangle 
% \quad \because \mbox{\(\mathcal{L}_{\mathtt{PAL}}(\cdot, \mb{y}_S)\) is strictly convex}
\nonumber
\\&= 
\langle \nabla \mathcal{L}_{\mathtt{PAL}}(\mb z^*, \mb{y}_S), \ \mb{z} \rangle + \mathcal{L}_{\mathtt{PAL}}(\mb z^*, \mb{y}_S) - \langle \nabla \mathcal{L}_{\mathtt{PAL}}(\mb z^*, \mb{y}_S), \ \mb z^* \rangle
\label{equation:first-order-taylor-lower-bound}
\end{align}
Next, we calculate \( \nabla \mathcal{L}_{\mathtt{PAL}}(\mb z^*, \mb{y}_S)\). First, we observe that
\[
\nabla \mathcal{L}_{\mathtt{PAL}}(\mb z^*, \mb{y}_S) 
= 
\sum_{k \in S}
{\nabla \mathcal{L}_{\mathrm{CE}}(\mb z^*, \mb{y}_{k})}.
\]
Recall the well-known fact that the gradient of the cross-entropy is given by
\begin{equation}\label{equation:gradient-of-CE}
\nabla \mathcal{L}_{\mathrm{CE}}(\mb z^*, \mb{y}_{k})
=
\mathrm{softmax}(\mb z^*) - \mb{y}_{k}.    
\end{equation}

Below, it is useful to define 
\[\alpha := \frac{\mathrm{exp}(z_{\mathrm{in}}^*)}{\sum_j{\mathrm{exp}(z_j^*)}} \quad \mbox{and} \quad \beta := \frac{\mathrm{exp}(z_{\mathrm{out}}^*)}{\sum_j{\mathrm{exp}(z_j^*)}} \numberthis \label{alpha_beta_def}
\]
where 
\(\sum_j{\mathrm{exp}(z_j^*)}
= m \mathrm{exp}(z_{\mathrm{in}}^*)+(K-m)
 \mathrm{exp}(z_{\mathrm{out}}^*)
\). In view of this notation and the definition of \(\mb z^*\) in \Cref{equation:definition-of-z-star}, we have 
\begin{equation}\label{equation:softmax-z-star}
\mathrm{softmax}(\mb z^*)  =
\alpha \mathbb{I}_S + \beta \mathbb{I}_{S^c}    
\end{equation}

where we recall that 
\(\mathbb{I}_S\) and \(\mathbb{I}_{S^c} \in \mathbb{R}^K\) are the indicator vectors for the set \(S\) and \(S^c\), respectively.
Thus, combining \Cref{equation:gradient-of-CE} and \Cref{equation:softmax-z-star}, we get
\begin{align*}
\nabla \mathcal{L}_{\mathtt{PAL}}(\mb z^*, \mb{y}_S) 
= 
\sum_{k \in S}{\nabla \mathcal{L}_{\mathrm{CE}}(\mb z^*, \mb{y}_{k})}
= 
\sum_{k \in S}{
\left(\alpha \mathbb{I}_S + \beta \mathbb{I}_{S^c}
- \mb{y}_{k}\right)
}
= 
m(\alpha \mathbb{I}_S + \beta \mathbb{I}_{S^{c}}) - \mathbb{I}_S.
\end{align*}
The above right-hand-side can be rewritten as
\begin{align*}
m(\alpha \mathbb{I}_S + \beta \mathbb{I}_{S^{c}}) - \mathbb{I}_S
&= 
(m \alpha - 1) \cdot \mathbb{I}_S + m \beta \cdot \mathbb{I}_{S^{c}}
\\&=
(m\alpha - 1 + m \beta - m\beta) \cdot \mathbb{I}_S + m \beta \cdot \mathbb{I}_{S^{c}}
\\&= 
m\beta \cdot \mb{1} - (m\beta + 1 - m \alpha) \cdot \mathbb{I}_{S}
\\&= 
m\beta \cdot \left(\mb{1} - \frac{m\beta + 1 - m \alpha}{m \beta} \cdot \mathbb{I}_S\right).
\end{align*}
% where under IGOG, we define $\alpha = \frac{\mathrm{exp}(z_{in})}{\sum_j{\mathrm{exp}(z_j)}}$ and $\beta = \frac{\mathrm{exp}(z_{\mathrm{out}})}{\sum_j{\mathrm{exp}(z_j)}}$. \yw{Define \(z_{in}\) and \(z_{\mathrm{out}}\)} 
Note that  from
\Cref{equation:softmax-z-star} we have
\(
m \alpha + (K - m) \beta =1 \).
Manipulating this expression algebraically, we have 
\begin{align*}
&m \alpha + (K - m) \beta = 1 \\
&\iff k - m = \frac{1 - m \alpha}{\beta} \\
 &\iff \frac{1}{\beta} \left(\frac{1}{m} - \alpha\right) = \frac{K}{m} - 1 \\
 &\iff 1 + \frac{1}{m\beta} - \frac{\alpha}{\beta} = \frac{K}{m} \\ 
 &\iff \frac{m\beta + 1 - m \alpha}{m \beta} = \frac{K}{m}.
\end{align*}
Putting it all together, we have
\[
\nabla \mathcal{L}_{\mathtt{PAL}}(\mb z^*, \mb{y}_S) = m\beta \cdot (\mb{1} - \tfrac{K}{m} \cdot \mathbb{I}_S).
\]
Thus, combining \Cref{equation:first-order-taylor-lower-bound} with the above identity, we have
\begin{equation}
\mathcal{L}_{\mathtt{PAL}}(\mb{z}, \mb{y}_S) \ge m\beta \cdot \langle \mb{1} - \tfrac{K}{m} \cdot \mathbb{I}_S, \ \mb{z} \rangle + \mathcal{L}_{\mathtt{PAL}}(\mb z^*, \mb{y}_S) - m \beta \cdot \langle \mb{1} - \tfrac{K}{m} \cdot \mathbb{I}_S, \mb z^* \rangle.
\label{equation:lemma-gradient-based-lower-bound-2}
\end{equation}
Let 
\begin{equation}
c_{2,m} 
:=
\mathcal{L}_{\mathtt{PAL}}(\mb z^*, \mb{y}_S) - m \beta \cdot \langle \mb{1} - \tfrac{K}{m} \cdot \mathbb{I}_S, \mb z^* \rangle
\label{equation:definition-gamma-2-m}    
\end{equation}
Note that this definition depends on \(\beta\), which in terms depends in \(z^*_{\mathrm{in}}\)
and
\(z^*_{\mathrm{out}}\)
which we have not yet defined.
To define these quantities,
note that in order to derive \Cref{equation:lemma-gradient-based-lower-bound-} from \Cref{equation:lemma-gradient-based-lower-bound-2}, a sufficient condition is to ensure that
\begin{equation}
   \label{equation:m-beta-relation} 
\frac{1}{1+
c_{1,m}} \frac{m}{K-m}= m \beta 
=
 \frac{m\exp(z_{\mathrm{out}}^*)}{ 
\sum_{j} \exp(z_j^*)
}
=
 \frac{1}{ 
 \exp(z_{\mathrm{in}}^* - z_{\mathrm{out}}^*)
+
\frac{(K-m)}{m}
}
\end{equation}

Rearranging, the above can be rewritten as
\[
(1+
c_{1,m}) \tfrac{K-m}{m}
=
 \exp(z_{\mathrm{in}}^* - z_{\mathrm{out}}^*)
+
\tfrac{(K-m)}{m}
\iff
c_{1,m} 
=
\tfrac{m}{K-m}
 \exp(z_{\mathrm{in}}^* - z_{\mathrm{out}}^*)
\]
or, equivalently, as
\begin{equation}
\label{equation:zin-zout-c1m-relation}
z_{\mathrm{in}}^* - z_{\mathrm{out}}^*
=
\log\left(\tfrac{(K-m)}{m} c_{1,m}\right).
\end{equation}
Thus, if we choose \(z_{\mathrm{in}}^*,z_{\mathrm{out}}^*\) such that 
the above  holds, then 
 \Cref{equation:lemma-gradient-based-lower-bound-} holds.

 Finally, we compute the closed-form expression for \(c_{2,m}\) defined in 
 \Cref{equation:definition-gamma-2-m}, which we restate below for convenience:
 \[
 c_{2,m}:=
\mathcal{L}_{\mathtt{PAL}}(\mb z^*, \mb{y}_S) - m \beta 
\cdot \langle \mb{1} - \tfrac{K}{m} \cdot \mathbb{I}_S, \mb z^* \rangle
 \]
 The expression for \(m \beta \) is given at 
 \Cref{equation:m-beta-relation}.
 Moreover, we have
 \[
 \langle \mb{1} - \tfrac{K}{m} \cdot \mathbb{I}_S, \mb z^* \rangle
 =
 m z_{\mathrm{in}}
 + (K-m) z_{\mathrm{out}}
 - \tfrac{K}{m} mz_{\mathrm{in}}
 =
 -(K-m)(z_{\mathrm{in}} - z_{\mathrm{out}}).
 \]
 Thus, we have
 \begin{align*}
 - m \beta 
\cdot \langle \mb{1} - \tfrac{K}{m} \cdot \mathbb{I}_S, \mb z^* \rangle
&=
\frac{(K-m)(z_{\mathrm{in}} - z_{\mathrm{out}})}{ 
 \exp(z_{\mathrm{in}}^* - z_{\mathrm{out}}^*)
+
\frac{(K-m)}{m}
}
\\&=
\frac{(K-m)\log\left(\tfrac{(K-m)}{m} c_{1,m}\right)}{ 
 \tfrac{(K-m)}{m} c_{1,m}
+
\frac{(K-m)}{m}
}
\\&=
\frac{m}{ 
  c_{1,m}
+
1
}
\log\left(\tfrac{(K-m)}{m} c_{1,m}\right).
 \end{align*}
On the other hand,
\[
\mathcal{L}_{\mathtt{PAL}}(\mb z^*, \mb{y}_S)
=
\sum_{k \in S}
\mathcal{L}_{\mathrm{CE}}(\mb z^*, \mb{y}_k)
\]
Now, 
\begin{align*}
\mathcal{L}_{\mathrm{CE}}(\mb z^*, \mb{y}_k)
&=
-\log
(
[\mathrm{softmax}(\mb z^*)]_{k})
\\&=
-\log
(
\exp(z^*_{\mathrm{in}})/
(m\exp(z^*_{\mathrm{in}})
+
(K-m)\exp(z^*_{\mathrm{out}})
)
)
\\&=
\log
(m
+
(K-m)\exp(z^*_{\mathrm{out}} - z^*_{\mathrm{in}})
)
\\&=
\log
\left(m
+
(K-m)(1/\exp(z^*_{\mathrm{in}} - z^*_{\mathrm{out}}))
\right)
\\&=
\log
\left(m
+
(K-m)\frac{1}
{\tfrac{(K-m)}{m} c_{1,m}}
\right) \qquad \quad \quad \mbox{by \Cref{equation:zin-zout-c1m-relation}}
\\&=
\log
\left(m
+
m\frac{1}
{ c_{1,m}}
\right)
\\&=
\log
\left(m\left(
\frac{c_{1,m}+1}
{ c_{1,m}}\right)
\right)
\end{align*}
Thus
\[
\mathcal{L}_{\mathtt{PAL}}(\mb z^*, \mb{y}_S)
=
m
\log
\left(m\left(
\frac{c_{1,m}+1}
{ c_{1,m}}\right)
\right).
\]
Putting it all together, we have
\[
 c_{2,m}=
 m
\log
\left(m\left(
\frac{c_{1,m}+1}
{ c_{1,m}}\right)
\right)
+
\frac{m}{ 
  c_{1,m}
+
1
}
\log\left(\tfrac{(K-m)}{m} c_{1,m}\right).
 \]
 \[
 m
\log
\left(m\left(
\frac{c_{1,m}+1}
{ c_{1,m}}\right)
\right)
=
m \log (m)
+
m
\log
\left(
\frac{c_{1,m}+1}
{ c_{1,m}}\right)
 \]
 \[
 \frac{m}{ 
  c_{1,m}
+
1
}
\log\left(\tfrac{(K-m)}{m} c_{1,m}\right)
=
\frac{m}{ 
  c_{1,m}
+
1
}
\log\left((K-m) c_{1,m}\right)
-
\frac{m}{ 
  c_{1,m}
+
1
}
\log(m)
 \]
 Putting it all together, we have
\begin{align}
 c_{2,m}&=
 m
\log
\left(m\left(
\frac{c_{1,m}+1}
{ c_{1,m}}\right)
\right)
+
\frac{m}{ 
  c_{1,m}
+
1
}
\log\left(\tfrac{(K-m)}{m} c_{1,m}\right)
\\&=
m \log (m)
+
m
\log
\left(
\frac{c_{1,m}+1}
{ c_{1,m}}\right)
+
\frac{m}{ 
  c_{1,m}
+
1
}
\log\left((K-m) c_{1,m}\right)
-
\frac{m}{ 
  c_{1,m}
+
1
}
\log(m)
\\&=
\frac{c_{1,m}m}{c_{1,m}+1} \log (m)
+
m
\log
\left(
\frac{c_{1,m}+1}
{ c_{1,m}}\right)
+
\frac{m}{ 
  c_{1,m}
+
1
}
\log\left((K-m) c_{1,m}\right)
\end{align}
 Next, for simplicity, let us drop the subscript and simply write \(c := c_{1,m}\). Then
 \begin{align*}
     &
     m
\log
\left(
\frac{c+1}
{ c}\right)
+
\frac{m}{ 
  c
+
1
}
\log\left((K-m) c\right)
\\&=
\frac{m}{1+c}
\log
\left(
\frac{c+1}
{ c}\right)
+
\frac{mc}{1+c}
\log
\left(
\frac{c+1}
{ c}\right)
+
\frac{m}{ 
  c
+
1
}
\log\left((K-m) c\right)
\qquad \because \tfrac{1}{1+c} + \tfrac{c}{1+c} = 1
\\&=
\frac{m}{1+c}
\log
\left(
\frac{c+1}
{ c}\right)
+
\frac{mc}{1+c}
\log
\left(
\frac{c+1}
{ c}\right)
+
\frac{m}{ 
  c
+
1
}
\log\left((K-m) c\right)
\\& \qquad 
+
\frac{m}{ 
  c
+
1
}
\log\left((K-m) (c+1)\right)
-
\frac{m}{ 
  c
+
1
}
\log\left((K-m) (c+1)\right)
\qquad \because \mbox{add a ``zero''}
\\&=
\frac{m}{1+c}
\log
\left(
\frac{c+1}
{ c}\right)
+
\frac{mc}{1+c}
\log
\left(
\frac{c+1}
{ c}\right)
+
\frac{m}{ 
  c
+
1
}
\log\left(\frac{c}{c+1}\right)
\qquad \because \mbox{property of \(\log\)}
\\& \qquad 
+
\frac{m}{ 
  c
+
1
}
\log\left((K-m) (c+1)\right)
\\&=
\frac{mc}{1+c}
\log
\left(
\frac{c+1}
{ c}\right)
+
\frac{m}{ 
  c
+
1
}
\log\left((K-m) (c+1)\right)
\qquad \because \log(\tfrac{c+1}{c}) = -\log(\tfrac{c}{c+1})
 \end{align*}
 To conclude, we have
\begin{align*}
 c_{2,m}=
\frac{c_{1,m}m}{c_{1,m}+1} \log (m)
+
\frac{mc_{1,m}}{1+c_{1,m}}
\log
\left(
\frac{c_{1,m}+1}
{ c_{1,m}}\right)
+
\frac{m}{ 
  c_{1,m}
+
1
}
\log\left((K-m) (c_{1,m}+1)\right)
\end{align*}
as desired.
\end{proof}

\section{Nonconvex Landscape Analysis}\label{app:landscape}

\begin{proof}[Proof of Theorem 
\ref{thm:landscape}]
%\ref{thm:optim_landscape}

We note that the proof for Theorem 3.2 in \cite{zhu2021geometric} could be directly extended in our analysis. More specifically, the proof in \cite{zhu2021geometric} relies on a connection for the original loss function to its convex counterpart, in particular, letting $\mb Z = \mb W \mb H \in \mathbb{R}^{K \times N}$ with $N = \sum_{m}{n_m}$ and $\alpha=\frac{\lambda_{\mb H}}{\lambda_{\mb W}}$, the original proof first shows the following fact:
\begin{align*}
   \min_{\mb H\mb W = \mb Z} \; \lambda_{\mb W} ||\mb W||_{F}^2 + \lambda_{\mb H} ||\mb H||_{F}^2 \;&=\; \sqrt{\lambda_{\mb W} \lambda_{\mb H} }  \min_{\mb H\mb W = \mb Z}  \;\frac{1}{\sqrt{\alpha} } ( ||\mb W||_{F}^2 + \alpha ||\mb H||_{F}^2 ) \\
   \;&=\; \sqrt{\lambda_{\mb W} \lambda_{\mb H} }  ||\mb Z||_{*}.
\end{align*}
With the above result, the original proof relates the original loss function
\begin{align*}
     \min_{\mb W , \mb H,\mb b  } \; f(\mb W,\mb H,\mb b) \;:=\; g(\mb W \mb H + \mb b \mb 1^\top) \;+\; \lambda_{\mb W} ||\mb W||_{F}^2 + \lambda_{\mb H}||\mb H||_{F}^2 + \lambda_{\mb b} ||\mb b||_{2}^2
\end{align*}
with 
\begin{align*}
    g(\mb W \mb H + \mb b \mb 1^\top) := \frac{1}{N}\sum_{k=1}^K\sum_{i=1}^n \mathcal{L}_{\mathrm{CE}}(\mb W \mb h_{k,i} + \mb b,\mb y_k),
\end{align*}
to a convex problem:
\begin{align*}
    \min_{\mb Z \in \mathbb R^{K \times N},\; \mb b \in \mathbb R^{K}}\; \widetilde{f}(\mb Z,\mb b) \;:=\; g(\mb Z + \mb b \mb 1^\top) +  \sqrt{\lambda_{\mb W} \lambda_{\mb H} }  ||\mb Z||_{*} + \lambda_{\mb b} ||\mb b||_{2}^2.
\end{align*}

In our analysis, by letting $\widetilde{g}(\mb W \mb H + \mb b \mb 1^\top) := \frac{1}{Nm} \sum_{m=1}^m \sum_{i=1}^{n_m} \sum_{k=1}^{\binom{K}{m}} \mathcal{L}_{\mathtt{PAL}}( \mb{W} \mb{h}_{m,k,i} + \mb{b},  \mb{y}_{S_{m,k}})$, we can directly apply the original proof for our problem. For more details, we refer readers to the proof of Theorem 3.2 in \cite{zhu2021geometric}.

\end{proof}

%\section{Future Works and Discussion}\label{app:extra_diss}

%\input{appendices/app_extra_discussion}

\end{document}